\documentclass{article}
\usepackage{hyperref}

\usepackage[preprint]{icml2026}
\usepackage{latexsym}
\usepackage[T1]{fontenc}
\usepackage[utf8]{inputenc}
\usepackage{microtype}
\usepackage{inconsolata}
\usepackage{amsmath,amssymb,amsthm}
\usepackage{booktabs}
\usepackage{longtable}
\usepackage{capt-of}
\usepackage{graphicx}
\usepackage{array}
\usepackage{multirow}
\usepackage{algorithm}
\newcommand{\Require}{\REQUIRE}
\newcommand{\Ensure}{\ENSURE}
\newcommand{\State}{\STATE}
\newcommand{\For}{\FOR}
\newcommand{\EndFor}{\ENDFOR}
\newcommand{\Comment}[1]{\COMMENT{#1}}
\newcommand{\Return}{\textbf{return}\space}
\usepackage{xcolor}
\usepackage{tcolorbox}
\tcbuselibrary{skins,breakable}
\usepackage{needspace}

\definecolor{OECPBlue}{HTML}{24649A}   
\definecolor{OECPBlueBG}{HTML}{F2F6FA}
\definecolor{OECPTeal}{HTML}{287D78}   
\definecolor{OECPTealBG}{HTML}{F1F8F7}
\definecolor{OECPGold}{HTML}{9A641F}   
\definecolor{OECPGoldBG}{HTML}{FBF7EE}
\tcbset{oecpbase/.style={arc=1.8mm, boxrule=1.2pt, left=1.5mm, right=1.5mm,
  top=0.9mm, bottom=0.9mm, boxsep=0.5mm, before skip=6pt, after skip=6pt,
  fonttitle=\bfseries, toptitle=0.3mm, bottomtitle=0.3mm}}
\newtcolorbox{oecpformulationbox}[1][]{oecpbase, colback=OECPBlueBG,
  colframe=OECPBlue!55, coltitle=black, colbacktitle=OECPBlueBG, #1}
\newtcolorbox{oecpmainresultbox}[1][]{oecpbase, colback=OECPGoldBG,
  colframe=OECPGold!70, coltitle=black, colbacktitle=OECPGoldBG, #1}
\newtcolorbox{oecpcompressionbox}[1][]{oecpbase, colback=OECPTealBG,
  colframe=OECPTeal!40, coltitle=black, colbacktitle=OECPTealBG, #1}

\newtheorem{theorem}{Theorem}

\newtheorem{corollary}{Corollary}
\newtheorem{proposition}{Proposition}
\newtheorem{appthm}{Theorem}

\newtheorem{appcor}[appthm]{Corollary}
\theoremstyle{definition}

\newcommand{\Rdist}{R^*_{T,\mathrm{dist}}}
\newcommand{\meff}{m_{\mathrm{eff}}}
\newcommand{\Fiso}{\mathcal F^{\uparrow}}
\newcommand{\1}{\mathbf 1}
\DeclareMathOperator{\width}{width}
\DeclareMathOperator{\rank}{rank}

\icmltitlerunning{Online Evidence-Containment Prediction under Distinct Queries}
\icmlkeywords{online learning, evidence containment, isotonic regression, evidence acquisition}
\hypersetup{
  pdfsubject={Research preprint},
  pdfborder={0 0 0}
}
\makeatletter
\renewcommand{\ICML@preprint}{\textit{Preprint.}}
\makeatother

\begin{document}
\twocolumn[
\icmltitle{\texorpdfstring{Online Evidence-Containment Prediction over Text Spans\\under Distinct Queries}{Online Evidence-Containment Prediction over Text Spans under Distinct Queries}}
\begin{icmlauthorlist}
\icmlauthor{Xiang Liu}{xmu}
\icmlauthor{Haochen Ding}{nefu}
\end{icmlauthorlist}
\icmlaffiliation{xmu}{Xiamen University, China}
\icmlaffiliation{nefu}{Northeast Forestry University, China}
\icmlcorrespondingauthor{Xiang Liu}{\nolinkurl{persevereliu@stu.xmu.edu.cn}}
\vskip 0.3in
]
\printAffiliationsAndNotice{}

\begin{abstract}
Evidence-grounded NLP systems often verify nested text regions sequentially, so verifying one region constrains others.
We formalize this as Online Evidence-Containment Prediction (OECP): each candidate is predicted before exact verification and queried at most once.
Span containment is exactly a triangular 2D product order, making OECP distinct-query online isotonic regression.
With multiple admissible targets, worst-case regret is near-linear, $\widetilde{\Theta}(m_{\rm eff})$, in the effective scale given by endpoint-rank compression; a single fixed target, as in standard QA, admits an explicit polynomial-time learner with optimal $\Theta(\log m_{\rm eff})$ regret, run on real NQ episodes.
On Natural Questions, containment yields a large deductive gain plus a smaller learned gain on unresolved candidates, reproduced across two backbones and DPR hierarchies.
Expected Search-space Reduction (ESR) turns these probabilities into decisions, localizing evidence in fewer verifier calls on SciFact, ERASER, and DPR hierarchies; on DPR hierarchies under call budgets it outperforms four recent external selectors and raises fixed-reader F1 over learned and generator-greedy scheduling.
\end{abstract}

\section{Introduction}
\label{sec:intro}

\begin{figure*}[t]
\centering
\includegraphics[width=\textwidth]{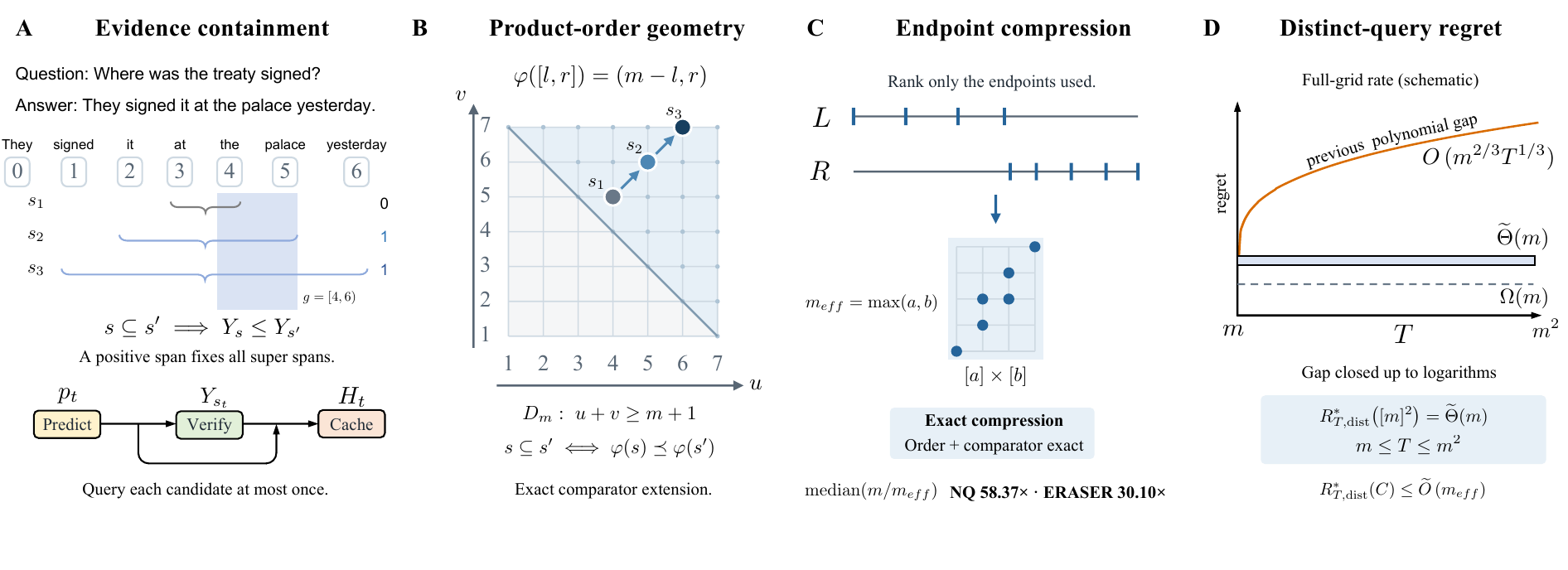}
\caption{Overview of OECP. \textbf{(A)} Nested spans with monotone containment labels, each verified once. \textbf{(B)} Exact map from spans to a triangle in the $m\times m$ grid ($m$: raw token length). \textbf{(C)} Endpoint-rank compression to the $\meff\times\meff$ grid of observed endpoints. \textbf{(D)} Distinct-query regret is $\widetilde\Theta(m)$ on the raw grid, hence $\widetilde O(\meff)$ after compression.}
\label{fig:overview}
\end{figure*}

Evidence-grounded NLP systems increasingly verify evidence incrementally, proposing a candidate region, checking it, and continuing the search \citep{thorne2018fever,trivedi2023interleaving,xie2025fire}.
Candidates nest, sentences inside paragraphs inside sections, so labels are dependent: if a span contains complete evidence, so does every superspan, and if not, no subspan can.
Existing work studies retrieval, sufficiency, and adaptive acquisition \citep{karpukhin2020dense,joren2025sufficient,jeong2024adaptiverag}, but not how verification outcomes propagate across nested candidates.

We formalize the predictor-side problem as \textbf{Online Evidence-Containment Prediction (OECP)}: for externally proposed candidates, a learner predicts $Y_s=\1\{\exists g\in G: g\subseteq s\}$, whether span $s$ contains a complete admissible target, before an exact verifier reveals it, and each candidate is queried at most once (Fig.~\ref{fig:overview}A).
Containment compares endpoints coordinatewise, so it is exactly a triangular two-dimensional product order and OECP is distinct-query online isotonic regression (Fig.~\ref{fig:overview}B); exact endpoint-rank compression replaces token length by the effective scale $\meff$, the number of distinct candidate endpoints (Fig.~\ref{fig:overview}C).
Target multiplicity then changes the complexity: with alternative admissible targets, the worst case is near-linear, $\widetilde\Theta(\meff)$ (Fig.~\ref{fig:overview}D), whereas a single fixed target, as in standard single-answer QA, admits an explicit polynomial-time learner with optimal $\Theta(\log\meff)$ regret.

Empirically, containment acts through exact deduction and a smaller learned signal on unresolved candidates.
\emph{Expected Search-space Reduction} (ESR; Section~\ref{sec:esr}) turns the resulting probabilities into decisions that localize compact evidence in fewer verifier calls on SciFact, ERASER, and DPR hierarchies, with scoped call-budget and reader-F1 gains.

\paragraph{Contributions.}
\begin{itemize}\setlength{\itemsep}{0pt}\setlength{\parsep}{0pt}\setlength{\topsep}{2pt}
\item \textbf{OECP.} We formalize sequential evidence-containment prediction and show that span containment is exactly a triangular 2D product order.
\item \textbf{Sharp complexity characterization.} Exact endpoint compression yields the effective scale $\meff$; the general multi-target class has $\widetilde\Theta(\meff)$ regret, while fixed-target OECP has an explicit polynomial-time learner with worst-case optimal $\Theta(\log\meff)$ regret, which we execute on real NQ candidate families.
\item \textbf{Prediction and verification.} Beyond exact deduction, containment adds a reproducible learned signal on unresolved candidates; ESR localizes evidence with fewer verifier calls on SciFact, ERASER, and DPR hierarchies, where under call budgets it outperforms four recent external selectors and raises fixed-reader F1 over learned and generator-greedy scheduling.
\end{itemize}

\section{Related Work}
\label{sec:related}

\paragraph{Evidence selection, sufficiency, and scientific verification.}
Evidence-centric NLP studies retrieval, rationale extraction, and evidence sufficiency across QA and fact verification \citep{kwiatkowski2019natural,deyoung2020eraser,thorne2018fever,joren2025sufficient}.
Beyond sentence-wise retrieval, evidence has been modeled through document-level token selection \citep{stammbach2021evidence}, modular and joint rationale selection on scientific claims \citep{rana2022rerrfact,wadden2022multivers}, and open-domain scientific verification \citep{wadden2022scifactopen}.
These methods optimize which evidence to retrieve or retain; rationale metrics themselves can be gamed without better evidence \citep{hsia2024goodhart}, so we report localization separately from downstream accuracy.
OECP studies a complementary dependency among already proposed candidates: verification outcomes propagate deterministically through span containment.

\paragraph{Sequential, hierarchical, and adaptive evidence acquisition.}
Evidence acquisition can itself be hierarchical or sequential.
Dense hierarchical retrieval first narrows documents before passages \citep{liu2021dense}, while recent multi-granular RAG organizes sentence- and phrase-level evidence in hierarchical structures and iteratively refines retrieval \citep{hu2026mgranrag}.
Adaptive systems also interleave retrieval, reasoning, stopping, or verification \citep{jiang2023active,asai2024selfrag,jeong2024adaptiverag,xie2025fire}.
Recent selectors span reasoning-based reranking \citep{zhang2025rearank}, nested evidence selection \citep{verma2026nest}, submodular budgeted routing \citep{qureshi2026budgetrouting}, and cost-aware selection \citep{wu2026knowledgefree}.
OECP operates after candidates are proposed and exploits a different signal: the consequences of observed verification outcomes under containment. ESR uses this structure to shrink the remaining search space, in the spirit of search-space reduction for QA \citep{christmann2022beyond}, but one verification at a time.

\paragraph{Structured and online prediction under partial orders.}
Online isotonic regression studies monotone prediction over ordered domains, including linear and product orders \citep{kotlowski2016online,kotlowski2017random,wang2026product}, using sequential-complexity and entropy tools \citep{rakhlin2014online,gao2007entropy}.
OECP connects this theory to text-span verification: interval containment induces the relevant partial order \citep{dushnik1941partially,golumbic1989containment}, while distinct queries arise because verified candidates are cached rather than re-verified.
Endpoint-rank compression replaces raw token coordinates by observed endpoints, and we further separate the general multi-target class from the logarithmic fixed-target subclass.

\section{OECP: Formulation and Theory}
\label{sec:formulation}

\subsection{Sequential Evidence-Containment Prediction}
\label{sec:oecp}

\paragraph{Running example.}
For \emph{``Where was the treaty signed?''}, a system may consider the document, a paragraph, a sentence in it, and the phrase ``the palace'' (Fig.~\ref{fig:esr}).
Before checking each, a model predicts whether it contains complete evidence.
If the sentence verifies positive, so must the paragraph and document; if the paragraph verifies negative, so must the sentence and phrase.
A verified span is cached, never rechecked.
OECP studies prediction before each verification; generator and scheduler may be external.
We now formalize this predict--verify--update loop.
Fix a document of $m$ tokens; a candidate is a span $s=[l,r)$, $0\le l<r\le m$, and $\mathcal S_m$ denotes all spans.

\begin{oecpformulationbox}[title={Protocol 1: OECP}]
\textbf{Candidates.} A family $\mathcal C\subseteq\mathcal S_m$, generated independently of the gold evidence.\\[0.3ex]
\textbf{Targets.} Hidden admissible target spans $G\subseteq\mathcal S_m$.\\[0.3ex]
\textbf{Containment label.}
\[
Y_s=\1\{\exists g\in G: g\subseteq s\}.
\]
\textbf{Online protocol.} An external scheduler selects an unverified $s_t\in\mathcal C$; the learner predicts $p_t\in[0,1]$ before an exact verifier reveals and caches $Y_{s_t}$. Each candidate is queried at most once.
\end{oecpformulationbox}

\noindent
\textbf{In words.} The scheduler presents an unseen candidate, the learner predicts, and the verifier's binary outcome joins the history $H_t$; since re-verification yields nothing new, queries are \emph{distinct}.
By transitivity, $s\subseteq s'\Rightarrow Y_s\le Y_{s'}$: the label is isotonic by definition, not by a relevance assumption.

\paragraph{Online loss.}
The learning question is how much additional squared error a learner incurs relative to the best containment-consistent predictor on the same candidates.
On a finite poset $(P,\preceq)$ with isotonic maps $\Fiso(P)$, an adversary presents $x_t\in P$, the learner predicts $\hat y_t\in[0,1]$, and $y_t\in[0,1]$ is revealed, with $x_s\neq x_t$ for all $s\neq t$ pathwise.
The regret and distinct-query minimax regret are
\begin{equation*}
\begin{aligned}
R_T&=\sum_{t=1}^T(\hat y_t-y_t)^2\\
&\quad-\min_{f\in\Fiso(P)}\sum_{t=1}^T\bigl(f(x_t)-y_t\bigr)^2,\\
\Rdist(P)&=\inf_{\mathcal A}\;\sup_{\mathcal E_{\mathrm{dist}}}\;\mathbb E\,[R_T],
\end{aligned}
\end{equation*}
over learners $\mathcal A$ and adaptive distinct-query adversaries $\mathcal E_{\mathrm{dist}}$; OECP is $P=(\mathcal C,\subseteq)$, $y_t=Y_{s_t}$, where regret equals cumulative Brier loss.

\subsection{What Structure Does Containment Induce?}
\label{sec:sparse}

A span has two degrees of freedom, and $s\subseteq s'$ exactly when $s'$ starts no later and ends no earlier; flipping the left endpoint makes this a grid order.
With the coordinatewise order $\preceq$ on $[m]^2$, let
\begin{equation*}
\begin{aligned}
\phi([l,r))&=(m-l,\,r),\\
D_m&=\{(u,v)\in[m]^2: u+v\ge m+1\}.
\end{aligned}
\end{equation*}

\begin{proposition}[Span-order isomorphism]
\label{prop:span-order}
$\phi$ is a bijection from $\mathcal S_m$ onto $D_m$, and for all $s,s'\in\mathcal S_m$,
\[
s\subseteq s'\iff\phi(s)\preceq\phi(s').
\]
\end{proposition}

\noindent
\textbf{Takeaway.} Nested evidence spans are exactly a triangular subset of a two-dimensional product order, not merely analogous to a partial order (Fig.~\ref{fig:overview}B); hence $\Rdist(\mathcal S_m,\subseteq)\le\Rdist([m]^2)$ (Appendix~\ref{app:span-reduction}).

\subsection{What Does Distinct Querying Buy?}
\label{sec:minimax}

Repeated-query online isotonic regression may revisit the same point; verification caches each checked span instead. Does forbidding repeats change the minimax difficulty?

\begin{oecpmainresultbox}[title={Main result: the distinct-query rate}]
\begin{theorem}[Upper bound]
\label{thm:upper}
There is a universal constant $C>0$ such that for all integers $m\ge 2$ and $1\le T\le m^2$,
\begin{equation*}
\Rdist([m]^2)\le C\min\{T,\Gamma_{m,T}\},
\end{equation*}
where $\Gamma_{m,T}:=m\bigl[1+\log(em)\log(1+T/m)\bigr]$.
\end{theorem}
\begin{corollary}[Near-matching minimax rate]
\label{cor:rate}
For all $m\ge 2$ and $m\le T\le m^2$, $\Rdist([m]^2)\ge m/4$; hence $m/4\le\Rdist([m]^2)\le Cm(1+\log^2 m)$, that is,
\begin{equation*}
\Rdist([m]^2)=\widetilde\Theta(m)\qquad(m\le T\le m^2).
\end{equation*}
\end{corollary}
\end{oecpmainresultbox}

\noindent
\textbf{Takeaway.} Distinct verification removes the polynomial $T$-dependence of the repeated-query $O(m^{2/3}T^{1/3})$ bound \citep{wang2026product}: for $m\le T\le m^2$ the general multi-target problem is near-linear in $m$.
The upper bound turns static covers into sequential covers along distinct-query trajectories and applies chaining; the lower bound uses fair labels on an $m$-antichain realized by multiple targets (Appendices~\ref{app:legal-trees}--\ref{app:oecp-lower-semantics}).

\subsection{What Is the Real Problem Size?}
\label{sec:transfer}

Raw length is the wrong scale: containment depends only on endpoints occurring among the candidates. We use $\phi$ only for the conceptual span-to-grid map and $\psi$ for compressing a finite family.
Let $L,R$ be the observed left and right endpoint sets, $a=|L|$, $b=|R|$, with ascending ranks $\rank_L,\rank_R$.

\begin{oecpcompressionbox}
\vspace{-1.6ex}
\begin{equation*}
\begin{aligned}
\meff&=\max(a,b)\le m,\\
\psi([l,r))&=\bigl(a-\rank_L(l)+1,\ \rank_R(r)\bigr).
\end{aligned}
\end{equation*}
\begin{theorem}[Endpoint-rank compression]
\label{thm:compression}
For finite $\mathcal C$, $\psi$ is injective and $s\subseteq s'\iff\psi(s)\preceq\psi(s')$ on $\mathcal C$, so $(\mathcal C,\subseteq)$ embeds exactly into $[a]\times[b]\subseteq[\meff]^2$, and isotonic comparators on $\psi(\mathcal C)$ extend to $[\meff]^2$.
Hence $\Rdist(\mathcal C,\subseteq)\le\Rdist([\meff]^2)$ for $T\le|\mathcal C|$, and $\Rdist(\mathcal C,\subseteq)\le C\,\Gamma_{\meff,T}$ if $\meff\le T\le\meff^2$.
\end{theorem}
\end{oecpcompressionbox}

\noindent
\textbf{Takeaway.} The relevant scale is $\meff$, the number of distinct observed endpoints, not the raw token length $m$ (Fig.~\ref{fig:overview}C); a family of width $w$ also has regret at least $w/4$, so the bound is tight up to logarithms when $w\approx\meff$ (Appendices~\ref{app:endpoint-compression}--\ref{app:width-lower}).

\subsection{Why Does Target Multiplicity Matter?}
\label{sec:single}

The general lower bound needs multiple admissible targets; a single fixed target generates far fewer labelings and may be easier.
A fixed target $g$ labels $h_g(s)=\1\{g\subseteq s\}$; let $\mathcal H_{\rm one}(\mathcal C)=\{h_g|_{\mathcal C}:g\in\mathcal S_m\}$, $N_1(\mathcal C)=|\mathcal H_{\rm one}(\mathcal C)|$, and measure single-target regret against the best $h\in\mathcal H_{\rm one}(\mathcal C)$ under arbitrary $y_t\in[0,1]$.

\begin{oecpmainresultbox}[enhanced jigsaw, breakable, title={Main result: the fixed-target rate}]
\begin{theorem}[Fixed-target OECP]
\label{thm:single}
For every finite $\mathcal C$ with $a$ left and $b$ right endpoints,
\[
N_1(\mathcal C)\le ab+1\le\meff^2+1,
\]
exponential weights over these labelings is a polynomial-time learner with single-target regret $R_T\le\min\{T,\,2\ln N_1(\mathcal C)\}$, and for a known horizon $T$, choosing before play between it and the constant prediction $\tfrac12$ gives
\[
\begin{aligned}
R_T&\le\min\{T/4,\ 2\ln N_1(\mathcal C)\}\\
&=O\bigl(\min\{T,\log(\meff{+}1)\}\bigr).
\end{aligned}
\]
Conversely, for every $K\ge1$ there is a family $\mathcal C_K$ with $\meff=K$ and $|\mathcal C_K|=K^2$ on which, for all $1\le T\le K^2$, every learner incurs, against a target fixed before play and distinct queries, expected loss at least $\frac14\min\{T,\lfloor\log_2(K{+}1)\rfloor\}$.
Hence, in the worst case over families of effective scale $\meff$, the fixed-target rate is $\Theta(\min\{T,\log(\meff{+}1)\})$.
\end{theorem}
\end{oecpmainresultbox}

\noindent\emph{Intuition.}
Each fixed target labels the candidates like an interval with observed endpoints, so at most $ab+1$ labelings exist and exponential weights pays their logarithm ($T/4$ comes from predicting $\tfrac12$); a threshold target probed by median prefix queries gives the lower bound (Appendix~\ref{app:single}).

\begin{corollary}[Target-multiplicity separation]
\label{cor:separation}
For $K\ge2$, $\mathcal C_K$ contains a $K$-element antichain whose labels multiple admissible targets realize independently, so on this same geometry the general rate is $\widetilde\Theta(K)$ for $K\le T\le K^2$, while the fixed-target rate is $\Theta(\log K)$.
\end{corollary}

\noindent
\textbf{Takeaway.} Target multiplicity changes the statistical regime: one fixed target has worst-case $\Theta(\log\meff)$ regret with a constructive learner, whereas the multi-target envelope is near-linear, $\widetilde\Theta(\meff)$.
The two regimes describe different data: single-answer QA instances are fixed-target, while the near-linear rate needs wide antichains with independently mixed labels, which alternative admissible targets can create (e.g., ERASER evidence groups) but which are rare on NQ ($2.13\%$ of documents; Appendix~\ref{app:seq}).

\begin{figure}[t]
\centering
\includegraphics[width=\columnwidth]{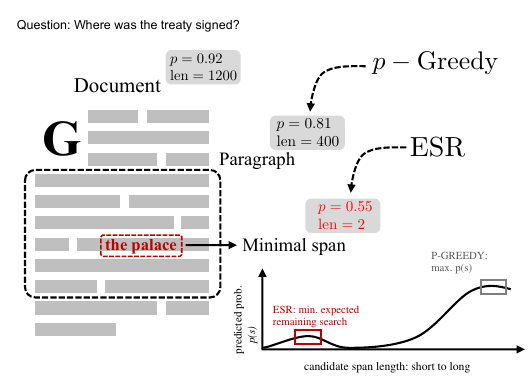}
\caption{Schematic running example: gray boxes give predicted containment probability $p$ and length. The inset is illustrative, not measured: P-GREEDY picks the maximum-$p$ coarse candidate; ESR minimizes the expected remaining search space, moving toward the minimal span.}
\label{fig:esr}
\end{figure}

\Needspace{8\baselineskip}
\section{Decision-Aware Verification}
\label{sec:esr}

Section~\ref{sec:formulation} controls how well containment can be predicted before each verification; it does not decide what to verify next. Of two equally likely candidates, the one whose outcome removes more alternatives is more informative (Fig.~\ref{fig:esr}). ESR consumes the predicted probabilities and the containment order together and formalizes this one-step tradeoff: better calibrated probabilities change ESR's expected remaining set, whereas a greedy rule uses only their ranking. We do not derive search guarantees from the regret bounds; the link is empirical (Section~\ref{sec:decision}).
Targeting one compact sufficient region, not every admissible target, we maintain a \textcolor{OECPTeal}{\textbf{\emph{verified-positive incumbent}}} $R_t$ and an \textcolor{OECPTeal}{\textbf{\emph{active set}}} $A_t$ of feasible unverified candidates, $N_t=|A_t|$.
For $s\in A_t$, let $B_t(s)=\{u\in A_t:u\subseteq s\}$ with \textcolor{OECPTeal}{\textbf{\emph{branch size}}} $b_t(s)=|B_t(s)|$; verifying $s$ gives
\[
A_{t+1}=
\begin{cases}
B_t(s)\setminus\{s\}, & Y_s=1,\\
A_t\setminus B_t(s), & Y_s=0,
\end{cases}
\]
with $R_{t+1}=s$ if positive: a positive outcome leaves the $b_t(s)-1$ candidates inside $s$, a negative one prunes $s$ and its descendants and leaves $N_t-b_t(s)$.
With predicted containment probability $p_t(s)$, the \textcolor{OECPTeal}{\textbf{\emph{expected remaining set}}} after verifying $s$ is
\[
\textcolor{OECPTeal}{\mathcal L_t(s)}=p_t(s)\,[b_t(s)-1]+\bigl(1-p_t(s)\bigr)\,[N_t-b_t(s)],
\]
and \textcolor{OECPTeal}{\textbf{\emph{Expected Search-space Reduction (ESR)}}} verifies the candidate with the smallest expected remaining search space:
\[
\textcolor{OECPTeal}{s_t^{\mathrm{ESR}}}\in\arg\min_{s\in A_t}\mathcal L_t(s).
\]

\begin{proposition}[One-step optimality]
\label{prop:esr}
Conditioned on $A_t$ and the supplied probabilities $p_t$, ESR minimizes the expected number of active candidates after the next verification over all $s\in A_t$.
\end{proposition}

\noindent\emph{Proof.}
Querying $s$ leaves $b_t(s)-1$ active candidates if $Y_s=1$ and $N_t-b_t(s)$ if $Y_s=0$; take expectations. \qed

\noindent
This \textcolor{OECPTeal}{\textbf{\emph{one-step}}} result is conditional on $p_t$ and implies no global, minimax, or multi-step optimality; ESR is a fixed, parameter-free rule.

\paragraph{The OECP pipeline.}
Algorithm~\ref{alg:oecp-pipeline} combines three modules defined above: containment prediction $P_\theta$, optional isotonic correction, and ESR selection.
Given probabilities $q_t(s)=P_\theta(Y_s=1\mid q,d,s,H_t)$ from any containment-aware predictor and the verified history $H_t$, the isotonic operator
\[
\begin{aligned}
\mathcal I(q_t,H_t)={}&\arg\min_{f\in\Fiso(\mathcal C,\subseteq)}\Bigl[\lambda\sum_{s\in\mathcal C}\bigl(f(s)-q_t(s)\bigr)^2\\
&+\sum_{(s,y)\in H_t}\bigl(f(s)-y\bigr)^2\Bigr]
\end{aligned}
\]
returns probabilities monotone under containment ($s\subseteq s'\Rightarrow f(s)\le f(s')$), from which ESR selects the next verification; Section~\ref{sec:experiments} evaluates each module.

\begin{algorithm}[t]
\caption{OECP Prediction and Adaptive Evidence Verification}
\label{alg:oecp-pipeline}
\begin{algorithmic}[1]
\Require query $q$, document $d$, candidates $\mathcal C$, root $r$, predictor $P_\theta$, verifier $\mathcal V$, budget $B$
\State $E\gets\mathrm{Hasse}(\psi(\mathcal C))$; $\mathrm{Sub}(s)\gets\{u\in\mathcal C{:}\,u\subseteq s\}$
\State $R\gets r$; $H\gets\{(r,1)\}$; $A\gets\mathrm{Sub}(r)\setminus\{r\}$ \Comment{root}
\For{$t=2,\dots,B$}
  \State $q_t(s)\gets P_\theta(Y_s{=}1\mid q,d,s,H)$ \Comment{predict}
  \State $p_t\gets\mathcal I(q_t,H)$ over edges $E$ \Comment{isotonic}
  \State $B_t(s)\gets A\cap\mathrm{Sub}(s)$; $\mathcal L_t(s)$ as in Sec.~\ref{sec:esr}
  \State $s_t\gets\arg\min_{s\in A}\mathcal L_t(s)$ \Comment{ESR}
  \State $y_t\gets\mathcal V(s_t)$; $H\gets H\cup\{(s_t,y_t)\}$ \Comment{verify}
  \State \textbf{if} $y_t=1$ \textbf{then} $R\gets s_t$; $A\gets B_t(s_t)\setminus\{s_t\}$
  \State \textbf{else} $A\gets A\setminus B_t(s_t)$ \Comment{update}
  \State \textbf{if} $A=\emptyset$ \textbf{then break}
\EndFor
\State \Return localized evidence $R$
\end{algorithmic}
\end{algorithm}

\section{Experiments}
\label{sec:experiments}

We test whether the containment order appears in real candidate families, separate its deductive and learned contributions to prediction, instantiate the theory's learners, and evaluate ESR on benchmark and retriever-generated candidates, separating localization efficiency from downstream reader accuracy.

\begin{figure*}[t]
\centering
\includegraphics[width=\textwidth]{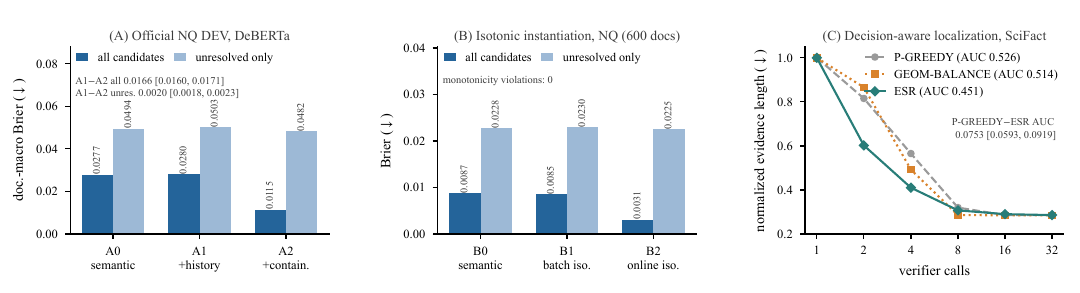}
\caption{\textbf{(A)} Official public NQ DEV ($7{,}830$ documents): matched DeBERTa models differing only in history inputs; containment-aware A2 lowers Brier on all and on unresolved candidates. \textbf{(B)} Internal NQ DEV ($600$ documents): batch (B1) and prequential (B2) isotonic learners. \textbf{(C)} SciFact held-out ($209$ instances): normalized length of the verified-positive region per verifier-call budget; ERASER is in Table~\ref{tab:esr-cross-benchmark}. Brackets: 95\% paired bootstrap CIs.}
\label{fig:experiments}
\end{figure*}

\subsection{Experimental Setup}
\label{sec:setup}

\paragraph{Data.}
We use $3{,}000$ label-free Natural Questions training documents for geometry \citep[NQ;][]{kwiatkowski2019natural}, the official public NQ DEV set ($7{,}830$ documents, $4.36$M candidates), and, for the isotonic and sequential mechanism analyses only, $600$ locked internal NQ DEV documents ($354{,}856$ candidates); official public DEV is the primary neural evaluation.
ERASER Evidence Inference \citep{lehman2019inferring,deyoung2020eraser} has alternative evidence groups, hence multiple targets; SciFact \citep{wadden2020fact} gives $209$ fresh held-out claim--abstract instances ($13{,}463$ sentence windows).
NQ and ERASER candidates are exhaustive multigranular families (paragraphs, sentences, and runs of $2$, $4$, or $8$ units), used for controlled geometry and prediction analysis; SciFact candidates are all contiguous sentence windows (Appendix~\ref{app:setup}). Section~\ref{sec:realgen} replaces them with retriever-generated hierarchies.

\paragraph{Predictors.}
On ERASER, BASE is a logistic regression over static span geometry and query--sentence similarities under a frozen encoder; BASE-H adds generic online-history counts; COMBINED$^*$ adds sixteen learned containment features of previously verified sub- and superspans; Closure+BASE applies the deterministic consequences of monotonicity, and RAND-CONT permutes the containment graph (Appendix~\ref{app:model}).
On NQ, A0 (semantic only), A1 (+ generic history), and A2 (+ containment history) are matched fine-tuned DeBERTa-v3-base models (three seeds; Appendix~\ref{app:lanea}).
The three models share the same $384$-token template and prediction head and differ only in their history inputs.
A candidate is \textbf{unresolved} if, when predicted, its label is not forced by a verified positive subspan or negative superspan; these candidates isolate learned signal from deterministic closure.

\subsection{Geometry of Real Candidate Families}
\label{sec:geometry}

\begin{table}[t]
\centering
\footnotesize
\setlength{\tabcolsep}{2.8pt}
\begin{tabular}{@{}lrrrrrr@{}}
\toprule
Data & $m$ & $\meff$ & $m/\meff$ & $w{=}\meff$ & Rel. & Err. \\
\midrule
NQ & 6336.5 & 90.5 & 58.37$\times$ & 93.43\% & 26.5M & 0 \\
ERASER & 4410 & 145 & 30.10$\times$ & 100\% & 127.6M & 0 \\
\bottomrule
\end{tabular}
\caption{Containment geometry.}
\label{tab:geom}
\end{table}

Endpoint-rank compression reduces the median document from $m=6{,}336.5$ tokens to $\meff=90.5$ on NQ ($58.37\times$) and from $4{,}410$ to $145$ on ERASER ($30.10\times$), and exhaustive checks confirm that $\psi$ preserves all $26.5$M and $127.6$M strict containments (Table~\ref{tab:geom}).
Width equals $\meff$ in $93.43\%$ of NQ and all ERASER documents, so the lower and upper bounds coincide up to logarithms; $27.48\%$ of ERASER instances have multiple targets.

\subsection{Containment Improves Online Prediction}
\label{sec:predictive}

\begin{table}[t]
\centering
\footnotesize
\setlength{\tabcolsep}{2.6pt}
\begin{tabular}{@{}lrrrrr@{}}
\toprule
ERASER & Macro$\downarrow$ & Raw$\downarrow$ & PRC$\uparrow$ & ROC$\uparrow$ & BAcc$\uparrow$ \\
\midrule
STATIC & .2896 & .0704 & .4447 & .8495 & .5928 \\
SEMANT. & .3038 & .0764 & .3384 & .8421 & .5439 \\
BASE & .2783 & .0691 & .4578 & .8700 & .5965 \\
BASE-H & .2732 & .0670 & .5004 & .8757 & .6017 \\
Closure+B. & \underline{.1283} & \underline{.0294} & \underline{.8843} & \underline{.9775} & \underline{.8594} \\
RAND-CONT & .4303 & .3576 & .1455 & .6302 & .5674 \\
\textbf{COMB.$^*$} & \textbf{.0455} & \textbf{.0115} & \textbf{.9803} & \textbf{.9973} & \textbf{.9427} \\
\multicolumn{6}{@{}l}{\scriptsize $\Delta$Macro vs BASE .2328 [.2268,.2389]; vs BASE-H} \\
\multicolumn{6}{@{}l}{\scriptsize .2277 [.2211,.2346]; vs RAND-CONT .3848 [.3819,.3875]} \\
\midrule
NQ & Brier$\downarrow$ & Unres.$\downarrow$ & PRC$\uparrow$ & ROC$\uparrow$ & ECE$\downarrow$ \\
\midrule
A0 semantic & .0277 & .0494 & .2156 & .9047 & .0071 \\
A1 +history & .0280 & .0503 & .1876 & .9037 & .0066 \\
\textbf{A2 +contain.} & \textbf{.0115} & \textbf{.0482} & \textbf{.8159} & \textbf{.9808} & \textbf{.0030} \\
\multicolumn{6}{@{}l}{\scriptsize $\Delta$(A1$-$A2): all .0166 [.0160,.0171]; unresolved .0020 [.0018,.0023]} \\
\bottomrule
\end{tabular}
\caption{Predictive results on ERASER test and official NQ DEV. Macro: macro balanced Brier; Unres.: unresolved Brier; PRC/ROC: AUPRC/AUROC; BAcc: balanced accuracy.}
\label{tab:main}
\end{table}

On ERASER test data, COMBINED$^*$ reduces macro balanced Brier from $0.278$ (BASE) to $0.045$, and its gain over generic-history BASE-H is $0.228$ (CI $[0.221, 0.235]$; Table~\ref{tab:main}, top).
Deterministic closure alone (Closure+BASE) accounts for $69\%$ of the raw-Brier reduction from BASE to COMBINED$^*$; on the $1.45$M unresolved predictions, where closure has nothing to propagate, COMBINED$^*$ still lowers raw Brier from $0.114$ to $0.045$ (Appendix~\ref{app:closure}). Permuting the containment graph (RAND-CONT) is worse than every static baseline.

\subsection{Deductive and Learned Gains Across Backbones}
\label{sec:neural}

On official NQ public DEV, the primary neural evaluation, A2 lowers document-macro Brier from $0.0280$ (A1) to $0.0115$ (A1$-$A2 $0.0166$, CI $[0.0160, 0.0171]$; Table~\ref{tab:main}, bottom; Fig.~\ref{fig:experiments}A), positive in all three seeds, while generic history alone does not help (A0 $0.0277$).
Almost all of this gain is deductive: closure contributes $0.0161$ and the learned unresolved component $0.0005$ (Appendix~\ref{app:lanea}).
On unresolved candidates the residual is smaller but reproducible: Brier falls from $0.0503$ to $0.0482$ ($0.0020$, $[0.0018, 0.0023]$), and with RoBERTa-base \citep{liu2019roberta} the unresolved A1$-$A2 gain is $0.00120$ ($[0.00096, 0.00145]$; Appendix~\ref{app:roberta}).

\subsection{Practical and Constructive Instantiation}
\label{sec:isotonic}

On the internal NQ DEV split, batch isotonic projection of a frozen semantic score (B1) lowers all-candidate Brier from $0.0087$ (B0) to $0.0085$, and the prequential isotonic learner (B2) to $0.0031$, with unresolved Brier $0.0228\to0.0225$ (Fig.~\ref{fig:experiments}B; Appendix~\ref{app:laneb}).
On $1{,}689$ real NQ, ERASER, and SciFact families ($|\mathcal C|\le3{,}885$, $\meff\le724$), exact projection keeps every containment constraint and stays lightweight (median $0.359$\,ms; Table~\ref{tab:computational-footprint}, Appendix~\ref{app:scaling}).

\paragraph{Direct theorem instantiation.}
The fixed-target regime is the one closest to single-answer QA, so we run its constructive learners on the $3{,}640$ official NQ DEV documents with an exact unique-target realization: $N_1(\mathcal C)\le ab+1\le\meff^2+1$ holds for every family (median $N_1=721$), and over $21{,}840$ distinct-query trajectories exponential weights satisfies $R_T\le2\ln N_1(\mathcal C)$ and the version-space learner $L_T\le\frac12\ln N_1(\mathcal C)$ on every trajectory (medians $4.32$, $1.80$; Appendix~\ref{app:fixedinst}).

\subsection{Sequential Mechanism}
\label{sec:sequential}

On internal NQ DEV trajectories of a frozen-BERT predictor pair with and without containment inputs, comparable verified history precedes $92.0\%$ of predictions and the containment gain concentrates there ($0.0159$, $[0.0142, 0.0177]$, $599$ documents), largest when a verified span forces the label (Appendix~\ref{app:seq}).
Swapping comparable spans' labels for non-comparable ones raises Brier by $0.3155$ on the $44.98\%$ of events where feasible: the signal comes from \emph{which} spans were verified.

\subsection{Robustness}
\label{sec:robust}

COMBINED$^*$ improves on BASE by $0.205$--$0.270$ under all four ERASER schedules, whose logically determined share ranges from $11.2\%$ to $88.4\%$ (Appendix~\ref{app:scheduler}), and permuting the containment graph (RAND-CONT) costs $0.385$ ($[0.382, 0.388]$).

\subsection{Decision-Aware Evidence Localization}
\label{sec:decision}

\begin{table}[t]
\centering
\small
\setlength{\tabcolsep}{3pt}
\begin{tabular}{@{}lrrrr@{}}
\toprule
Dataset & P-GREEDY & ESR & P-GREEDY$-$ESR$\uparrow$ & Calls$\downarrow$ \\
\midrule
SciFact & .526 & \textbf{.451} & \textbf{.075} [.059, .092] & 6$\rightarrow$\textbf{4} \\
ERASER & .470 & \textbf{.354} & \textbf{.116} [.098, .136] & 16$\rightarrow$\textbf{11} \\
\bottomrule
\end{tabular}
\caption{Evidence-length AUC ($\downarrow$), P-GREEDY$-$ESR ($\uparrow$ favors ESR; 95\% cluster-bootstrap CI), and median calls to the minimal oracle region.}
\label{tab:esr-cross-benchmark}
\end{table}

With a frozen, temperature-calibrated SciBERT predictor on held-out SciFact, ESR lowers evidence-length AUC from $0.526$ under probability-greedy acquisition (P-GREEDY) to $0.451$ (P-GREEDY$-$ESR $0.0753$, $[0.0593, 0.0919]$) and median calls to the minimal oracle region from $6$ to $4$; on $959$ ERASER instances the AUC falls from $0.470$ to $0.354$ ($0.116$, $[0.098, 0.136]$) and calls from $16$ to $11$ (Table~\ref{tab:esr-cross-benchmark}; Fig.~\ref{fig:experiments}C; Appendices~\ref{app:f2}, \ref{app:eraser-esr}).
Evidence-length AUC measures localization and verification efficiency, not task correctness; downstream readers are evaluated separately.
On SciFact, ESR also improves on the probability-free balanced split GEOM-BALANCE (length AUC $0.0637$, $[0.0486, 0.0790]$; Appendix~\ref{app:f2}), so probabilities matter beyond branch sizes.
Conversely, on ERASER decision frontiers COMBINED$^*$ improves Brier over BASE yet induces the identical candidate ranking (Appendix~\ref{app:f1diag}), so probability-greedy acquisition cannot use the better predictor; ESR, which weighs probabilities against branch sizes, chooses differently from P-GREEDY in $207$ of $212$ shared SciFact states ($97.6\%$) and $83.7\%$ of ERASER states, at a median decision cost of $0.162$\,ms (Appendix~\ref{app:scaling}).

\subsection{Validation with Retrieval-Generated Candidate Hierarchies}
\label{sec:realgen}

\begin{table}[t]
\centering
\footnotesize
\setlength{\tabcolsep}{1.6pt}
\begin{tabular*}{\columnwidth}{@{\extracolsep{\fill}}lrrrr@{}}
\toprule
DPR@8 & Nonroot+ & Min.\,tgt+ & Med.\,$|\mathcal C|$ & p90 \\
\midrule
NQ DEV & 70.69\% & 66.01\% & 40 & 56 \\
\multicolumn{5}{@{}l}{\scriptsize max $|\mathcal C|$ 332; comparable pairs / ${\ge}3$ levels: 99.97\% of covered} \\
\midrule
Predictor & Unres.$\downarrow$ & $-$A2$\uparrow$ & \multicolumn{2}{r@{}}{95\% CI} \\
\midrule
A0 semantic & .0762 & .0176 & \multicolumn{2}{r@{}}{\mbox{[.0168,.0184]}} \\
A1 +history & .0647 & .0062 & \multicolumn{2}{r@{}}{\mbox{[.0057,.0067]}} \\
\textbf{A2 +contain.} & \textbf{.0586} & -- & \multicolumn{2}{r@{}}{--} \\
\midrule
Verification & Base & ESR & \multicolumn{2}{r@{}}{Gain$\uparrow$ [95\% CI]} \\
\midrule
Len.\,AUC$\downarrow$ (LS) & .2560 & .2473 & \multicolumn{2}{r@{}}{\mbox{.0087 [.0055,.0121]}} \\
Len.\,AUC$\downarrow$ (GG) & .2752 & .2473 & \multicolumn{2}{r@{}}{\mbox{.0279 [.0213,.0346]}} \\
S@8$\uparrow$ (GG) & .6310 & .6419 & \multicolumn{2}{r@{}}{\mbox{.0109 [.0031,.0189]}} \\
\midrule
Recent sel. & Call$\downarrow$ & \multicolumn{3}{r@{}}{Method$-$ESR$\uparrow$ [95\% CI]} \\
\midrule
NEST-Adapt & .3085 & \multicolumn{3}{r@{}}{\mbox{.0612 [.0548,.0677]}} \\
REARANK & .3239 & \multicolumn{3}{r@{}}{\mbox{.0766 [.0697,.0835]}} \\
Cost-MMR-Adapt & .6758 & \multicolumn{3}{r@{}}{\mbox{.4285 [.4208,.4360]}} \\
RCD-Adapt & .8639 & \multicolumn{3}{r@{}}{\mbox{.6166 [.6093,.6240]}} \\
ESR & \textbf{.2473} & \multicolumn{3}{r@{}}{--} \\
\midrule
Reader F1 & AUC$\uparrow$ & @8$\uparrow$ & \multicolumn{2}{r@{}}{Gain$\uparrow$ [95\% CI]} \\
\midrule
LEARN.-SCHED & .3528 & .4392 & \multicolumn{2}{r@{}}{\mbox{.0046 [.0022,.0070]}} \\
GEN.-GREEDY & .3509 & .4381 & \multicolumn{2}{r@{}}{\mbox{.0065 [.0017,.0113]}} \\
ESR & .3574 & .4467 & \multicolumn{2}{r@{}}{--} \\
\bottomrule
\end{tabular*}
\caption{Frozen DPR ($k=8$) hierarchies on NQ public DEV. Prediction and length AUC use $3{,}887$ generated-covered documents; SUCCESS@8 (S@8) keeps all $5{,}499$ answerable ones, scoring retrieval misses zero. LS: TRAIN-only learned scheduler; GG: generator-greedy; Gain is base$-$ESR for length AUC and ESR$-$base for S@8. Recent sel.: verifier-call evidence-length AUC (lower is better). Reader F1: call-budget F1-AUC and F1@8 of the frozen DPR reader (gain ESR$-$comparator) on the official short-answer/yes-no subset ($N=4{,}460$).}
\label{tab:real-generator}
\end{table}

We finally replace the exhaustive families with realistic retriever-generated pools: the document root, the top-$8$ paragraphs of a frozen NQ-trained DPR retriever \citep{karpukhin2020dense}, and their sentences, on $5{,}499$ answerable NQ DEV documents.
DPR yields non-root positive proposals for $70.7\%$ of documents (median pool $40$), and $99.97\%$ of covered documents keep comparable candidates and ${\ge}3$ hierarchy levels (Table~\ref{tab:real-generator}); without retraining, A2 improves unresolved Brier over A1 and A0, and ESR improves length AUC over the TRAIN-only learned scheduler and generator-greedy verification and raises coverage-aware SUCCESS@8 over generator-greedy.
On the same frozen pools and exact containment transitions, we compare ESR with four runnable recent external selectors: listwise reasoning reranking, nested evidence survival, budgeted routing, and cost-aware selection, each adapted as documented in Table~\ref{tab:recent-fidelity}.
Under verifier-call budgets, ESR has lower evidence-length AUC than each of the four; the strongest, NEST-Adapt, scores $0.3085$ versus ESR's $0.2473$ (paired difference $0.0612$, 95\% CI $[0.0548, 0.0677]$).
With a frozen NQ-trained DPR reader on the official public-DEV short-answer/yes-no subset ($N=4{,}460$), ESR improves call-budget F1-AUC over the TRAIN-only learned scheduler by $0.0046$ ($[0.0022, 0.0070]$) and over generator-greedy by $0.0065$ ($[0.0017, 0.0113]$); at eight calls ESR's F1 is $0.4467$, versus $0.4392$ and $0.4381$; these reader gains are small and specific to this setting.
Token-budget, success, BM25 and cutoff sensitivity, and full reader analyses are in Appendices~\ref{app:realgen}, \ref{app:recent}, and~\ref{app:downstream} \citep{robertson2009probabilistic}.

\section{Conclusion}
\label{sec:conclusion}

Nested evidence verification is exactly distinct-query online prediction on a two-dimensional order scaled by observed endpoints; the multi-target envelope is near-linear, while a single fixed target, the standard QA case, is logarithmic with a constructive learner.
Empirically, containment acts through exact deduction plus a smaller reproducible learned signal, and ESR, a one-step structural rule, turns them into fewer verifier calls on SciFact, ERASER, and DPR hierarchies; on DPR hierarchies under call budgets it outperforms four recent external selectors and gives small fixed-reader F1 gains over learned and generator-greedy scheduling.

\section*{Limitations}

The general-class upper bound is non-constructive and tight up to logarithmic factors; a constructive near-optimal multi-target learner remains open.
OECP assumes exact verifier feedback; robustness to noisy or model-based verification remains open.
ESR is a one-step rule without multi-step guarantees, and all evaluations are in English.

\bibliographystyle{icml2026}
\bibliography{oecp}

\clearpage
\appendix
\setcounter{secnumdepth}{3}

\onecolumn
\section{Notation}
\label{app:notation}

Table~\ref{tab:notation} collects the notation of the main text (Sections~\ref{sec:formulation}--\ref{sec:experiments}) and of the proofs in Appendices~\ref{app:complete-proofs} and~\ref{app:single}.
Symbols marked \emph{(proof)} appear only in the appendix proofs.

\begingroup
\captionsetup[table]{position=above}
\small
\setlength{\tabcolsep}{4pt}
\renewcommand{\arraystretch}{1.08}
\begin{longtable}{@{}>{\raggedright\arraybackslash}p{0.34\textwidth}p{0.62\textwidth}@{}}
\caption{Notation used in the paper.}\label{tab:notation}\\
\toprule
Symbol & Meaning \\
\midrule
\endfirsthead
\multicolumn{2}{@{}l}{\tablename~\thetable\ (continued)}\\
\toprule
Symbol & Meaning \\
\midrule
\endhead
\midrule
\multicolumn{2}{r@{}}{\emph{continued on next page}}\\
\endfoot
\bottomrule
\endlastfoot

\multicolumn{2}{@{}l}{\textbf{Evidence verification (Section~\ref{sec:oecp})}}\\
$q$, $d$ & Query (question or claim) and document. \\
$m$ & Raw document length in tokens; also the side of the full grid $[m]^2$. \\
$s=[l,r)$ & Candidate span with left endpoint $l$ and right endpoint $r$, $0\le l<r\le m$. \\
$\mathcal S_m$ & All spans of a document of $m$ tokens. \\
$\mathcal C$ & Finite candidate family, $\mathcal C\subseteq\mathcal S_m$, proposed independently of the gold evidence. \\
$G$, $g$ & Hidden set of admissible target spans; one target. \\
$Y_s=\1\{\exists g\in G:g\subseteq s\}$ & Containment label: $1$ iff $s$ contains a complete target. \\
$t$, $T$ & Round index; number of (distinct) verification rounds. \\
$s_t$ & Candidate verified at round $t$. \\
$H_t$ & Verification history (verified candidates and labels) before round $t$. \\
unresolved & A candidate whose label is not forced by a verified positive subspan or negative superspan when it is predicted. \\

\multicolumn{2}{@{}l}{\textbf{Online prediction and regret (Sections~\ref{sec:oecp}--\ref{sec:minimax})}}\\
$(P,\preceq)$ & Finite partially ordered set; $\preceq$ is also the coordinatewise order on grids. \\
$\Fiso(P)$ & $[0,1]$-valued isotonic (order-preserving) functions on $P$. \\
$x_t$, $\hat y_t$, $y_t$ & Queried point, learner prediction, and revealed outcome in the online game. \\
$R_T$ & Cumulative squared-loss regret against the best comparator on the realized trajectory. \\
$\Rdist(P)$ & Distinct-query minimax regret on $P$ (each point queried at most once on every path). \\
$\mathcal A$, $\mathcal E_{\mathrm{dist}}$ & Learners; adaptive distinct-query adversaries. \\
$\Gamma_{m,T}=m[1+\log(em)\log(1+T/m)]$ & Scale of the upper bound in Theorem~\ref{thm:upper}. \\
$C$ & Universal constant in Theorems~\ref{thm:upper} and~\ref{thm:compression}. \\
$\widetilde\Theta$, $\widetilde O$ & Rates up to polylogarithmic factors. \\

\multicolumn{2}{@{}l}{\textbf{Span geometry and compression (Sections~\ref{sec:sparse}--\ref{sec:transfer})}}\\
$[m]^2$ & Full $m\times m$ grid with the coordinatewise order. \\
$\phi([l,r))=(m-l,r)$ & Conceptual span-to-grid map (Proposition~\ref{prop:span-order}). \\
$D_m=\{(u,v)\in[m]^2:u+v\ge m+1\}$ & Triangle of grid points that represent legal spans. \\
$L$, $R$; $a=|L|$, $b=|R|$ & Observed left and right endpoints of $\mathcal C$ and their numbers. \\
$\rank_L$, $\rank_R$ & Ascending ranks within $L$ and $R$. \\
$\psi([l,r))=(a-\rank_L(l)+1,\rank_R(r))$ & Endpoint-rank compression of a finite family (Theorem~\ref{thm:compression}). \\
$\meff=\max(a,b)$ & Effective scale: the larger number of distinct observed endpoints; $\meff\le m$. \\
$w=\width(P)$ & Width: size of the largest antichain of $P$. \\

\multicolumn{2}{@{}l}{\textbf{Fixed-target OECP (Section~\ref{sec:single}, Appendix~\ref{app:single})}}\\
$h_g(s)=\1\{g\subseteq s\}$ & Labeling induced by one fixed target $g$. \\
$\mathcal H_{\rm one}(\mathcal C)$, $N_1(\mathcal C)$ & Distinct fixed-target labelings of $\mathcal C$ and their number. \\
$\mathcal C_K$ & Hard family with $\meff=K$ and $|\mathcal C_K|=K^2$ in Theorem~\ref{thm:single}. \\
$g^\star$ & Target fixed before play in the realizable game. \\
$L_T$ & Cumulative square loss (realizable game; also the version-space learner). \\
$n=|\mathcal C|$, $K=\meff$ & Family size; effective scale in the fixed-target analysis. \\
$\mathcal J(\mathcal C)$ & \emph{(proof)} Nonempty intersections of at most two candidates, used as canonical targets. \\
$\lambda_i$, $\rho_j$; $U_i$, $V_j$ & \emph{(proof)} Sorted left/right endpoints; $U_i=\max\{r:[\lambda_i,r)\in\mathcal C\}$ and $V_j=\min\{l:[l,\rho_j)\in\mathcal C\}$ for enumerating $\mathcal J(\mathcal C)$. \\
$\delta=\1\{[0,m)\notin\mathcal C\}$ & \emph{(proof, Part II)} Indicator that the all-zero fixed-target signature is realizable. \\
$V^{\rm ag}_{T,\rm dist}$, $V^{\rm real}_{T,\rm dist}$ & \emph{(proof)} Fixed-target minimax values in the agnostic and realizable games. \\
$W^{\rm ag}(K,T)$, $W^{\rm real}(K,T)$ & \emph{(proof)} Worst-case values over families of effective scale $K$. \\
$w_t(h)$, $\eta$ & \emph{(proof)} Exponential-weights mass of labeling $h$; learning rate. \\
$\mathcal V_t$ & \emph{(proof)} Version space of labelings consistent with labels before round $t$. \\
$\operatorname{Ldim}$ & \emph{(proof)} Littlestone dimension of the fixed-target class. \\
$\log$, $\ln$, $\log_2$ & Natural logarithms ($\log$, $\ln$); base-two logarithm ($\log_2$). \\

\multicolumn{2}{@{}l}{\textbf{General upper and lower bounds (Appendix~\ref{app:complete-proofs}; proof only)}}\\
$R_{m,T}$ & Distinct-query minimax regret on the full grid, $\Rdist([m]^2)$. \\
$\mathcal F_m^\uparrow$ & $[0,1]$-valued isotonic functions on $[m]^2$. \\
$\mathbf x$; $\mathfrak X_T^{\mathrm{dist}}$ & Adaptive covariate (query) tree; depth-$T$ trees that are pathwise distinct. \\
$\mathcal A_t(x_{1:t-1})$ & Legal next queries: $[m]^2\setminus\{x_1,\ldots,x_{t-1}\}$. \\
$\boldsymbol\mu$ & $[-1,1]$-valued tree used in the minimax-to-offset reduction. \\
$\mathcal R_t(h_{t-1})$ & Continuation value in the legal-query backward recursion (Lemma B.2). \\
$\epsilon\in\{\pm1\}^T$ & Path through a binary tree. \\
$N_2(\beta,\mathcal F_m^\uparrow,\mathbf x)$, $N^{\mathrm{dist}}_{2,\mathrm{seq}}(\beta,T)$ & Sequential $L_2$ covering number on tree $\mathbf x$; its supremum over distinct trees. \\
$H_{\mathrm{dist}}(\beta,T)$ & Logarithm of $N^{\mathrm{dist}}_{2,\mathrm{seq}}(\beta,T)$. \\
$d_{\mathrm{full}}$, $N_{2,\mathrm{full}}(\alpha)$ & Normalized $L_2$ metric on the full grid; static covering number at scale $\alpha$. \\
$\mathcal G$, $\mathcal G_{\mathrm{off}}$ & Static full-grid cover; centered class $g_f(x,\mu)=f(x)-\mu$ in the fixed-tree chaining step. \\
$\alpha,\beta,\gamma,\varepsilon,\rho$ & Covering and chaining scales; $\varepsilon$ is the integration variable, distinct from the path sign $\epsilon$. \\

\multicolumn{2}{@{}l}{\textbf{Decision-aware verification (Section~\ref{sec:esr})}}\\
$p_t(s)$ & Predicted containment probability of $s$ before its verification, used by ESR. \\
$q_t(s)$ & Raw probability from the predictor $P_\theta$, before isotonic correction. \\
$\mathcal I(q_t,H_t)$, $\lambda$ & Isotonic correction of $q_t$ given the history; weight on the raw scores. \\
$\mathcal V$ & Exact verifier returning $Y_s$. \\
$A_t$, $N_t=|A_t|$ & Active set of feasible unverified candidates; its size. \\
$R_t$ & Verified-positive incumbent (current localized evidence). \\
$B_t(s)$, $b_t(s)=|B_t(s)|$ & Active candidates contained in $s$; their number. \\
$\mathcal L_t(s)$ & Expected number of active candidates after verifying $s$. \\
$B$, $E$, $\mathrm{Sub}(s)$ & Verification budget, Hasse edges, and candidates contained in $s$ (Algorithm~\ref{alg:oecp-pipeline}). \\

\multicolumn{2}{@{}l}{\textbf{Models and policies (Section~\ref{sec:experiments})}}\\
A0, A1, A2 & NQ DeBERTa / RoBERTa predictors: semantic only; + generic history; + containment history. \\
B0, B1, B2 & Frozen semantic score; batch isotonic projection; prequential online isotonic learner. \\
BASE, BASE-H, COMBINED$^*$ & ERASER predictors: static; + generic history; + learned containment features. \\
Closure+BASE, RAND-CONT & BASE with forced labels applied; COMBINED$^*$ with a permuted containment graph. \\
N2, FULL; G1, G4 & Earlier NQ predictors (semantic; + containment features) and their feature groups (Appendix~\ref{app:neural}). \\
S0--S3 & History states: no comparable history; comparable, not forcing; forced by a positive subspan; forced by a negative superspan. \\
P-GREEDY, GEOM-BALANCE & Verify the maximum-$p_t(s)$ candidate; minimize the larger branch without probabilities. \\
ESR & Verify $\arg\min_{s\in A_t}\mathcal L_t(s)$. \\
LS, GG & LEARNED-SCHED (TRAIN-only learned scheduler) and GENERATOR-GREEDY (generator-score order), Section~\ref{sec:realgen}. \\
\end{longtable}
\endgroup

\clearpage
\twocolumn

\section{Complete Proofs}
\label{app:complete-proofs}

\noindent
{\small\setlength{\tabcolsep}{3pt}\renewcommand{\arraystretch}{1.12}
\begin{tabular}{@{}>{\raggedright\arraybackslash}p{.24\columnwidth}>{\raggedright\arraybackslash}p{.22\columnwidth}>{\raggedright\arraybackslash}p{.46\columnwidth}@{}}
\toprule
Main claim & Proof & Key ingredients \\
\midrule
Thm.~\ref{thm:upper}: upper bound & \ref{app:legal-trees}--\ref{app:upper-bound} & Legal trees; restricted chaining; radius rescaling; isotonic entropy \\
Cor.~\ref{cor:rate}: rate & \ref{app:lower-bound} & Central antichain; zero-loss isotonic comparator \\
Prop.~\ref{prop:span-order}: geometry & \ref{app:span-reduction} & Span/grid isomorphism; triangle extension \\
Thm.~\ref{thm:compression}: compression & \ref{app:endpoint-compression} & Endpoint ranks; lower-set extension \\
Width lower bound & \ref{app:width-lower} & Maximum antichain \\
Multi-target realization & \ref{app:oecp-lower-semantics} & One-token targets realize antichain bits \\
Thm.~\ref{thm:single}: fixed target & \ref{app:single-compression}--\ref{app:single-lower} & Target equivalence; EW; version space; threshold lower bound \\
Cor.~\ref{cor:separation}: separation & \ref{app:single-separation} & Same $\mathcal C_K$ in both regimes \\
\bottomrule
\end{tabular}}

\subsection{Part I: The General Multi-Target Class}
\label{app:proofs-general}

\paragraph{Proof architecture.}
\label{app:proof-dependencies}
\emph{Imported facts:} the minimax-to-offset reduction and fixed-tree
sequential chaining inequality of \citet{rakhlin2014online}, and the bounded
bivariate isotonic entropy estimate of \citet{chatterjee2018matrix}.
\emph{New argument for distinct queries:} preserving the history-dependent
legal query set through the reduction, and the static-to-sequential radius
rescaling enabled by pathwise distinctness (Lemma~B.3).
\[
\begin{array}{c}
\text{restricted chaining (B.2)}\\
+\;\text{radius rescaling (B.3)}\\
+\;\text{static entropy (B.4)}\\
\Downarrow\\
\text{sequential entropy (B.5)}\;\Longrightarrow\;\text{upper (B.6)}\\[1mm]
\text{central antichain}\;\Longrightarrow\;\text{lower (B.7)}\\[1mm]
\text{span/grid transfer + compression}\\
\text{(B.9--B.15)}\\
\Downarrow\\
\text{OECP bounds; width and realization}\\
\text{(B.16--B.18)}
\end{array}
\]

\subsubsection{Notation and the distinct-query game}
\label{app:proof-notation}

For $m\ge2$, let
\[
[m]^2=\{1,\ldots,m\}^2
\]
with the coordinate-wise order
\[
(i,j)\preceq(i',j')
\quad\Longleftrightarrow\quad
i\le i'
\ \text{and}\
j\le j'.
\]
Let
\[
\begin{aligned}
\mathcal F_m^\uparrow
=
\bigl\{
f:[m]^2\to[0,1]:{}
&x\preceq x'\\
&\Longrightarrow f(x)\le f(x')
\bigr\}.
\end{aligned}
\]

For a finite poset $P$, the distinct-query game proceeds for
$t=1,\ldots,T$ as follows.  The adversary chooses
\[
x_t\in P\setminus\{x_1,\ldots,x_{t-1}\},
\]
the learner predicts $\widehat y_t\in[0,1]$, the adversary reveals
$y_t\in[0,1]$, and the learner incurs
$(\widehat y_t-y_t)^2$.  The regret is
\[
R_T
=
\sum_{t=1}^T(\widehat y_t-y_t)^2
-
\min_{f\in\mathcal F^\uparrow(P)}
\sum_{t=1}^T(f(x_t)-y_t)^2.
\]
The minimax value is denoted by
$R^*_{T,\mathrm{dist}}(P)$, and for the full grid we write
\[
R_{m,T}
=
R^*_{T,\mathrm{dist}}([m]^2).
\]

A depth-$T$ covariate tree
$\mathbf x=(x_t)_{t=1}^T$ is a sequence of maps
\[
x_t:\{-1,+1\}^{t-1}\to[m]^2.
\]
For a path
$\epsilon=(\epsilon_1,\ldots,\epsilon_T)\in\{-1,+1\}^T$,
write
\[
x_t(\epsilon)
:=
x_t(\epsilon_{1:t-1}).
\]
We call $\mathbf x$ \emph{pathwise distinct} if for every path $\epsilon$,
\[
x_s(\epsilon)\ne x_t(\epsilon)
\qquad
\text{whenever }s\ne t.
\]
Let $\mathfrak X_T^{\mathrm{dist}}$ denote the set of all pathwise-distinct
depth-$T$ covariate trees.

For a fixed covariate tree $\mathbf x$, a finite collection $\mathcal V$ of
real-valued depth-$T$ trees is a sequential $L_2$ $\beta$-cover of
$\mathcal F_m^\uparrow$ on $\mathbf x$ if
\[
\forall f\in\mathcal F_m^\uparrow,\quad
\forall\epsilon\in\{-1,+1\}^T,\quad
\exists v\in\mathcal V
\]
such that
\[
\left[
\frac1T\sum_{t=1}^T
\left(
f(x_t(\epsilon))-v_t(\epsilon)
\right)^2
\right]^{1/2}
\le\beta.
\]
Let
\[
N_2(\beta,\mathcal F_m^\uparrow,\mathbf x)
\]
be the minimum cardinality of such a cover, and define
\[
N^{\mathrm{dist}}_{2,\mathrm{seq}}(\beta,T)
=
\sup_{\mathbf x\in\mathfrak X_T^{\mathrm{dist}}}
N_2(\beta,\mathcal F_m^\uparrow,\mathbf x),
\]
\[
H_{\mathrm{dist}}(\beta,T)
=
\log
N^{\mathrm{dist}}_{2,\mathrm{seq}}(\beta,T).
\]

All logarithms in Part~I are natural.  On the full grid define
\[
d_{\mathrm{full}}(f,g)
=
\left[
\frac1{m^2}
\sum_{x\in[m]^2}
(f(x)-g(x))^2
\right]^{1/2},
\]
and let $N_{2,\mathrm{full}}(\alpha)$ be the covering number of
$\mathcal F_m^\uparrow$ under $d_{\mathrm{full}}$.

\subsubsection{The adaptive legality constraint and pathwise-distinct trees}
\label{app:legal-trees}

The first point that must be checked is that the standard sequential-tree
representation does not silently re-introduce repeated covariates.

\paragraph{Lemma B.1 (Legal-tree characterization).}
A covariate strategy is legal for the distinct-query game if and only if its
binary-tree representation is pathwise distinct.

\paragraph{Proof.}
At a node on level $t$, the ancestors of that node record the covariates
chosen on the unique root-to-node history.  The distinct-query rule says
exactly that the covariate placed at the node must differ from all those
ancestor covariates.  Therefore every root-to-leaf path contains pairwise
distinct covariates.  This is precisely pathwise distinctness.

Conversely, if a covariate tree is pathwise distinct, then at every node its
label differs from every ancestor label on that node's unique path.
Therefore choosing that node label at the corresponding history is legal
under the distinct-query rule. \qed

\paragraph{Range normalization.}
The chaining result of \citet{rakhlin2014online} is stated for a comparator
class contained in $[-1,1]^{\mathcal X}$ and responses in a symmetric
interval $[-B,B]$.  Our comparator class is contained in $[0,1]^{[m]^2}$,
hence also in $[-1,1]^{[m]^2}$, and our responses lie in the smaller set
$[0,1]\subset[-1,1]$.  We may therefore upper bound our game by the
corresponding game with responses in $[-1,1]$, i.e.\ with $B=1$.
If the resulting strategy predicts outside $[0,1]$, project its prediction
onto $[0,1]$ when it is used in the original game.  For every
$y\in[0,1]$, Euclidean projection cannot increase squared loss:
\[
\bigl(\Pi_{[0,1]}(u)-y\bigr)^2
\le
(u-y)^2.
\]
The comparator loss is unchanged.  Thus any regret upper bound obtained for
the enlarged response range yields a valid upper bound for the original
$[0,1]$ game.

\subsubsection{Sequential chaining with the distinct-query constraint}
\label{app:restricted-rs}

We next record the precise form of the square-loss chaining inequality needed
below.

\paragraph{Lemma B.2 (Restricted-tree square-loss chaining).}
For every $m\ge2$, $1\le T\le m^2$, and $\gamma\in(0,1]$,
\begin{align}
R_{m,T}
\le\;&
32\log
N^{\mathrm{dist}}_{2,\mathrm{seq}}(\gamma,T)
\nonumber\\
&+
\inf_{\rho\in(0,\gamma)}
\Bigl\{
4\rho T
\nonumber\\
&\quad+
12\sqrt T
\int_\rho^\gamma
\sqrt{
H_{\mathrm{dist}}(\varepsilon,T)
}\,d\varepsilon
\Bigr\}.
\label{eq:distinct-chaining}
\end{align}

\paragraph{Proof.}
For a history $h_{t-1}=((x_s,y_s))_{s<t}$, the enlarged response game
has legal set
\[
\mathcal A_t(x_{1:t-1})
= [m]^2\setminus\{x_1,\ldots,x_{t-1}\}.
\]
Its continuation value satisfies the backward recursion
\[
\begin{aligned}
\mathcal R_t(h_{t-1})
={}&\sup_{x_t\in\mathcal A_t(x_{1:t-1})}
\inf_{p_t\in[-1,1]}\sup_{y_t\in[-1,1]}\\
&\bigl\{(p_t-y_t)^2+
\mathcal R_{t+1}(h_{t-1},x_t,y_t)\bigr\},
\end{aligned}
\]
with terminal value
$\mathcal R_{T+1}(h_T)=-\min_f\sum_{s=1}^T(f(x_s)-y_s)^2$,
where $f\in\mathcal F_m^\uparrow$.
Thus the legal covariate supremum precedes the prediction/response minimax
step.  In Lemma~7 and Eq.~(12) of \citet{rakhlin2014online}, the minimax
interchange acts only on that prediction/response step after $x_t$ is fixed;
response symmetrization changes response-side variables only.  Predictable
sign changes merely relabel the children of a node and preserve its ancestor
covariates.  Backward iteration therefore keeps the same legal supremum at
each node, indexed by that node's ancestor list.  None of these operations
enlarges $\mathcal A_t$.  The resulting covariate trees are precisely
$\mathfrak X_T^{\mathrm{dist}}$ by Lemma~B.1.

More explicitly, after the same symmetrization as in their Lemma~7, one
obtains
\begin{align}
R_{m,T}
&\le
\sup_{\mathbf x\in\mathfrak X_T^{\mathrm{dist}}}
\sup_{\boldsymbol\mu}
\mathbb E_\epsilon
\sup_{f\in\mathcal F_m^\uparrow}
\sum_{t=1}^T
\Bigl[
\nonumber\\
&\qquad
4\epsilon_t
\bigl(f(x_t(\epsilon))-\mu_t(\epsilon)\bigr)
\nonumber\\
&\qquad-
\bigl(f(x_t(\epsilon))-\mu_t(\epsilon)\bigr)^2
\Bigr],
\label{eq:restricted-offset}
\end{align}
where $\boldsymbol\mu$ ranges over $[-1,1]$-valued trees.
The only difference from the unrestricted formula is the domain of the
supremum over $\mathbf x$.

The second imported fact is the following fixed-tree specialization of
Lemma~8 of \citet{rakhlin2014online}: for square loss with response and
comparator ranges contained in $[-1,1]$, fix one
$\mathbf x\in\mathfrak X_T^{\mathrm{dist}}$
and one $\boldsymbol\mu$.  Form the product-valued tree
\[
z_t(\epsilon)
=
\bigl(x_t(\epsilon),\mu_t(\epsilon)\bigr)
\]
and the class
\[
\mathcal G_{\mathrm{off}}
=
\left\{
g_f:\;
g_f(x,\mu)=f(x)-\mu,\quad
f\in\mathcal F_m^\uparrow
\right\}.
\]
For this fixed tree, the sequential $L_2$ covering number of
$\mathcal G_{\mathrm{off}}$ on $\mathbf z$ equals that of
$\mathcal F_m^\uparrow$ on $\mathbf x$, because subtraction of the same
$\mu_t(\epsilon)$ from both a target and a cover tree preserves every
pointwise difference.  In the notation of that lemma, take
$n=T$, $A=2$, $B=1$, and multiplier tree $\eta_t\equiv1$;
indeed $\mathcal G_{\mathrm{off}}\subseteq[-2,2]$ and $0<\gamma\le1\le A$.
Its specialized inequality, pointwise in the fixed $\mathbf x$ and
$[-1,1]$-valued $\boldsymbol\mu$, is
\begin{align*}
&\mathbb E_\epsilon
\sup_{f\in\mathcal F_m^\uparrow}
\sum_{t=1}^T
\Bigl[
4\epsilon_t
\bigl(f(x_t(\epsilon))-\mu_t(\epsilon)\bigr)
\\
&\qquad-
\bigl(f(x_t(\epsilon))-\mu_t(\epsilon)\bigr)^2
\Bigr]
\\
&\le
32\log N_2(\gamma,\mathcal F_m^\uparrow,\mathbf x)
+
\inf_{\rho\in(0,\gamma)}
\Bigl\{
4\rho T
\\
&\quad+
12\sqrt T
\int_\rho^\gamma
\sqrt{
\log N_2(\varepsilon,\mathcal F_m^\uparrow,\mathbf x)
}\,d\varepsilon
\Bigr\}.
\end{align*}
Taking the supremum over legal $\mathbf x$ and using
\[
\log N_2(\varepsilon,\mathcal F_m^\uparrow,\mathbf x)
\le
H_{\mathrm{dist}}(\varepsilon,T)
\]
for every such $\mathbf x$ gives
\eqref{eq:distinct-chaining}.  No step enlarges a legal distinct-query tree
to a tree containing a repeated covariate. \qed

\subsubsection{Static-to-sequential radius rescaling}
\label{app:radius-rescaling}

\paragraph{Lemma B.3 (Static-to-sequential radius rescaling).}
For every $m\ge2$, $1\le T\le m^2$, and $\beta>0$,
\[
N^{\mathrm{dist}}_{2,\mathrm{seq}}(\beta,T)
\le
N_{2,\mathrm{full}}
\left(
\beta\frac{\sqrt T}{m}
\right).
\]

The cover quantifiers are, respectively,
\[
\underbrace{\forall f\ \exists g_j}_{\text{static}}
\qquad\text{and}\qquad
\underbrace{\forall f\ \forall\epsilon\ \exists v^j}_{\text{sequential}}.
\]

\paragraph{Proof.}
Set
\[
N=m^2,
\qquad
\alpha
=
\beta\sqrt{\frac{T}{N}}
=
\beta\frac{\sqrt T}{m}.
\]
Let
\[
\mathcal G=\{g_1,\ldots,g_M\},
\qquad
M=N_{2,\mathrm{full}}(\alpha),
\]
be an $\alpha$-cover of $\mathcal F_m^\uparrow$ under
$d_{\mathrm{full}}$.

Fix an arbitrary
$\mathbf x\in\mathfrak X_T^{\mathrm{dist}}$.
For every $g_j\in\mathcal G$, define a depth-$T$ prediction tree
\[
v_t^j(\epsilon)
=
g_j(x_t(\epsilon)).
\]
We claim that
\[
\mathcal V_{\mathbf x}
=
\{v^1,\ldots,v^M\}
\]
is a sequential $L_2$ $\beta$-cover of
$\mathcal F_m^\uparrow$ on $\mathbf x$.
The collection $\mathcal V_{\mathbf x}$ may depend on the fixed covariate
tree $\mathbf x$, as permitted by $N_2(\beta,\mathcal F_m^\uparrow,\mathbf x)$.

Fix $f\in\mathcal F_m^\uparrow$.  Because $\mathcal G$ is an
$\alpha$-cover, choose $g_j$ satisfying
\[
\frac1N
\sum_{x\in[m]^2}
(f(x)-g_j(x))^2
\le
\alpha^2.
\]
Crucially, this choice of $g_j$ depends on $f$ but not on the path
$\epsilon$.

Now fix an arbitrary path $\epsilon$.
By pathwise distinctness, the queried points
\[
x_1(\epsilon),\ldots,x_T(\epsilon)
\]
are $T$ distinct elements of $[m]^2$.  Therefore their nonnegative squared
errors form a sub-sum of the full-grid squared-error sum:
\begin{align*}
&\frac1T
\sum_{t=1}^T
\bigl(
f(x_t(\epsilon))
-
v_t^j(\epsilon)
\bigr)^2
\\
&\quad=
\frac1T
\sum_{t=1}^T
\bigl(
f(x_t(\epsilon))
-
g_j(x_t(\epsilon))
\bigr)^2
\\
&\quad\le
\frac1T
\sum_{x\in[m]^2}
(f(x)-g_j(x))^2
\le
\frac{N\alpha^2}{T}
=
\beta^2.
\end{align*}
Thus the same cover element $v^j$ works for every path of the fixed
covariate tree.  In particular, the usual sequential-cover quantifier
\[
\forall f,\ \forall\epsilon,\ \exists v
\]
is satisfied.  Hence
\[
N_2(\beta,\mathcal F_m^\uparrow,\mathbf x)
\le
M.
\]
Taking the supremum over
$\mathbf x\in\mathfrak X_T^{\mathrm{dist}}$
proves the result. \qed

\paragraph{Remark.}
Repeated queries break exactly the sub-sum inequality: a revisited point
can contribute more than once to the pathwise error sum.

\subsubsection{Normalized full-grid entropy}
\label{app:static-entropy}

We next convert the bounded bivariate-isotonic Frobenius entropy bound of
\citet{chatterjee2018matrix} to the normalized empirical metric used here.

\paragraph{Lemma B.4 (Normalized full-grid entropy).}
There is a universal constant $C_0>0$ such that for every
$m\ge2$ and $0<\alpha\le1/2$,
\[
\log N_{2,\mathrm{full}}(\alpha)
\le
C_0
\alpha^{-2}
\log^2\left(\frac{e}{\alpha}\right).
\]

\paragraph{Proof.}
Identify each
$f\in\mathcal F_m^\uparrow$
with its $m\times m$ coordinate-wise nondecreasing matrix of values.
Because $0\le f(i,j)\le1$, this class is a subset of the bounded isotonic
matrix class
\[
\begin{aligned}
B_\infty(\mathbf 0,1)
=
\bigl\{
\theta:{}
&\theta\text{ is coordinate-wise}\\
&\text{nondecreasing},
\ \|\theta\|_\infty\le1
\bigr\}
\end{aligned}
\]
used in Lemma~3.4 of \citet{chatterjee2018matrix}.

Since $d_{\mathrm{full}}(f,g)=\|f-g\|_F/m$, the parameter mapping, in the notation of the cited result, is
\[
n=m^2,\qquad \varepsilon_F=\sqrt n\,\alpha=m\alpha,
\qquad t=1.
\]
For the admissible regime $0<\alpha\le1/2$, Lemma~3.4 of
\citet{chatterjee2018matrix} specializes to
\[
\begin{aligned}
&\log N\bigl(m\alpha,B_\infty(\mathbf0,1),\|\cdot\|_F\bigr)\\
&\qquad\le C\alpha^{-2}\log^2(1/\alpha),
\end{aligned}
\]
with entropy ratio $t\sqrt n/\varepsilon_F=1/\alpha\ge2$.
The original class is already contained in $B_\infty(\mathbf0,1)$.
Hence
\[
\log N_{2,\mathrm{full}}(\alpha)
\le
C\alpha^{-2}\log^2(1/\alpha).
\]
For $0<\alpha\le1/2$,
\[
\log(1/\alpha)
\le
\log(e/\alpha),
\]
so increasing the universal constant if necessary yields
\[
\log N_{2,\mathrm{full}}(\alpha)
\le
C_0
\alpha^{-2}
\log^2(e/\alpha).
\]
This is the claimed normalized entropy bound. \qed

\paragraph{Corollary B.5 (Distinct-query sequential entropy).}
For $1\le T\le m^2$ and $0<\varepsilon\le1/2$,
\[
H_{\mathrm{dist}}(\varepsilon,T)
\le
C_0
\frac{m^2}{T\varepsilon^2}
\log^2
\left(
\frac{em}{\varepsilon\sqrt T}
\right).
\]

\paragraph{Proof.}
Apply Lemma~B.3 with
\[
\alpha
=
\frac{\varepsilon\sqrt T}{m}.
\]
Because $T\le m^2$ and $\varepsilon\le1/2$,
\[
0<\alpha\le1/2.
\]
Lemma~B.4 therefore gives
\begin{align*}
H_{\mathrm{dist}}(\varepsilon,T)
&\le
\log
N_{2,\mathrm{full}}
\left(
\frac{\varepsilon\sqrt T}{m}
\right)
\\
&\le
C_0
\left(
\frac{m}{\varepsilon\sqrt T}
\right)^2
\log^2
\left(
\frac{em}{\varepsilon\sqrt T}
\right),
\end{align*}
as claimed. \qed

\subsubsection{Complete proof of the distinct-query upper bound}
\label{app:upper-bound}

\paragraph{Theorem B.6 (Distinct-query upper bound; restatement of Theorem~\ref{thm:upper}).}
There exists a universal constant $C>0$ such that for every integer
$m\ge2$ and every $1\le T\le m^2$,
\[
\begin{aligned}
R_{m,T}
\le
C
\min
\Bigl\{
&T,\;
m
\Bigl[
1+
\log(em)\\
&\cdot\log\Bigl(1+\frac{T}{m}\Bigr)
\Bigr]
\Bigr\}.
\end{aligned}
\]

\paragraph{Proof.}
\textbf{Step 1. Trivial $T$ bound.} We first record
\[
R_{m,T}\le T.
\]
Indeed, a learner that always predicts in $[0,1]$ incurs at most $1$ squared
loss per round, while the comparator loss is nonnegative.

\paragraph{Step 2. Choose $\gamma=1/2$ and $\rho$.}
It remains to prove the second term.  The constant function
$g\equiv1/2$ is a one-element sequential $L_2$ cover at radius $1/2$:
for every $f:[m]^2\to[0,1]$ and every path,
\[
\frac1T
\sum_{t=1}^T
\left(
f(x_t)-\frac12
\right)^2
\le
\frac14.
\]
Consequently,
\[
N^{\mathrm{dist}}_{2,\mathrm{seq}}(1/2,T)=1,
\qquad
H_{\mathrm{dist}}(1/2,T)=0.
\]

Apply Lemma~B.2 with
\[
\gamma=\frac12.
\]
The entropy term at scale $\gamma$ vanishes, so for every
$\rho\in(0,1/2)$,
\begin{equation}
R_{m,T}
\le
4\rho T
+
12\sqrt T
\int_\rho^{1/2}
\sqrt{
H_{\mathrm{dist}}(\varepsilon,T)
}\,d\varepsilon.
\label{eq:chaining-half}
\end{equation}

We split into two regimes.

\paragraph{Regime 1: $T\le2m$.}
The trivial estimate gives
\[
R_{m,T}\le T\le2m.
\]
Since
\[
1+\log(em)\log\left(1+\frac{T}{m}\right)\ge1,
\]
we have
\[
\begin{aligned}
R_{m,T}
&\le
2m
\\
&\le
2m
\Bigl[
1+
\log(em)\log\Bigl(1+\frac{T}{m}\Bigr)
\Bigr].
\end{aligned}
\]
Thus the desired second-term bound holds in this regime.

\paragraph{Regime 2: $T>2m$.}
Choose
\[
\rho=\frac{m}{T}.
\]
Then $0<\rho<1/2$.
\paragraph{Step 3. Evaluate the chaining integral.}
For every $\varepsilon\in[\rho,1/2]$, Corollary~B.5 yields
\[
\sqrt{
H_{\mathrm{dist}}(\varepsilon,T)
}
\le
\sqrt{C_0}
\frac{m}{\sqrt T\,\varepsilon}
\log
\left(
\frac{em}{\varepsilon\sqrt T}
\right).
\]
The logarithm is positive on the whole integration interval:
because $T\le m^2$ and $\varepsilon\le1/2$,
\[
\frac{em}{\varepsilon\sqrt T}
\ge
2e>1.
\]
Substituting in \eqref{eq:chaining-half} gives
\begin{align}
R_{m,T}
&\le
4m
\nonumber\\
&\quad+
12\sqrt{C_0}\,m
\int_{m/T}^{1/2}
\frac{
\log\!\left(em/(\varepsilon\sqrt T)\right)
}{\varepsilon}
\,d\varepsilon.
\label{eq:upper-integral}
\end{align}

Set
\[
A=\frac{em}{\sqrt T}.
\]
Since
\[
\frac{d}{d\varepsilon}
\left[
-\frac12\log^2(A/\varepsilon)
\right]
=
\frac{\log(A/\varepsilon)}{\varepsilon},
\]
we obtain
\begin{align*}
&\int_{m/T}^{1/2}
\frac{\log(A/\varepsilon)}{\varepsilon}
\,d\varepsilon
\\
&\quad=
\frac12
\Bigl[
\log^2\!\Bigl(\frac{A}{m/T}\Bigr)
-
\log^2\!\Bigl(\frac{A}{1/2}\Bigr)
\Bigr]
\\
&\quad=
\frac12
\Bigl[
\log^2(e\sqrt T)
-
\log^2\!\Bigl(\frac{2em}{\sqrt T}\Bigr)
\Bigr]
\\
&\quad=
\frac12
\log\Bigl(\frac{T}{2m}\Bigr)
\log(2e^2m),
\end{align*}
where the last equality uses
$x^2-y^2=(x-y)(x+y)$.

Therefore
\begin{equation}
R_{m,T}
\le
4m
+
6\sqrt{C_0}\,m
\log\left(\frac{T}{2m}\right)
\log(2e^2m).
\label{eq:upper-explicit}
\end{equation}
Since $T>2m$,
\[
0<
\log\left(\frac{T}{2m}\right)
\le
\log\left(1+\frac{T}{m}\right),
\]
and for $m\ge2$ there is a universal numerical constant $c_1$ such that
\[
\log(2e^2m)
\le
c_1\log(em).
\]
Hence
\[
R_{m,T}
\le
C_1m
\left[
1+
\log(em)
\log\left(1+\frac{T}{m}\right)
\right]
\]
for a universal constant $C_1$.

\paragraph{Step 4. Combine regimes.}
Combining Regimes 1 and 2 proves the second-term bound.  Combining that
bound with $R_{m,T}\le T$ proves
\[
\begin{aligned}
R_{m,T}
\le
C
\min
\Bigl\{
&T,\;
m
\Bigl[
1+
\log(em)\\
&\cdot\log\Bigl(1+\frac{T}{m}\Bigr)
\Bigr]
\Bigr\}.
\end{aligned}
\]
\qed

The only step using distinctness is Lemma~B.3.

\subsubsection{Complete proof of the full-grid lower bound}
\label{app:lower-bound}

\paragraph{Theorem B.7 (Distinct-query lower bound; first part of Corollary~\ref{cor:rate}).}
For every integer $m\ge2$ and every
\[
m\le T\le m^2,
\]
\[
R_{m,T}\ge\frac{m}{4}.
\]

\paragraph{Proof.}
All latent bits are sampled at time zero, queries use only legally revealed
history (here a fixed order), and validity for randomized learners follows
by conditioning on their seed and averaging, as detailed below.
Define the central antichain
\[
A_m
=
\{(i,m+1-i):i=1,\ldots,m\}.
\]
If $i<i'$, then the first coordinate of $(i,m+1-i)$ is smaller than that of
$(i',m+1-i')$, while the second coordinate is larger.  Thus the two points
are incomparable, and $A_m$ is an antichain of size $m$.

Fix any ordering
\[
a_1,\ldots,a_m
\]
of $A_m$.  At time zero sample independent bits
\[
B_a\sim\mathrm{Bernoulli}(1/2),
\qquad a\in A_m.
\]
The randomized adversary queries
$a_1,\ldots,a_m$ in the first $m$ rounds and reveals
\[
y_t=B_{a_t}.
\]
If $T>m$, choose in advance any additional $T-m$ distinct grid points outside
$A_m$.  Their labels will be defined below.  Thus every query is distinct,
and all adversarial randomness is sampled before play begins.

First condition on an arbitrary realization of the learner's internal random
seed.  The learner is then deterministic.  Immediately before predicting on
$a_t$, the bit $B_{a_t}$ is independent of the entire observed history.
For any prediction $p_t\in[0,1]$,
\begin{align*}
&\mathbb E
\bigl[
(p_t-B_{a_t})^2
\mid
\text{history before round }t
\bigr]
\\
&\quad=
\frac12p_t^2
+
\frac12(1-p_t)^2
=
\Bigl(p_t-\frac12\Bigr)^2
+
\frac14
\ge
\frac14.
\end{align*}
Hence the learner's expected loss on the first $m$ rounds is at least
$m/4$.

For the sampled vector $B$, define
\[
f_B(x)
=
\max
\left(
\{B_a:a\in A_m,\ a\preceq x\}
\cup\{0\}
\right).
\]
If $x\preceq x'$, every antichain element below $x$ is also below $x'$,
so
\[
f_B(x)\le f_B(x').
\]
Thus $f_B$ is isotonic.

For any $a\in A_m$, no distinct $a'\in A_m$ satisfies
$a'\preceq a$, because $A_m$ is an antichain.  Therefore
\[
f_B(a)=B_a.
\]
For every preselected filler point $x_t$ with $t>m$, define
\[
y_t=f_B(x_t).
\]
Once $B$ has been sampled, these labels are fixed before play.  The single
comparator $f_B$ therefore incurs zero loss on all $T$ rounds.

We now remove the conditioning on the learner's random seed.  For every
fixed seed, expectation over the randomized oblivious adversary gives
expected regret at least $m/4$.  Averaging over the learner's seed preserves
this inequality:
\[
\mathbb E_{B}
\mathbb E_{\text{learner}}
[R_T]
\ge
\frac{m}{4}.
\]
Therefore there exists at least one deterministic realization of $B$
(and hence one legal oblivious labeled sequence) for which the randomized
learner's expected regret is at least $m/4$.  Since this holds for every
randomized learner,
\[
R_{m,T}\ge\frac{m}{4}.
\]
\qed

\paragraph{Corollary B.8 (Near-matching rate throughout the open window; Corollary~\ref{cor:rate}).}
For every $m\ge2$ and $m\le T\le m^2$,
\[
\frac{m}{4}
\le
R_{m,T}
\le
Cm(1+\log^2m),
\]
and hence
\[
R_{m,T}
=
\widetilde\Theta(m)
\]
uniformly over $m\le T\le m^2$.

\paragraph{Proof.}
The lower bound is Theorem~B.7.  For the upper bound, Theorem~B.6 and
$T\le m^2$ imply
\[
\log\left(1+\frac{T}{m}\right)
\le
\log(1+m)
=
O(\log m),
\]
while
\[
\log(em)=O(\log m)
\qquad
(m\ge2).
\]
Substitution gives the displayed upper bound. \qed

\subsubsection{Complete proof of the span-containment reduction}
\label{app:span-reduction}

Let a length-$m$ document have token boundaries
$0,1,\ldots,m$, and define the set of nonempty half-open spans
\[
\mathcal S_m
=
\{[l,r):0\le l<r\le m\}.
\]

\paragraph{Part A: order.}
\paragraph{Lemma B.9 (Span-containment order isomorphism; Proposition~\ref{prop:span-order}).}
Define
\[
\phi:\mathcal S_m\to[m]^2,
\qquad
\phi([l,r))=(m-l,r),
\]
and
\[
D_m
=
\{(u,v)\in[m]^2:u+v\ge m+1\}.
\]
Then $\phi$ is a bijection from $\mathcal S_m$ onto $D_m$, and
\[
s\subseteq s'
\quad\Longleftrightarrow\quad
\phi(s)\preceq\phi(s')
\]
for all $s,s'\in\mathcal S_m$.

\paragraph{Proof.}
Write
$s=[l,r)$ and $s'=[l',r')$.  For half-open intervals,
\[
s\subseteq s'
\quad\Longleftrightarrow\quad
l'\le l
\ \text{and}\
r\le r'.
\]
The first inequality is equivalent to
$m-l\le m-l'$, while the second is unchanged.  Hence
\[
\begin{aligned}
s\subseteq s'
&\Longleftrightarrow
(m-l,r)\preceq(m-l',r')
\\
&\Longleftrightarrow
\phi(s)\preceq\phi(s').
\end{aligned}
\]

If $s=[l,r)\in\mathcal S_m$, then
$u=m-l\in[m]$ and $v=r\in[m]$.  Moreover,
$l<r$ is equivalent to
$m-u<v$, which for integers is equivalent to
$u+v\ge m+1$.  Hence $\phi(s)\in D_m$.

Conversely, if $(u,v)\in D_m$, set
\[
l=m-u,\qquad r=v.
\]
Then $0\le l\le m-1$, $1\le r\le m$, and
$u+v\ge m+1$ implies $l<r$.  Hence $[l,r)\in\mathcal S_m$ and
\[
\phi([m-u,v))=(u,v).
\]
Thus $\phi$ is bijective onto $D_m$. \qed

\paragraph{Part B: comparator extension.}
\paragraph{Lemma B.10 (Comparator-preserving extension from the triangle).}
The set $D_m$ is an upper set of $[m]^2$.  For every isotonic
$f:D_m\to[0,1]$, define
\[
F(x)
=
\min
\{f(y):y\in D_m,\ x\preceq y\}.
\]
Then the minimum is well defined,
$F:[m]^2\to[0,1]$ is isotonic, and
\[
F|_{D_m}=f.
\]

\paragraph{Proof.}
If $(u,v)\in D_m$ and
$(u,v)\preceq(u',v')$, then
\[
u'+v'\ge u+v\ge m+1,
\]
so $(u',v')\in D_m$.  Thus $D_m$ is an upper set.

For any $x\in[m]^2$, the set
\[
U_x=\{y\in D_m:x\preceq y\}
\]
is nonempty because $(m,m)\in D_m$ and $x\preceq(m,m)$.
It is finite, so the minimum exists.

If $x\preceq x'$, then
\[
U_{x'}\subseteq U_x,
\]
hence
\[
F(x)=\min_{y\in U_x}f(y)
\le
\min_{y\in U_{x'}}f(y)
=
F(x').
\]
Thus $F$ is isotonic.

Finally, if $x\in D_m$, then $x\in U_x$, so $F(x)\le f(x)$.
For every $y\in U_x$, isotonicity gives $f(x)\le f(y)$.
Therefore $f(x)$ is the attained minimum and
$F(x)=f(x)$. \qed

Therefore the comparator evaluation-vector sets are exactly equal:
restriction and extension, together, give the following consequence.

\paragraph{Corollary B.11 (Exact comparator equality on queried span points).}
For any queried points
$x_1,\ldots,x_T\in D_m$, the set of achievable evaluation vectors
\[
\left\{
(f(x_1),\ldots,f(x_T)):
f\in\mathcal F^\uparrow(D_m)
\right\}
\]
is exactly equal to
\[
\left\{
(F(x_1),\ldots,F(x_T)):
F\in\mathcal F_m^\uparrow
\right\}.
\]
Consequently, for any loss functional depending only on these evaluations,
the two comparator infima are equal.  In particular, for any continuous
per-round loss $\ell$ (including squared loss), the minima are attained and
\[
\begin{aligned}
&\min_{f\in\mathcal F^\uparrow(D_m)}
\sum_{t=1}^T\ell(f(x_t),y_t)
\\
&\quad=
\min_{F\in\mathcal F_m^\uparrow}
\sum_{t=1}^T\ell(F(x_t),y_t).
\end{aligned}
\]

\paragraph{Proof.}
Restriction of a full-grid isotonic function to $D_m$ gives an isotonic
function on $D_m$, giving one inclusion between evaluation-vector sets.
Lemma~B.10 extends every isotonic function on $D_m$ to a full-grid isotonic
function with identical values on all queried points, giving the reverse
inclusion.  Equality of infima follows immediately.  For continuous
$\ell$, both comparator classes are compact subsets of finite-dimensional
cubes, so the infima are attained. \qed

\paragraph{Corollary B.12 (OECP upper-bound transfer on all spans).}
For
\[
1\le T\le|\mathcal S_m|=\frac{m(m+1)}2,
\]
\[
R^*_{T,\mathrm{dist}}(\mathcal S_m,\subseteq)
\le
R_{m,T}.
\]

\paragraph{Proof.}
Since $|\mathcal S_m|=m(m+1)/2\le m^2$, every admissible distinct-query
span horizon is within the full-grid theorem's horizon range.
Run any full-grid learner on the mapped sequence
$\phi(s_t)$.  Lemma~B.9 makes $\phi$ injective, so distinct spans map to
distinct grid points.  Round-by-round predictions and losses are unchanged
by this relabeling.  Corollary~B.11 shows that the best
containment-monotone span comparator has exactly the same squared loss as the
best full-grid isotonic comparator on the mapped points.  Thus the simulated
span learner has exactly the full-grid regret on the mapped sequence.
Taking the infimum over learners and supremum over legal span sequences gives
the stated inequality. \qed

\subsubsection{Complete proof of endpoint-rank compression}
\label{app:endpoint-compression}

Let
$\mathcal C\subseteq\mathcal S_m$
be a finite set of candidate spans.  Define
\[
\begin{aligned}
L&=\{l:\exists r,\ [l,r)\in\mathcal C\},
\\
R&=\{r:\exists l,\ [l,r)\in\mathcal C\},
\end{aligned}
\]
\[
a=|L|,
\qquad
b=|R|,
\qquad
m_{\mathrm{eff}}=\max(a,b).
\]
Let
$\operatorname{rank}_L$ and $\operatorname{rank}_R$
be ascending ranks starting at $1$, and define
\[
\psi([l,r))
=
\left(
a-\operatorname{rank}_L(l)+1,\;
\operatorname{rank}_R(r)
\right).
\]

\paragraph{Theorem B.13 (Exact endpoint-rank compression; Theorem~\ref{thm:compression}).}
The map $\psi$ is injective and
\[
s\subseteq s'
\quad\Longleftrightarrow\quad
\psi(s)\preceq\psi(s')
\]
for all $s,s'\in\mathcal C$.  Thus
$(\mathcal C,\subseteq)$ is order-isomorphic to
\[
P=\psi(\mathcal C)
\subseteq
[a]\times[b]
\subseteq
[m_{\mathrm{eff}}]^2.
\]
Moreover, every isotonic
$f:P\to[0,1]$
has an isotonic extension
$F:[m_{\mathrm{eff}}]^2\to[0,1]$
with
$F|_P=f$.

\paragraph{Proof.}
\textbf{Part A: order.} Take
$s=[l,r)$ and $s'=[l',r')$.
Then
\[
s\subseteq s'
\quad\Longleftrightarrow\quad
l'\le l
\ \text{and}\
r\le r'.
\]
Because ascending rank is strictly increasing on the distinct endpoint sets,
\[
l'\le l
\quad\Longleftrightarrow\quad
\operatorname{rank}_L(l')
\le
\operatorname{rank}_L(l),
\]
which is equivalent to
\[
a-\operatorname{rank}_L(l)+1
\le
a-\operatorname{rank}_L(l')+1.
\]
Similarly,
\[
r\le r'
\quad\Longleftrightarrow\quad
\operatorname{rank}_R(r)
\le
\operatorname{rank}_R(r').
\]
These are exactly the two coordinate inequalities in
$\psi(s)\preceq\psi(s')$.

If
$\psi([l,r))=\psi([l',r'))$,
equality of the first coordinates implies equal left-endpoint ranks and hence
$l=l'$.  Equality of the second coordinates implies
$r=r'$.  Thus $\psi$ is injective.

\textbf{Part B: comparator extension.} For
$x\in[m_{\mathrm{eff}}]^2$
let
\[
D_x=\{z\in P:z\preceq x\}
\]
and define
\[
F(x)
=
\max
\left(
\{f(z):z\in D_x\}
\cup\{0\}
\right).
\]
The maximum is well defined because $P$ is finite and
$f(P)\subseteq[0,1]$.  If $x\preceq x'$, then
$D_x\subseteq D_{x'}$, hence $F(x)\le F(x')$.
Thus $F$ is isotonic.

For $x\in P$, the set $D_x$ contains $x$, so $F(x)\ge f(x)$.
For every $z\in D_x$, isotonicity on $P$ gives
$f(z)\le f(x)$, so all terms in the maximum are at most $f(x)$.
Hence $F(x)=f(x)$. \qed

Therefore the comparator evaluation-vector sets are exactly equal;
this is the additional fact needed to transfer comparator losses.

\paragraph{Corollary B.14 (Exact comparator equality after compression).}
For any queried points $x_1,\ldots,x_T\in P$, the evaluation-vector set of
isotonic functions on $P$ equals the evaluation-vector set obtained by
restricting full-grid isotonic functions on
$[m_{\mathrm{eff}}]^2$ to those points.  Therefore the two comparator
infima are equal for every loss functional depending only on those
evaluations, and the minima are equal for every continuous per-round loss.

\paragraph{Proof.}
Restriction gives one inclusion.  The extension in Theorem~B.13 gives the
reverse inclusion.  The loss conclusion is then identical to the argument in
Corollary~B.11. \qed

\paragraph{Corollary B.15 (Sparse-candidate upper-bound transfer; second part of Theorem~\ref{thm:compression}).}
For every
\[
1\le T\le|\mathcal C|,
\]
\[
R^*_{T,\mathrm{dist}}(\mathcal C,\subseteq)
\le
R^*_{T,\mathrm{dist}}([m_{\mathrm{eff}}]^2).
\]
If $m_{\mathrm{eff}}\ge2$, then automatically
\[
T\le|\mathcal C|\le ab\le m_{\mathrm{eff}}^2,
\]
so Theorem~B.6 gives
\[
\begin{aligned}
&R^*_{T,\mathrm{dist}}(\mathcal C,\subseteq)
\le
C
\min
\Bigl\{
T,\;
m_{\mathrm{eff}}
\Bigl[
1
\\
&\quad+
\log(e m_{\mathrm{eff}})
\log\Bigl(1+\frac{T}{m_{\mathrm{eff}}}\Bigr)
\Bigr]
\Bigr\}.
\end{aligned}
\]
If $m_{\mathrm{eff}}=1$, then
$|\mathcal C|=1$ and the trivial one-round bound applies.

\paragraph{Proof.}
Simulate the full-grid learner on $\psi(s_t)$.  Injectivity preserves
distinctness, and Corollary~B.14 preserves the comparator loss exactly.
This proves the first display.  The cardinality bound
$|\mathcal C|\le ab\le m_{\mathrm{eff}}^2$
allows Theorem~B.6 to be applied whenever $m_{\mathrm{eff}}\ge2$. \qed

Algorithm~\ref{alg:endpoint-compression} computes $\psi$ and $\meff$ by sorting the distinct endpoints once, in $O(n\log n)$ time and $O(n)$ memory for $n=|\mathcal C|$, without enumerating candidate pairs.

\begin{algorithm}[t]
\caption{Exact Endpoint-Rank Compression}
\label{alg:endpoint-compression}
\begin{algorithmic}[1]
\Require spans $\mathcal C=\{[l_i,r_i)\}_{i=1}^{n}$
\Ensure coordinates $\psi(\mathcal C)$; side length $\meff$
\State $L\gets\mathrm{sort}(\mathrm{unique}\{l_i\}_{i=1}^{n})$
\State $R\gets\mathrm{sort}(\mathrm{unique}\{r_i\}_{i=1}^{n})$
\State $a\gets|L|$, \ $b\gets|R|$
\State build ascending rank maps $\rank_L$, $\rank_R$
\For{$i=1,\ldots,n$}
  \State $u_i\gets a-\rank_L(l_i)+1$
  \State $v_i\gets\rank_R(r_i)$
  \State $\psi([l_i,r_i))\gets(u_i,v_i)$
\EndFor
\State $\meff\gets\max\{a,b\}$
\State \Return $\psi(\mathcal C)$, $\meff$
\end{algorithmic}
\end{algorithm}

\subsubsection{Width lower bound for arbitrary finite candidate posets}
\label{app:width-lower}

\paragraph{Proposition B.16 (Width lower bound).}
Let $(P,\preceq)$ be any finite poset with
\[
w=\operatorname{width}(P).
\]
For every
\[
w\le T\le|P|,
\]
distinct-query online isotonic regression under squared loss satisfies
\[
R^*_{T,\mathrm{dist}}(P)
\ge
\frac{w}{4}.
\]

\paragraph{Proof.}
All latent bits are sampled at time zero, queries use only legally revealed
history (here a fixed order), and randomized learners are handled by the
seed conditioning and averaging of Theorem~B.7.
Let
$A\subseteq P$
be a maximum antichain, $|A|=w$.
At time zero sample independent bits
\[
B_a\sim\operatorname{Bernoulli}(1/2),
\qquad a\in A.
\]
Query the points of $A$ first in any fixed order and reveal label $B_a$ at
$a$.  Conditioned on any fixed learner random seed, each unrevealed bit is
independent of the observed history, so the same computation as in
Theorem~B.7 gives expected squared loss at least $1/4$ per antichain round.

Define
\[
f_B(x)
=
\max
\left(
\{B_a:a\in A,\ a\preceq x\}
\cup\{0\}
\right).
\]
If $x\preceq x'$, the set of antichain elements below $x$ is contained in
the set below $x'$, hence $f_B(x)\le f_B(x')$.
Thus $f_B$ is isotonic.
Because $A$ is an antichain,
$f_B(a)=B_a$ for each $a\in A$.

If $T>w$, preselect any $T-w$ unused points and label each filler point $x$
by $f_B(x)$.  The comparator $f_B$ has zero loss on every round.
The averaging argument of Theorem~B.7 gives expected regret at least $w/4$
against every randomized learner; taking the infimum over learners proves
the minimax bound. \qed

\paragraph{Corollary B.17 (Width versus effective size).}
For a finite candidate family $\mathcal C$ with
\[
w=\operatorname{width}(\mathcal C,\subseteq),
\]
\[
w\le\min(a,b)\le m_{\mathrm{eff}}.
\]
Hence, whenever
\[
m_{\mathrm{eff}}\le T\le|\mathcal C|,
\]
\[
\frac{w}{4}
\le
R^*_{T,\mathrm{dist}}(\mathcal C,\subseteq)
\le
\widetilde O(m_{\mathrm{eff}}).
\]
If
$w=\Theta(m_{\mathrm{eff}})$,
the lower and upper scales match up to logarithmic factors.

\paragraph{Proof.}
By Theorem~B.13,
$\mathcal C$ embeds order-isomorphically into $[a]\times[b]$.
Two distinct elements of an antichain in $[a]\times[b]$ cannot share the
same first coordinate, because points with equal first coordinate are
comparable by their second coordinates.  Therefore every such antichain has
size at most $a$.  The symmetric argument gives size at most $b$, so
\[
w\le\min(a,b)\le m_{\mathrm{eff}}.
\]
The lower bound follows from Proposition~B.16 because
$T\ge m_{\mathrm{eff}}\ge w$.
The upper bound follows from Corollary~B.15. \qed

\subsubsection{Realizing the full-grid antichain lower bound within OECP}
\label{app:oecp-lower-semantics}

The preceding lower bounds concern the full isotonic comparator class.  We
now identify precisely when the central-antichain hard family is also an
evidence-containment family.

\paragraph{Lemma B.18 (Multiple-target OECP realization).}
Let
\[
A_m=\{(i,m+1-i):i\in[m]\}.
\]
Under the span isomorphism of Lemma~B.9,
$\phi^{-1}(A_m)$ is exactly the set of all $m$ one-token spans
\[
\{[j,j+1):j=0,\ldots,m-1\}.
\]
For any binary assignment
$B:A_m\to\{0,1\}$, define
\[
G_B
=
\left\{
\phi^{-1}(a):
a\in A_m,\ B_a=1
\right\}.
\]
Then the OECP containment label
\[
Y_s
=
\mathbf 1
\{\exists g\in G_B:g\subseteq s\}
\]
satisfies
\[
Y_s=f_B(\phi(s)),
\]
where
\[
f_B(x)
=
\max
\left(
\{B_a:a\in A_m,\ a\preceq x\}
\cup\{0\}
\right).
\]

\paragraph{Proof.}
For
$a=(i,m+1-i)\in A_m$,
\[
\phi^{-1}(a)
=
[m-i,m+1-i),
\]
which has length one.
By Lemma~B.9,
\[
\phi^{-1}(a)\subseteq s
\quad\Longleftrightarrow\quad
a\preceq\phi(s).
\]
Therefore
\begin{align*}
Y_s=1
&\Longleftrightarrow
\exists a\in A_m:
B_a=1
\ \text{and}\
a\preceq\phi(s)
\\
&\Longleftrightarrow
\max
\bigl(
\{B_a:a\in A_m,\ a\preceq\phi(s)\}
\\
&\qquad\qquad
\cup\{0\}
\bigr)
=1,
\end{align*}
which is exactly
$f_B(\phi(s))=1$.
Both sides are binary, so equality holds for every $s$. \qed

\paragraph{Single-target scope.}
With one fixed target span $g$,
\[
Y_s=\mathbf1\{g\subseteq s\}
\]
is the indicator of a single principal upper set.  Such a family cannot
realize arbitrary binary labelings of a size-$m$ antichain.  Therefore
Theorem~B.7 and Proposition~B.16 concern the general multi-target class;
the fixed-target rate is characterized in Appendix~\ref{app:single}.

\paragraph{Nonempty-target variant.}
The general OECP definition permits $G_B=\varnothing$, which occurs for the
all-zero assignment.  If an application insists that the target set be
nonempty, fix one antichain bit to $1$ and sample the remaining $m-1$ bits
independently.  The same argument forces expected loss at least
\[
\frac{m-1}{4}
\]
on the randomized bits while preserving a zero-loss containment comparator.
Thus the semantic lower bound remains $\Omega(m)$ even under an explicit
nonempty-target requirement.

\subsubsection{Technical conditions used in Part I}
\label{app:proof-audit-checklist}

\begin{itemize}
\item \emph{Adaptive legality.} The history-dependent legal query set is
retained through the reduction (Lemmas~B.1--B.2).
\item \emph{Response range.} Enlarge to $[-1,1]$, then clip predictions to
$[0,1]$ without increasing squared loss.
\item \emph{Sequential covers.} Covers may depend on the fixed covariate
tree; Lemma~B.3 satisfies the required pathwise quantifiers.
\item \emph{Entropy normalization.} Use Frobenius radius $m\alpha$, bounded
class parameter $t=1$, and $0<\alpha=\varepsilon\sqrt T/m\le1/2$ (B.4--B.5).
\item \emph{Randomized learners.} The lower bounds use time-zero bits and
seed conditioning followed by averaging (B.7 and B.16).
\item \emph{Comparators and scope.} Both embeddings preserve comparator
evaluation sets; the antichain realization allows multiple targets,
including the empty set for the all-zero labeling (B.9--B.18).
\end{itemize}

\setcounter{appthm}{18}
\subsection{Part II: Fixed-Target OECP}
\label{app:single}

\paragraph{Proof architecture.}
The fixed-target rate has three pieces: (i)~class-size compression, $N_1(\mathcal C)\le ab+1$ (Appendix~\ref{app:single-compression}); (ii)~a constructive upper bound: exponential weights attains $\min\{T,2\ln N_1(\mathcal C)\}$ (Theorem~\ref{thm:single-ew}), and the $T/4$ term of Theorem~\ref{thm:single} comes from selecting, before play, between exponential weights and the constant prediction $\tfrac12$ (Corollary~\ref{cor:single-horizon}; Appendix~\ref{app:single-upper}); and (iii)~a matching worst-case lower bound, using a hidden threshold on prefix candidates (Appendix~\ref{app:single-lower}).
Appendix~\ref{app:single-realizable} gives the realizable version-space bound; Appendix~\ref{app:single-ldim} gives an optional family-specific refinement; Appendix~\ref{app:single-separation} gives the same-family target-multiplicity separation.

Throughout, $\mathcal S_m=\{[u,v):u,v\in\mathbb Z,\ 0\le u<v\le m\}$, the finite candidate family $\mathcal C\subseteq\mathcal S_m$ is known before play, $n=|\mathcal C|$, $L=\{l_s:s\in\mathcal C\}$, $R=\{r_s:s\in\mathcal C\}$, $a=|L|$, $b=|R|$, and $K=\meff=\max(a,b)$.
$\ln$ is the natural logarithm; $\log_2$ is written explicitly.

\subsubsection{Fixed-target comparator class and games}
\label{app:single-games}

For a target $g\in\mathcal S_m$ put $h_g(s)=\1\{g\subseteq s\}$.
Write $g\sim_{\mathcal C}g'$ if $h_g|_{\mathcal C}=h_{g'}|_{\mathcal C}$, and
\[
\begin{aligned}
\mathcal H_{\rm one}(\mathcal C)&=\{h_g|_{\mathcal C}:g\in\mathcal S_m\},\\
N_1(\mathcal C)&=|\mathcal H_{\rm one}(\mathcal C)|.
\end{aligned}
\]
Exactly one nonempty target is used; the all-zero labeling, when possible, comes from a real target contained in no candidate, not from an empty target.

\paragraph{Games.}
In the \emph{agnostic} game the outcomes $y_t\in[0,1]$ are arbitrary and the comparator class is $\mathcal H_{\rm one}(\mathcal C)$:
\[
R_T=\sum_{t=1}^T(p_t-y_t)^2-\min_{h\in\mathcal H_{\rm one}(\mathcal C)}\sum_{t=1}^T\bigl(h(s_t)-y_t\bigr)^2.
\]
In the \emph{realizable} game a target $g^\star$ is fixed before the trajectory and $y_t=h_{g^\star}(s_t)$, so regret equals cumulative loss $L_T=\sum_t(p_t-y_t)^2$.
At round $t$ the scheduler chooses $s_t$ from the public history, the learner predicts $p_t\in[0,1]$, and then $y_t$ is revealed; a distinct-query scheduler satisfies $s_t\ne s_u$ for $t\ne u$ on every realized path, so $T\le n$.
Write $V^{\rm ag}_{T,\rm dist}(\mathcal C)=\inf_{\mathcal A}\sup_{\mathcal E}\mathbb E R_T$ and $V^{\rm real}_{T,\rm dist}(\mathcal C)=\inf_{\mathcal A}\sup_{g^\star,\mathcal E}\mathbb E L_T$.
Lower bounds below use less adversarial information: queries depend only on earlier labels, and the target is sampled at time zero independently of learner randomness.

\subsubsection{Exact target-equivalence compression}
\label{app:single-compression}

\begin{appthm}[Exact target-equivalence compression]
\label{thm:single-compression}
Let $\delta=\1\{[0,m)\notin\mathcal C\}$ and let
$\mathcal J(\mathcal C)=\{s\cap t:s,t\in\mathcal C,\ s\cap t\ne\varnothing\}$ be the nonempty intersections of at most two candidates ($s=t$ allowed).
Then
\[
N_1(\mathcal C)=|\mathcal J(\mathcal C)|+\delta\le ab+\delta\le ab+1.
\]
For nonempty $\mathcal C$, enumerate $L=\{\lambda_1<\dots<\lambda_a\}$, $R=\{\rho_1<\dots<\rho_b\}$, and set $U_i=\max\{r:[\lambda_i,r)\in\mathcal C\}$, $V_j=\min\{l:[l,\rho_j)\in\mathcal C\}$.
Then
\[
\begin{aligned}
N_1(\mathcal C)=\delta+\sum_{i=1}^a\sum_{j=1}^b\1\{&\lambda_i<\rho_j,\ \rho_j\le U_i,\\
&V_j\le\lambda_i\},
\end{aligned}
\]
and the qualifying $[\lambda_i,\rho_j)$, together with $[0,m)$ if $\delta=1$, form exactly one representative per labeling class.
\end{appthm}

\begin{proof}
\emph{Claim 1: every nonzero signature has a canonical representative.}
For a target $g$ with nonempty positive set $P_g=\{s\in\mathcal C:g\subseteq s\}$, define $I_g=\bigcap_{s\in P_g}s=[\max_{s\in P_g}l_s,\min_{s\in P_g}r_s)$.
It contains $g$, so it is a nonempty legal target.
If $g\subseteq s$ then $I_g\subseteq s$; conversely $I_g\subseteq s$ implies $g\subseteq s$.
Hence $h_{I_g}|_{\mathcal C}=h_g|_{\mathcal C}$.

\emph{Claim 2: $I_g$ is the intersection of at most two candidates.}
The maximal left and minimal right endpoints are attained by at most two candidates, so $I_g$ is their nonempty intersection and lies in $\mathcal J(\mathcal C)$.
Note that $g$ itself need not have observed endpoints; it is first mapped to its canonical representative $I_g$, whose endpoints are observed.

\emph{Claim 3: distinct canonical intersections have distinct signatures.}
Conversely let $I=s\cap t\ne\varnothing$.
The candidates positive for target $I$ include $s$ and $t$, so their intersection is contained in $s\cap t=I$; since all of them contain $I$, it equals $I$.
Hence distinct elements of $\mathcal J(\mathcal C)$ have distinct signatures, which proves a bijection between nonzero signatures and $\mathcal J(\mathcal C)$.
If $[0,m)\in\mathcal C$, every target lies in it and no zero signature exists; otherwise the target $[0,m)$ lies in no candidate and realizes the unique zero signature.

\emph{Endpoint-pair enumeration.}
Every nonempty intersection has left endpoint in $L$ and right endpoint in $R$, giving the $ab$ bound.
An interval $[\lambda_i,\rho_j)$ is such an intersection only if it is nonempty, some candidate with left endpoint $\lambda_i$ reaches at least $\rho_j$, and some candidate with right endpoint $\rho_j$ starts at most at $\lambda_i$; these are the stated conditions, and $[\lambda_i,U_i)$ and $[V_j,\rho_j)$ witness sufficiency.
\end{proof}

\paragraph{Sharpness.}
For positive $a,b$, take $m=a+b+1$ and $\mathcal C=\{[i,a+j):0\le i<a,\ 1\le j\le b\}$.
All $ab$ endpoint pairs are candidates and hence distinct canonical intersections, and $[0,m)\notin\mathcal C$, so $N_1(\mathcal C)=ab+1$; the zero-class correction cannot be dropped.

\subsubsection{Constructive agnostic upper bound}
\label{app:single-upper}

\begin{appthm}[Exponential weights over fixed targets]
\label{thm:single-ew}
Enumerate the $N=N_1(\mathcal C)$ representatives of Theorem~\ref{thm:single-compression}, set $w_1(h)=1/N$, and for every round
\[
\begin{aligned}
p_t&=\textstyle\sum_h w_t(h)h(s_t),\\
w_{t+1}(h)&\propto w_t(h)\,e^{-(h(s_t)-y_t)^2/2}.
\end{aligned}
\]
For every query sequence, adaptive and even repeated, and every outcome sequence in $[0,1]$, pathwise
\[
\begin{aligned}
R_T&\le\min\{T,2\ln N\}\le\min\{T,2\ln(ab+1)\}\\
&\le\min\{T,4\ln(K+1)\}.
\end{aligned}
\]
\end{appthm}

\begin{proof}
For $y\in[0,1]$ and $\eta>0$,
\[
\frac{d^2}{dp^2}e^{-\eta(p-y)^2}=2\eta e^{-\eta(p-y)^2}\bigl\{2\eta(p-y)^2-1\bigr\},
\]
so $p\mapsto e^{-\eta(p-y)^2}$ is concave on $[0,1]$ whenever $\eta\le1/2$ (square loss is $\tfrac12$-exp-concave; cf.\ \citealp{cesabianchi2006prediction}).
By Jensen's inequality,
\[
(p_t-y_t)^2\le-\eta^{-1}\ln\sum_h w_t(h)e^{-\eta(h(s_t)-y_t)^2}.
\]
With $L_t(h)=\sum_{u\le t}(h(s_u)-y_u)^2$ and $Z_t=N^{-1}\sum_h e^{-\eta L_t(h)}$, $Z_0=1$, the sum inside the logarithm equals $Z_t/Z_{t-1}$.
Summing and using $Z_T\ge N^{-1}e^{-\eta\min_hL_T(h)}$ gives $\sum_t(p_t-y_t)^2\le\min_hL_T(h)+\eta^{-1}\ln N$.
Take $\eta=1/2$.
Learner loss is at most $T$ and comparator loss is nonnegative, and $ab+1\le(K+1)^2$.
No distinctness or independence is used.
\end{proof}

\begin{appcor}[Horizon-aware fixed-target upper bound]
\label{cor:single-horizon}
Let $\mathcal A^{1/2}$ be the learner that predicts $p_t=\tfrac12$ in every round, and let $\mathcal A^{\rm EW}$ be the learner of Theorem~\ref{thm:single-ew}.
For a known horizon $T\ge1$, let $\mathcal A^{\rm H}_T$ run $\mathcal A^{1/2}$ for the whole game if $T/4\le2\ln N_1(\mathcal C)$ and $\mathcal A^{\rm EW}$ for the whole game otherwise; the choice is made before the first query and depends only on $T$ and $\mathcal C$.
Then for every adaptive query sequence (in particular every distinct-query scheduler, so $1\le T\le|\mathcal C|$) and every outcome sequence in $[0,1]$, pathwise,
\[
\begin{aligned}
R_T(\mathcal A^{\rm H}_T)&\le\min\{T/4,\,2\ln N_1(\mathcal C)\}\\
&\le\min\{T/4,\,2\ln(ab+1)\},
\end{aligned}
\]
and hence $V^{\rm ag}_{T,\rm dist}(\mathcal C)\le\min\{T/4,2\ln N_1(\mathcal C)\}$.
The bound $T/4$ is a property of $\mathcal A^{1/2}$, not of $\mathcal A^{\rm EW}$: exponential weights alone need not satisfy $R_T\le T/4$ (see the example below).
\end{appcor}
\begin{proof}
For $\mathcal A^{1/2}$ and any $y_t\in[0,1]$, $(\tfrac12-y_t)^2\le\tfrac14$, so its cumulative loss is at most $T/4$; every comparator has nonnegative loss, hence $R_T(\mathcal A^{1/2})\le T/4$.
Theorem~\ref{thm:single-ew} gives $R_T(\mathcal A^{\rm EW})\le2\ln N_1(\mathcal C)$ pathwise.
Since the selection rule is fixed before play and uses only $T$ and $\mathcal C$, the selected learner is a single complete learner for the whole game, and its regret is bounded by the bound of the learner it runs, which is the smaller of the two.
Finally $N_1(\mathcal C)\le ab+1$ by Theorem~\ref{thm:single-compression}, and $\ln$ is increasing.
\end{proof}

\paragraph{Exponential weights alone does not give $T/4$.}
Let $K=2$, $m=5$, and $\mathcal C=\{[0,3),[0,4),[1,3),[1,4)\}$.
By Theorem~\ref{thm:single-compression}, $N_1(\mathcal C)=5$: the four canonical intersections $[0,3),[0,4),[1,3),[1,4)$ give four distinct nonzero labelings, and $[0,5)\notin\mathcal C$ gives the all-zero labeling.
Exactly four of these five labelings assign $1$ to $[0,4)$.
With $T=1$, query $[0,4)$ and reveal $y_1=0$, which is realized by the legal target $[4,5)$.
Uniform exponential weights predicts $p_1=4/5$, so its regret is $(4/5)^2-0=16/25>1/4=T/4$.
Thus the horizon-aware selection in Corollary~\ref{cor:single-horizon} is needed for the $T/4$ term.

\subsubsection{Realizable fixed-target predictor}
\label{app:single-realizable}

\begin{appthm}[Version space]
\label{thm:single-version}
For a fixed $g^\star$ and any adaptive query sequence, including repeated queries, let $\mathcal V_t$ be the representatives consistent with all earlier labels and predict $p_t=|\{h\in\mathcal V_t:h(s_t)=1\}|/|\mathcal V_t|$; after observing $y_t=h_{g^\star}(s_t)$, delete inconsistent representatives.
Then
\[
L_T\le\min\{T,\tfrac12\ln N_1(\mathcal C)\}\le\min\{T,\ln(K+1)\}.
\]
\end{appthm}

\begin{proof}
The true signature is never removed, so $\mathcal V_t\ne\varnothing$.
Let $q_t=|\mathcal V_{t+1}|/|\mathcal V_t|\in(0,1]$.
If $y_t=1$ then $q_t=p_t$, and if $y_t=0$ then $q_t=1-p_t$; in both cases the loss is $(1-q_t)^2$.
For $q\in(0,1]$, $F(q)=-\ln q-2(1-q)^2$ satisfies $F(1)=0$ and $F'(q)=-(2q-1)^2/q\le0$, so $F\ge0$.
Hence $L_T\le\tfrac12\sum_t-\ln q_t=\tfrac12\ln(|\mathcal V_1|/|\mathcal V_{T+1}|)\le\tfrac12\ln N_1(\mathcal C)$.
This is a realizable cumulative-loss bound; with fractional or inconsistent labels the version space can become empty, and the learner of Theorem~\ref{thm:single-ew} is the agnostic algorithm.
\end{proof}

\subsubsection{Optional family-specific refinement}
\label{app:single-ldim}

This refinement is not needed for the worst-case rate.
Let $D=\operatorname{Ldim}(\mathcal H_{\rm one}(\mathcal C))$ be the Littlestone dimension \citep{littlestone1988learning}: the maximal depth of a complete binary query tree shattered by $\mathcal H_{\rm one}(\mathcal C)$.
Because a fixed hypothesis assigns one label per candidate, a shattered tree cannot repeat a query on a path, so these trees already satisfy distinctness.

\begin{appthm}[Family-specific realizable rate]
\label{thm:single-ldim}
For nonempty $\mathcal C$ and $1\le T\le|\mathcal C|$,
\[
\begin{aligned}
\tfrac14\min\{T,D\}&\le V^{\rm real}_{T,\rm dist}(\mathcal C)\\
&\le\min\{T/4,\ D,\ \tfrac12\ln N_1(\mathcal C)\}.
\end{aligned}
\]
\end{appthm}

\begin{proof}
\emph{Upper bound.} The Standard Optimal Algorithm predicts the label whose nonempty version-space branch has larger Littlestone dimension (ties broken deterministically).
If both branches had dimension at least the current one, attaching their trees below the current query would give a deeper shattered tree, so each mistake decreases the current dimension; there are at most $D$ mistakes, and binary predictions have squared loss equal to their mistake indicators.
The other two upper bounds follow from the constant-half predictor and Theorem~\ref{thm:single-version}, selected at time zero.

\emph{Lower bound.} Take a shattered tree of depth $d=\min(T,D)$, choose one realizing hypothesis per leaf, and select a leaf uniformly at time zero, independently of learner randomness.
Distinct leaves cannot share a hypothesis, since their paths diverge at a query with different labels.
Following the tree with the target's labels, each next edge is a fair bit conditional on the history and on learner randomness, so each round has expected loss $(p_t-\tfrac12)^2+\tfrac14\ge\tfrac14$; queries are pathwise distinct.
If $T>d$, append unused candidates labeled by the same target.
Averaging over leaves shows that for every learner some fixed target attains the bound.
\end{proof}

\paragraph{Sparse-family example.}
$\log N_1(\mathcal C)$ is not a family-wise lower bound.
For $\mathcal C=\{[i,i+1):0\le i<n\}$ with $m=n\ge2$, $N_1(\mathcal C)=n+1$ and $K=n$, but $D=1$: predicting zero has realizable loss at most one and agnostic regret at most one on distinct queries, while a time-zero fair choice between the first queried singleton and the target $[0,m)$ gives loss $\ge1/4$.
Hence this family has value $\Theta(1)$, and the logarithmic rate of Theorem~\ref{thm:single} is a worst case over candidate families.

\subsubsection{Fixed-target lower bound}
\label{app:single-lower}

\begin{appthm}[Worst-case logarithmic lower bound]
\label{thm:single-lower}
For every integer $K\ge1$, set $m=2K+1$ and
\[
\mathcal C_K=\{[i,K+r):0\le i<K,\ 1\le r\le K\}.
\]
Then $a=b=K$, $|\mathcal C_K|=K^2$, and $N_1(\mathcal C_K)=K^2+1$.
For every $1\le T\le K^2$ and every (possibly randomized) learner there is a target fixed before play and an adaptive, pathwise-distinct scheduler with expected cumulative loss, and hence single-target regret, at least
\[
\tfrac14\min\bigl\{T,\lfloor\log_2(K+1)\rfloor\bigr\}.
\]
The bound holds in both the realizable and the agnostic game.
\end{appthm}

\paragraph{Threshold picture.}
A one-token target $g_J$ is a hidden threshold: the prefix candidate $s_r$ returns $1$ exactly when $J\le r$.
Median queries perform binary search on a uniform dyadic block of target indices.
Before each of its first $\lfloor\log_2(K+1)\rfloor$ answers, the next label is conditionally fair, forcing expected squared loss at least $1/4$.
If $J=K+1$ belongs to the chosen block, it realizes the all-zero signature on all candidates in $\mathcal C_K$.

\begin{proof}
The target index is sampled at time zero, queries use only earlier labels, and the randomized-learner bound follows by conditioning on the learner's seed and averaging.
Every candidate is legal, since $i\le K-1<K+1\le K+r\le2K<m$, and there are exactly $K$ endpoints on each side.
Every left--right pair is a distinct canonical intersection and $[0,m)\notin\mathcal C_K$, so Theorem~\ref{thm:single-compression} gives $N_1(\mathcal C_K)=K^2+1$.

Let $q=\lfloor\log_2(K+1)\rfloor$, $M=2^q\le K+1$, and $d=\min(T,q)$.
Before the first prediction, draw $J$ uniformly from $\{1,\dots,M\}$, independently of the learner's randomness, and fix $g^\star=g_J=[K+J-1,K+J)$, a legal one-token target.
The prefix candidates $s_r=[0,K+r)$, $1\le r\le K$, satisfy $h_{g_J}(s_r)=\1\{J\le r\}$, so the target is a hidden threshold.
Maintain the block $[A,B]$ of indices consistent with past labels, initially $[1,M]$; during the first $d$ rounds its size is a power of two.
Query $s_r$ with $r=A+(B-A+1)/2-1$, a median split with $1\le r\le M-1\le K$; after label $1$ keep $[A,r]$, after label $0$ keep $[r+1,B]$.
The posterior stays uniform on the retained block, so the next label is exactly Bernoulli$(1/2)$ conditional on the public history and the learner's randomness, because the scheduler ignores predictions and the learner's randomness is independent of $J$.
Thus for any prediction, $\mathbb E[(p_t-y_t)^2\mid\text{past},p_t]=(p_t-\tfrac12)^2+\tfrac14\ge\tfrac14$.

Each earlier split separates the current block from its discarded sibling, and each later split is strictly internal to the current block, so queries are distinct along every path.
If $T>d$, query the lexicographically first unused candidates, labeled by the same $g_J$; there are $K^2-d$ of them.
The true target has zero loss, so the realizable loss equals regret; in the agnostic game the same sequence is legal and its best comparator also has zero loss.
Averaging over $J$ shows that for every learner some fixed $J$ attains the bound.
\end{proof}

\begin{appcor}[Worst-case rate]
\label{cor:single-worst}
For $K\ge1$ and $1\le T\le K^2$, let
$W^{\diamond}(K,T)=\sup\{V^{\diamond}_{T,\rm dist}(\mathcal C):\mathcal C\subseteq\mathcal S_m,\ m\ge1,\ \max(a,b)=K,\ |\mathcal C|\ge T\}$ for $\diamond\in\{\rm ag,real\}$.
Then
\[
\begin{aligned}
W^{\diamond}(K,T)&\ge\tfrac14\min\{T,\lfloor\log_2(K+1)\rfloor\},\\
W^{\rm real}(K,T)&\le\min\{T/4,\tfrac12\ln(K^2+1)\},\\
W^{\rm ag}(K,T)&\le\min\{T/4,2\ln(K^2+1)\},
\end{aligned}
\]
so both equal $\Theta(\min\{T,\ln(K+1)\})$ uniformly.
\end{appcor}

\begin{proof}
The agnostic upper bound follows from Corollary~\ref{cor:single-horizon}; the realizable upper bound follows by choosing the constant-half learner or Theorem~\ref{thm:single-version} before play, with $ab\le K^2$ in both cases.
The lower bound is attained by $\mathcal C_K$ for every admissible horizon.
For $x=\log_2(K+1)\ge1$, $\lfloor x\rfloor\ge x/2$, so $\min(T,\lfloor x\rfloor)\ge\tfrac12\min(T,\ln(K+1))$.
\end{proof}

\subsubsection{Same-family target-multiplicity separation}
\label{app:single-separation}

For $K\ge2$, this subsection establishes the two rates of Corollary~\ref{cor:separation} on the \emph{same} candidate family $\mathcal C_K$.
The restriction $K\ge2$ is necessary for the expression $\Theta(\log K)$: for $K=1$ we have $\log K=0$, whereas the fixed-target value is positive.
Indeed, for $\mathcal C_1=\{[0,2)\}$ with $m=3$ and $T=1$, the targets $[0,1)$ and $[2,3)$ give labels $1$ and $0$ on the only candidate; drawing one of them uniformly at time zero makes the label a fair bit, so every learner has expected loss at least $\tfrac14$, and the constant prediction $\tfrac12$ attains $\tfrac14$.
In terms of $\ln(K+1)$, which is positive for all $K\ge1$, Corollary~\ref{cor:single-worst} covers every $K\ge1$.
\begin{proof}[Same-family bounds]
\emph{Multi-target: a $K$-antichain gives $\Omega(K)$, with upper bound $\widetilde O(K)$.}
In $\mathcal C_K$ the candidates $A_i=[i,K+i+1)$, $0\le i<K$, have strictly increasing left and right endpoints and therefore form an antichain of size $K$, the maximum possible since two spans sharing a left endpoint are comparable.
At time zero draw independent fair bits $B_i$ and let $G=\{A_i:B_i=1\}$; then $\1\{\exists g\in G:g\subseteq A_i\}=B_i$.
Query these $K$ candidates first and fill with unused candidates labeled by the same $G$.
For $K\le T\le K^2$ every learner has expected loss at least $K/4$, while the induced binary isotonic comparator has zero loss; requiring $G\ne\varnothing$ by fixing one bit to $1$ leaves $(K-1)/4$.
The general upper bound of Theorem~\ref{thm:compression} gives $\widetilde O(K)$ on this family.

\emph{Fixed-target: the threshold construction gives $\Theta(\log K)$.}
For $K\le T\le K^2$, Theorem~\ref{thm:single-lower} gives a lower bound $\tfrac14\min\{T,\lfloor\log_2(K+1)\rfloor\}=\tfrac14\lfloor\log_2(K+1)\rfloor$, since $\lfloor\log_2(K+1)\rfloor\le K\le T$.
Corollary~\ref{cor:single-horizon} with $N_1(\mathcal C_K)=K^2+1$ gives an upper bound $2\ln(K^2+1)$.
For $K\ge2$, $\lfloor\log_2(K+1)\rfloor\ge\tfrac12\log_2(K+1)\ge\tfrac{1}{2\ln2}\ln K$ and $2\ln(K^2+1)\le2\ln(2K^2)\le6\ln K$, because $2\le K$ implies $2K^2\le K^3$.
Hence the fixed-target value on $\mathcal C_K$ is $\Theta(\log K)$ with absolute constants for all $K\ge2$ and $K\le T\le K^2$.
\end{proof}

\subsubsection{Exact implementations and computational complexity}
\label{app:single-algorithm}

These implementations concern the fixed-target learners; the bounds below count arithmetic operations and assume $\mathcal C\ne\varnothing$.
They are separate from the minimax proof.

\paragraph{Baseline implementation.}
Sort the unique left and right endpoints, compute $U_i$ and $V_j$ in one pass, and output $[\lambda_i,\rho_j)$ for every $(i,j)$ with $\lambda_i<\rho_j$, $\rho_j\le U_i$, and $V_j\le\lambda_i$; add $[0,m)$ if it is not a candidate.
By Theorem~\ref{thm:single-compression} this covers every labeling class exactly once, and costs $O(n\log(n+1)+ab)$.

The agnostic learner keeps one log-weight per representative $[u_h,v_h)$; on query $[l,r)$ it evaluates $b_h=\1\{l\le u_h,\ v_h\le r\}$, predicts $\sum_hw_hb_h/\sum_hw_h$, and adds $-\tfrac12(b_h-y)^2$ to each log-weight (log-sum-exp stabilization does not change predictions).
The realizable learner keeps one alive flag per representative, predicts the alive positive fraction, and deletes representatives whose bit disagrees with the binary label.
Both use $O(N_1(\mathcal C))\subseteq O(K^2)$ time per prediction and update and $O(K^2)$ memory, since $n\le ab\le K^2$.

\paragraph{Faster exact implementation (exponential weights).}
Store canonical representative $[\lambda_i,\rho_j)$ in cell $(i,j)$ of an $a\times b$ array (invalid cells have weight zero) and keep the zero expert as a scalar.
Initialize each valid cell and the zero expert, when present, to mass one.
For query $[l,r)$, the positive experts occupy the valid cells of the endpoint-rank rectangle $\{(i,j):\lambda_i\ge l,\ \rho_j\le r\}$; predict its mass $P$ divided by the total mass $S$.
With $\eta=1/2$, put $c=e^{\eta(2y-1)}$.
Multiplying positive cells by $c$ and leaving every other mass unchanged reproduces Theorem~\ref{thm:single-ew}: the omitted factor $e^{-\eta y^2}$ is common to \emph{all} experts, including the zero expert, and cancels under normalization.

A balanced segment tree per row stores interval sums and lazy multiplicative tags; a prefix sum or prefix multiplication costs $O(\log(b+1))$.
Let $i_0=\min\{i:\lambda_i\ge l\}$ and $j_0=\max\{j:\rho_j\le r\}$, obtained from the endpoint ranks.
Each round consists of these operations:
\[
\begin{aligned}
P&=\sum_{i=i_0}^a\operatorname{Sum}_i(1,j_0),\\
p&=P/S;\\
&\operatorname{Scale}_i(1,j_0,c)\quad(i_0\le i\le a),\\
S&\leftarrow S+(c-1)P.
\end{aligned}
\]
The scale operations follow observation of $y$; the zero-expert scalar stays unchanged and is included in $S$.
Thus prediction and update together cost $O(a\log(b+1))$ per round, including maintenance of the global mass.
Alternatively, column trees use suffix operations on the left-endpoint ranks and visit columns with $\rho_j\le r$, costing $O(b\log(a+1))$.
Choose the cheaper orientation during preprocessing to obtain
\[
\begin{aligned}
&O\bigl(\min\{a\log(b+1),\,b\log(a+1)\}\bigr)\\
&\subseteq O\bigl(K\log(K+1)\bigr)
\end{aligned}
\]
per round.
Enumeration and construction of the trees take $O(n\log(n+1)+ab)$ preprocessing time and $O(ab)$ memory; building all row trees costs $O(ab)$, and the transposed construction has the same cost.

\section{Experimental Details}
\label{app:setup}

\subsection{Natural Questions Candidates}
\label{app:g2}
Coordinates are NQ document-token indices, including HTML tokens; each candidate is a half-open interval.
\emph{Paragraphs:} scanning HTML-marked tokens, a unit opens at \texttt{<P>} and closes at \texttt{</P>} (inclusive); an unclosed paragraph extends to the document end; if no paragraph is found in a nonempty document, the full document is used.
\emph{Sentences:} within a paragraph, a sentence ends after a non-HTML token equal to \texttt{.}, \texttt{?}, or \texttt{!} (punctuation included); a nonempty tail up to the paragraph end is another sentence.
The candidate family is the union of paragraph units, sentence units, and spans covering consecutive groups of $2$, $4$, or $8$ sentence units and of $2$, $4$, or $8$ paragraph units, with group start stride one unit.
Exact duplicates are removed.
If more than $2{,}048$ candidates remain, they are subsampled deterministically, stratified by exact token length.
The generator uses only the document tokens and their HTML flags; it never reads answers or labels.
For the evidence-containment outcome, $G$ is the set of all valid annotated long-answer intervals pooled over annotators, so $Y_s=1$ iff $s$ contains at least one annotated long answer; the geometric statistics in Table~\ref{tab:geometry} do not depend on $G$.

\subsection{ERASER Evidence Inference}
\label{app:eraser}
Geometry is measured on the training split ($7{,}958$ instances, $1{,}924$ documents) and the validation split ($972$ instances, $247$ documents), $8{,}930$ instances over $2{,}171$ documents in total.
All $11{,}665$ annotated evidence spans have valid token offsets that match the official tokens.
Candidates follow the same multigranular construction as for NQ, without a cap (median $T=896$, maximum $7{,}980$), and are generated without access to the evidence.
The number of minimal targets $k_{\min}$ per instance has mean $1.306$, median $1$, 75th/90th/95th percentiles $2$, and maximum $7$; $27.48\%$ of instances have $k_{\min}\ge2$, $2.52\%$ have $k_{\min}\ge3$, and $0.10\%$ have $k_{\min}\ge5$.
Every target is contained in at least one candidate, and on the validation split the outcome is monotone over all $63{,}792{,}172$ containment relations.

\begin{table*}[ht]
\centering
\small
\begin{tabular}{@{}lrrl@{}}
\toprule
Quantity & Natural Questions & ERASER Ev.\ Inference & Role in the theory \\
\midrule
Documents & 3{,}000 & 2{,}171 & \\
Evidence instances & -- & 8{,}930 & \\
Median raw length $m$ & 6{,}336.5 & 4{,}410 & side of the span triangle \\
Median $T=|\mathcal C|$ & 413 & 896 & horizon if all candidates are verified \\
Median $\meff$ & 90.5 & 145 & compressed side (Thm.~\ref{thm:compression}) \\
Median compression $m/\meff$ & $58.37\times$ & $30.10\times$ & reduction of the ambient scale \\
Median width$/\meff$ & 1.0 & 1.0 & lower-bound scale (App.~\ref{app:width-lower}) \\
Documents with width $=\meff$ & 93.43\% & 100\% & upper and lower scales coincide \\
Documents with $\meff\le T\le\meff^2$ & 100\% & 100\% & nontrivial regime (Thm.~\ref{thm:compression}) \\
Instances with $\ge2$ minimal targets & -- & 27.48\% & natural multi-target evidence \\
Order preservation under $\psi$ & exact & exact & Thm.~\ref{thm:compression} \\
\bottomrule
\end{tabular}
\caption{Effective geometry of real evidence-candidate families. Medians are over documents; median width equals median $\meff$ on both datasets ($90.5$ and $145$). ``--'': not applicable.}
\label{tab:geometry}
\end{table*}

\subsection{Structural Measurements}
\label{app:audit-cost}
\emph{Order preservation:} for each document we compare all $T(T-1)$ ordered candidate pairs, checking that $s\subsetneq s'$ holds exactly when $\psi(s)\preceq\psi(s')$ with $\psi(s)\neq\psi(s')$, and that $\psi$ is injective.
This covers $26{,}469{,}705$ strict containments on NQ and $127{,}582{,}641$ on ERASER, with no relation lost or added.
These $O(T^2)$ comparisons validate exactness and are not needed to compute $\psi$, which takes $O(T\log T)$ time (Algorithm~\ref{alg:endpoint-compression}).
\emph{Exact width:} by Dilworth's theorem, the width of a finite poset equals $T$ minus the size of a maximum matching in the bipartite graph with an edge from $s$ to $s'$ whenever $s\subsetneq s'$; we compute the matching with the Hopcroft--Karp algorithm in $O(E\sqrt T)$ time, where $E$ is the number of strict containments, after $O(T^2)$ graph construction.
The solver agrees with brute-force enumeration on small validation cases.

\subsection{Predictive Model and Features}
\label{app:model}
\label{app:predictive}
The ERASER test split contains $959$ instances over $240$ documents and $933{,}403$ candidates; across the four scheduler families (with three random seeds) this yields $5{,}600{,}418$ prediction rows.

\paragraph{BASE.}
BASE is an L2-regularized logistic regression ($C=0.01$, L-BFGS, at most $300$ iterations, tolerance $10^{-4}$) over eleven candidate-static features.
The eight geometry features are the normalized start, end, length, and center of the span, $\log(1+\text{token length})$, the number of overlapping sentences, $\log(1+\text{sentence count})$, and an indicator for the full document.
The three semantic features use the frozen \texttt{BAAI/bge-base-en-v1.5} encoder ($768$ dimensions, with the encoder's retrieval query prefix and L2-normalized embeddings): the maximum, mean, and top-3 mean cosine similarity between the query and the official sentences overlapping the candidate.
Candidates are never encoded as truncated passages.
All feature scaling is fit on training data only.

\paragraph{BASE-H.}
BASE-H adds six generic online-history features that do not use containment: the step index, the normalized step fraction, $\log(1+\text{number seen})$, the numbers of positive and negative labels seen, and the positive fraction among seen labels.

\paragraph{COMBINED$^*$.}
Let $p_{\mathrm{base}}(s)$ be the BASE probability, $h_t(s)$ the six history features, and $c_t(s)$ the sixteen containment features below, all computed before the label of $s$ is revealed. COMBINED$^*$ predicts
\[
\begin{aligned}
p_{\mathrm{comb}}(s,t)
=\sigma\bigl(\beta_0+\beta^\top z\bigl[&\operatorname{logit}(\operatorname{clip}(p_{\mathrm{base}}(s))),\\
&h_t(s),\,c_t(s)\bigr]\bigr),
\end{aligned}
\]
where $z$ is a standardization fit on training data and $\sigma$ is the logistic function; the coefficient on the BASE logit is learned.
It is trained by L2 logistic regression ($C=1.0$, L-BFGS, at most $250$ iterations, tolerance $10^{-4}$, no class reweighting), with $C$ and the containment-feature variant selected on validation.
The sixteen containment features use only previously verified candidates: the numbers of seen strict subspans and superspans; the numbers of positive and negative seen subspans and of positive and negative seen superspans; the $\log(1+\cdot)$ of each of these six counts; the positive fractions among seen subspans and among seen superspans; and the binary indicators
\begin{align*}
\texttt{inner\_positive}&=\1\{\exists\,\text{seen } s'\subsetneq s,\ y_{s'}=1\},\\
\texttt{outer\_negative}&=\1\{\exists\,\text{seen } s'\supsetneq s,\ y_{s'}=0\}.
\end{align*}
No distance-based or nearest-neighbor feature is used.
Parameters and scaling are fixed after training; at test time only the episode's features change as labels are revealed, and each prediction precedes its own feedback.
COMBINED$^*$ applies no hard monotone override: the two indicators enter only as learned features.
The decision threshold $\tau=0.22254114$, fit on validation, affects only thresholded metrics such as balanced accuracy, not Brier scores.

\paragraph{RAND-CONT.}
A random permutation reassigns candidates to the nodes of the containment graph; the graph's topology is unchanged, but its relation to the actual spans, and therefore to the labels, is destroyed.

\subsection{Deterministic Monotone Closure Control}
\label{app:closure}
Closure+BASE predicts $p_{\mathrm{cl}}(s_t)=1$ if
\[
\exists j<t:\ s_j\subsetneq s_t,\ y_j=1,
\]
$p_{\mathrm{cl}}(s_t)=0$ if
\[
\exists j<t:\ s_j\supsetneq s_t,\ y_j=0,
\]
and $p_{\mathrm{cl}}(s_t)=p_{\mathrm{base}}(s_t)$ otherwise.
By monotonicity both rules are always correct, and they never conflict on real trajectories.
Of the $5{,}600{,}418$ test predictions, $4{,}152{,}334$ ($74.1433\%$) are logically determined ($376{,}342$ forced positive, $6.7199\%$; $3{,}775{,}992$ forced negative, $67.4234\%$) and $1{,}448{,}084$ ($25.8567\%$) are unresolved; under equal weighting of scheduler families the determined fraction is $66.8801\%$.
On determined rows, Closure+BASE has raw Brier $0$ and COMBINED$^*$ has raw Brier $5.84\times10^{-6}$, the negligible price of not hard-coding the rules.
On unresolved rows, Closure+BASE equals BASE exactly, and raw Brier is $0.113645$ for BASE and Closure+BASE, $0.098396$ for BASE-H, and $0.044505$ for COMBINED$^*$.
Over all rows, raw Brier is $0.069136$ (BASE), $0.029385$ (Closure+BASE), and $0.011512$ (COMBINED$^*$); the additive decomposition attributes $68.98\%$ of the BASE-to-COMBINED$^*$ reduction to the step from BASE to Closure+BASE and $31.02\%$ to the step from Closure+BASE to COMBINED$^*$.
This is a system-level decomposition of the loss, not a causal attribution to features.
The macro balanced Brier difference between Closure+BASE and COMBINED$^*$ is $0.082868$ (descriptive).

\subsection{Detailed Result Tables}
\label{app:detailed}
Tables~\ref{tab:predictive}, \ref{tab:neural}, and \ref{tab:sequential} give full-precision ERASER results (including pooled balanced Brier), internal NQ DEV neural results, and sequential results summarized in Tables~\ref{tab:main} and~\ref{tab:seq}.

\begin{table*}[ht]
\centering
\small
\begin{tabular}{@{}lcccccc@{}}
\toprule
System & Macro BBr.\,$\downarrow$ & Pooled BBr.\,$\downarrow$ & Raw Brier\,$\downarrow$ & AUPRC\,$\uparrow$ & AUROC\,$\uparrow$ & Bal.\ Acc.\,$\uparrow$ \\
\midrule
STATIC-POSLEN & 0.289556 & 0.281161 & 0.070421 & 0.444671 & 0.849476 & 0.592813 \\
SEMANTIC & 0.303817 & 0.301019 & 0.076375 & 0.338441 & 0.842072 & 0.543903 \\
BASE & 0.278308 & 0.273260 & 0.069136 & 0.457781 & 0.869953 & 0.596526 \\
BASE-H & 0.273206 & 0.274229 & 0.066958 & 0.500378 & 0.875671 & 0.601722 \\
Closure+BASE & 0.128333 & 0.098892 & 0.029385 & 0.884289 & 0.977516 & 0.859424 \\
RAND-CONT & 0.430307 & 0.408977 & 0.357563 & 0.145537 & 0.630237 & 0.567365 \\
\textbf{COMBINED$^*$} & \textbf{0.045465} & \textbf{0.040173} & \textbf{0.011512} & \textbf{0.980260} & \textbf{0.997277} & \textbf{0.942720} \\
\midrule
\multicolumn{7}{@{}l}{\emph{Paired improvement of COMBINED$^*$ in macro balanced Brier, with 95\% document-cluster bootstrap CI}} \\
over BASE & \multicolumn{6}{l}{$0.232844$ \ $[0.226830,\,0.238939]$} \\
over BASE-H & \multicolumn{6}{l}{$0.227742$ \ $[0.221091,\,0.234609]$} \\
over RAND-CONT & \multicolumn{6}{l}{$0.384843$ \ $[0.381948,\,0.387544]$} \\
\bottomrule
\end{tabular}
\caption{Online prediction on ERASER Evidence Inference test data ($959$ instances, $240$ documents). Each metric is computed per scheduler family (random: mean over three seeds) and then averaged with equal weight over the four families. Macro BBr., the primary metric, is per-instance balanced Brier averaged over instances; pooled BBr.\ pools candidates first. CIs: paired bootstrap over documents. Closure+BASE applies the deterministic consequences of monotonicity and otherwise falls back to BASE. COMBINED$^*$ improves over BASE and over Closure+BASE in all four scheduler families (Table~\ref{tab:scheduler}).}
\label{tab:predictive}
\end{table*}

\begin{table}[ht]
\centering
\small
\setlength{\tabcolsep}{4pt}
\begin{tabular}{@{}lcc@{}}
\toprule
System & Doc.-macro Brier$\downarrow$ & $\Delta$ vs.\ N2 [95\% CI] \\
\midrule
N2 & 0.01732 & -- \\
N2+G1 & 0.00692 & 0.01040 [0.00897, 0.01190] \\
N2+G4 & 0.00614 & 0.01118 [0.00970, 0.01275] \\
FULL & \textbf{0.00629} & 0.01103 [0.00957, 0.01255] \\
\bottomrule
\end{tabular}
\caption{Document-macro ordinary Brier of neural span predictors on the $600$ locked NQ DEV evaluation documents ($354{,}856$ candidates). N2: BERT embedding, span length and center; G1: containment flags; G4: structure-local containment estimate; FULL: N2+G1+G4. CIs: paired bootstrap over documents ($10{,}000$ resamples).}
\label{tab:neural}
\end{table}

\begin{table}[ht]
\centering
\small
\setlength{\tabcolsep}{3.5pt}
\begin{tabular}{@{}lrc@{}}
\toprule
History state & Gain & 95\% CI \\
\midrule
S0: none & $-0.0033$ & $[-0.0042, -0.0025]$ \\
S1: open & $0.0090$ & $[0.0069, 0.0113]$ \\
S2: pos.\ closure & $0.7248$ & $[0.6970, 0.7512]$ \\
S3: neg.\ closure & $0.0020$ & $[0.0015, 0.0027]$ \\
\midrule
\textbf{History $-$ none} & $\mathbf{0.0159}$ & $\mathbf{[0.0142, 0.0177]}$ \\
Closure $-$ open & $0.0760$ & $[0.0686, 0.0836]$ \\
\bottomrule
\end{tabular}
\caption{History-conditioned gain (N2 minus FULL ordinary Brier; positive favors FULL) on all locked NQ DEV trajectories ($600$ documents). Upper rows: state-specific gains. Bold: the primary sequential contrast. S1: comparable but non-determining history. ``Closure $-$ open'': gain under closure (S2, S3) minus gain without closure (S0, S1). CIs: paired bootstrap over documents ($10{,}000$ resamples).}
\label{tab:sequential}
\end{table}

\subsection{Per-Scheduler Results and Logical Determination}
\label{app:scheduler}
Table~\ref{tab:scheduler-app} reports macro balanced Brier for BASE, Closure+BASE, and COMBINED$^*$ in each scheduler family, together with the random-seed breakdown. For the random family, scores are averaged over seeds $17$, $42$, and $2027$ per instance before aggregation; overall results in Table~\ref{tab:predictive} weight the four families equally.
For each schedule, the paired document-bootstrap interval for the improvement of COMBINED$^*$ over BASE excludes zero, and all six pairwise schedule comparisons for COMBINED$^*$ remain significant after Holm correction (adjusted $p=0.005997$).

\begin{table}[ht]
\centering
\small
\setlength{\tabcolsep}{4pt}
\begin{tabular}{@{}lrrr@{}}
\toprule
Schedule & BASE & Closure+BASE & COMBINED$^*$ \\
\midrule
Random & 0.278308 & 0.049931 & 0.037721 \\
\quad seed 17 & & 0.049876 & 0.038239 \\
\quad seed 42 & & 0.049941 & 0.037751 \\
\quad seed 2027 & & 0.049976 & 0.037173 \\
Short$\to$long & 0.278308 & 0.018327 & 0.008265 \\
Long$\to$short & 0.278308 & 0.273783 & 0.073379 \\
Semantic-first & 0.278308 & 0.171293 & 0.062494 \\
\bottomrule
\end{tabular}
\caption{Macro balanced Brier ($\downarrow$) by scheduler family on ERASER test data. BASE does not use revealed labels, so its score is identical across schedules.}
\label{tab:scheduler-app}
\end{table}

\begin{table}[ht]
\centering
\small
\setlength{\tabcolsep}{3.5pt}
\begin{tabular}{@{}lrrrr@{}}
\toprule
Schedule & Determ. & BASE & Closure & COMB.$^*$ \\
\midrule
Random & 88.41\% & 0.2783 & 0.0499 & \textbf{0.0377} \\
Short$\to$long & 11.23\% & 0.2783 & 0.0183 & \textbf{0.0083} \\
Long$\to$short & 82.25\% & 0.2783 & 0.2738 & \textbf{0.0734} \\
Semantic-first & 85.63\% & 0.2783 & 0.1713 & \textbf{0.0625} \\
\bottomrule
\end{tabular}
\caption{Macro balanced Brier ($\downarrow$) per scheduler family on ERASER Evidence Inference test data (random: mean over three seeds), with the fraction of logically determined predictions. Closure: Closure+BASE.}
\label{tab:scheduler}
\end{table}

\paragraph{Determination over trajectory progress.}
Splitting each trajectory into ten bins of normalized progress, the fraction of logically determined predictions rises from the first to the last bin as follows: random $62.16\%\to93.44\%$; short-to-long $0.02\%\to47.96\%$; long-to-short $4.92\%\to99.62\%$; semantic-first $48.90\%\to99.87\%$.
Short-to-long scheduling exposes positive inner evidence only late, so closure has almost nothing to propagate early.
Long-to-short scheduling starts with little forced information but rapidly accumulates outer-negative implications.
Random and semantic-first scheduling reveal substantial closure early.

\subsection{Neural Span Prediction on Natural Questions}
\label{app:neural}
The neural study uses the internal NQ DEV partition of $3{,}000$ documents, split into $2{,}400$ training and $600$ locked evaluation documents; it is disjoint from the holdout partition used for geometry, and neither partition is the official NQ test set.
The evaluation documents contain $354{,}856$ candidates, of which $3{,}357$ are positive; $298$ documents contain at least one positive and $302$ contain none.
Both models are feed-forward networks with one hidden layer of width $256$ and dropout $0.1$, trained for a fixed $10$ epochs with AdamW (learning rate $10^{-3}$, weight decay $0.01$, batch size $4{,}096$) and unweighted binary cross-entropy, on frozen BERT [CLS] embeddings, with seeds $17$, $42$, and $2027$; the evaluation documents are not used for early stopping or tuning.
N2 has $770$ inputs ($768$ embedding dimensions, normalized span length, and normalized span center).
FULL adds four structural inputs computed only from previously verified candidates: three binary containment flags (G1) and the probability of a structure-local containment estimator (G4), the fallback of the M5$^*$ predictor in Appendix~\ref{app:nq-pred}; these are the only two structural families in the model.
Table~\ref{tab:neural} reports their scores; the main text uses this pair only for the sequential mechanism analysis.

\paragraph{Results.}
Document-macro ordinary Brier is $0.017319804$ (N2) and $0.0062912899$ (FULL, a $63.7\%$ relative reduction); the paired difference is $0.011028514$ with 95\% CI $[0.0095747626, 0.0125467]$ over $10{,}000$ document resamples.
Decomposing by stratum:

\begin{center}
\small
\setlength{\tabcolsep}{3pt}
\begin{tabular}{@{}lrccc@{}}
\toprule
Docs & $n$ & N2 & FULL & $\Delta$ \\
\midrule
All & 600 & 0.01732 & 0.00629 & 0.01103 \\
Pos.-pres. & 298 & 0.03010 & 0.01184 & 0.01826 \\
Zero-pos. & 302 & 0.00471 & 0.00082 & 0.00389 \\
\bottomrule
\end{tabular}
\end{center}

\noindent The 95\% CIs of the differences are $[0.00957, 0.01255]$ (all), $[0.01597, 0.02073]$ (positive-present), and $[0.00263, 0.00554]$ (zero-positive).

\noindent Positive-present documents contribute $82.24\%$ of the mean improvement and zero-positive documents $17.76\%$.
On positive-present documents, conditional document-macro balanced Brier is $0.37433039$ (N2) and $0.13714559$ (FULL), difference $0.2371848$ with CI $[0.22492763, 0.24917338]$.
The structural features consist of exactly two families, G1 and G4. Feature-family variants give document-macro ordinary Brier $0.006919294$ (N2+G1) and $0.0061352853$ (N2+G4), with improvements over N2 of $0.01040051$ (CI $[0.0089666322, 0.011896913]$) and $0.011184519$ (CI $[0.00970477, 0.012745206]$).
Candidate-pooled descriptive metrics are, for N2 and FULL respectively: Brier $0.0086564979$ and $0.0026861944$; AUPRC $0.16834297$ and $0.83564373$; AUROC $0.91704649$ and $0.9927433$; balanced accuracy at $0.5$, $0.53208891$ and $0.84812621$.
Measured amortized cost is $0.0044474664$\,ms per candidate for N2 and $0.028133907$\,ms for FULL ($6.33\times$).

\subsection{Sequential State Analysis}
\label{app:seq}
\begin{table}[ht]
\centering
\footnotesize
\begin{tabular}{@{}lrc@{}}
\toprule
Setting & Gain & 95\% CI \\
\midrule
S0: no history & $-$.0033 & [$-$.0042, $-$.0025] \\
S1: open & .0090 & [.0069, .0113] \\
S2: positive closure & .7248 & [.6970, .7512] \\
S3: negative closure & .0020 & [.0015, .0027] \\
\midrule
\textbf{History $-$ none} & \textbf{.0159} & [.0142, .0177] \\
Closure $-$ open & .0760 & [.0686, .0836] \\
Destruction & .3155 & [.3053, .3258] \\
\bottomrule
\end{tabular}
\caption{Sequential dependence on verification history.}
\label{tab:seq}
\end{table}

Figure~\ref{fig:appendix-sequential} summarizes how history exposure evolves along trajectories and how the state-conditioned gains relate to the history-destruction control, natural antichains, and the exploratory ERASER analysis.

\begin{figure*}[t]
\centering
\includegraphics[width=\textwidth]{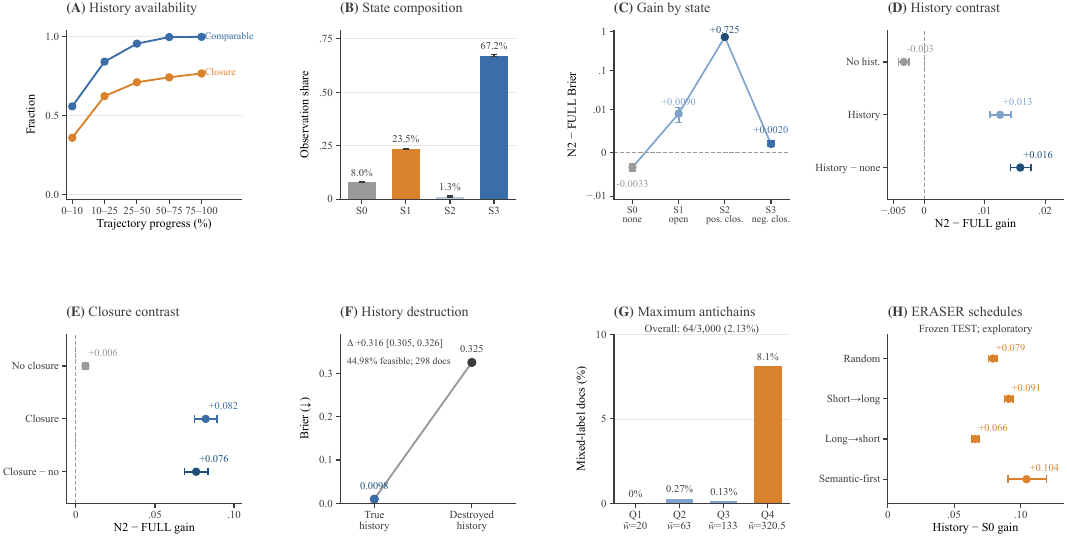}
\caption{Supplementary diagnostics for sequential relevance. (A) Fractions of trajectory observations with comparable history and logical closure across NQ DEV trajectory-progress bins. (B) Overall state composition and (C) state-conditioned N2$-$FULL ordinary-Brier gain. (D--E) History and closure contrasts on NQ DEV. (F) True versus destroyed history on the C2 exact-feasible subset; the interval shown is for destroyed minus true history, and feasibility is $44.98\%$ ($298$ documents). (G) Mixed-label frequency among geometry-selected maximum antichains across width quartiles; the overall rate is $64/3{,}000$ documents ($2.13\%$). (H) History-minus-S0 gain by scheduler on frozen ERASER TEST predictions, shown as exploratory corroboration. Error bars are 95\% paired document-bootstrap confidence intervals where available; panel G is descriptive. NQ and ERASER metrics are not pooled.}
\label{fig:appendix-sequential}
\end{figure*}

\paragraph{States and exposure.}
For each prediction along the frozen NQ DEV trajectories of FULL ($600$ documents, $354{,}856$ candidates), a previously verified candidate is comparable if it strictly contains or is strictly contained in the current span.
The mutually exclusive states are S0 (no comparable history; $120{,}317$ observations in $600$ documents), S1 (comparable, non-determining; $399{,}862$ in $599$), S2 (verified positive subspan; $11{,}524$ in $295$), and S3 (verified negative superspan; $1{,}242{,}577$ in $599$); counts follow trajectories, so they exceed the number of unique candidates.
No conflict between S2 and S3 and no closure violation occurred.
Comparable history is available for $91.9998\%$ of observations (CI $[0.915737, 0.923350]$) and logical closure for $68.4585\%$ (CI $[0.679195, 0.689588]$).
All CIs use $10{,}000$ paired document resamples.

\paragraph{State-conditioned Brier.}
Ordinary Brier for N2 and FULL is $0.010451538$ and $0.013787571$ in S0, $0.046093168$ and $0.037059575$ in S1, $0.724904556$ and $0.000108850$ in S2, and $0.002046139$ and $0.000000066$ in S3.
The gain with history is $0.012557952$ and without history $-0.003339657$; their difference is $0.015897610$ (CI $[0.014247860, 0.017678110]$, $599$ documents with both state types).
The non-closure gain is $0.006027718$ (CI $[0.004457883, 0.007655134]$), the closure gain $0.082061566$ (CI $[0.075024900, 0.089296113]$), and their difference $0.076033848$ (CI $[0.068577885, 0.083600278]$).

\paragraph{History-destruction control.}
The frozen FULL model is not retrained.
For each history event, the counterfactual uses only labels revealed in the observed prefix, preserves prefix length and the numbers of positive and negative labels, and reassigns to the comparable candidates labels taken from distinct previously observed non-comparable candidates (permutation seeds $17$, $42$, $2027$).
This is exactly feasible for $744{,}032$ of $1{,}653{,}963$ history events ($44.98\%$); the others either cannot change any label or lack enough non-comparable candidates and are excluded, with no fallback.
On the feasible subset, Brier is $0.009752124$ with the true history and $0.325284365$ with the destroyed history (difference $0.315532241$, CI $[0.305282041, 0.325756788]$); on S0 events the difference is $-1.86\times10^{-10}$, as expected when no comparable history exists.

\paragraph{Natural antichains.}
For all $3{,}000$ NQ DEV documents, a maximum antichain was selected from span geometry alone before labels were examined.
Median width is $91$ (mean $134.04$), but the selected antichain is mixed-label in only $64$ documents ($2.13\%$; $4.65\%$ among documents with a positive), with mean label entropy $0.0014$ bits (median $0$).
By width quartile, the mixed-label fraction is $0$, $0.27\%$, $0.13\%$, and $8.13\%$ (median widths $20$, $63$, $133$, $320.5$).
On the $600$ evaluation documents, ordinary Brier on antichain members is $0.000272$ (N2) and $0.000472$ (FULL), dominated by class imbalance; on the $20$ documents where conditional balanced Brier is defined, FULL improves it from $0.436$ to $0.353$ (CI of the gain $[0.047, 0.124]$), too few to support a claim about natural hardness.

\paragraph{ERASER corroboration.}
Using only frozen ERASER test trajectories and predictions ($5{,}600{,}418$ rows, $5{,}754$ trajectories, $240$ documents; no conflicts), we compare the BASE$-$COMBINED$^*$ gain in history states with that in S0.
This analysis is post-hoc.

\begin{center}
\small
\begin{tabular}{@{}lcc@{}}
\toprule
Schedule & History $-$ S0 & 95\% CI \\
\midrule
Random & 0.079052 & $[0.075942, 0.082180]$ \\
Short$\to$long & 0.090876 & $[0.087751, 0.094120]$ \\
Long$\to$short & 0.065844 & $[0.063252, 0.068618]$ \\
Semantic-first & 0.104391 & $[0.090418, 0.119283]$ \\
\bottomrule
\end{tabular}
\end{center}

\subsection{Secondary Predictive Evaluation on Natural Questions}
\label{app:nq-pred}
Natural Questions supports the full OECP construction: queries, documents, valid long-answer spans as targets, gold-independent candidates, and exact containment labels.
Before the neural study, an earlier one-shot study on the $3{,}000$-document NQ holdout partition evaluated four simpler online predictors under four schedules (random, short-to-long, long-to-short, left-to-right): a position-histogram predictor (M3), a length-isotonic predictor (M4), a containment-rule predictor with a pooled fallback (M5), and a variant with a structure-local fallback (M5$^*$).
Family-mean results are:

\begin{center}
\small
\begin{tabular}{@{}lrr@{}}
\toprule
Method & Raw Brier $\downarrow$ & Bal.\ Brier $\downarrow$ \\
\midrule
M3 & 0.017356 & 0.213855 \\
M4 & 0.111648 & 0.197339 \\
M5 & 0.004280 & 0.168749 \\
M5$^*$ & 0.039978 & 0.079786 \\
\bottomrule
\end{tabular}
\end{center}

\noindent The containment-aware M5$^*$ substantially improves balanced Brier, while M5 has the lowest raw Brier.
These learners differ from COMBINED$^*$ and FULL, and the holdout partition had previously been used to evaluate structured predictors, so this study is secondary evidence.

\section{Additional Experiment Details}
\label{app:new-exp}

\subsection{Strong Neural Validation on Official NQ DEV}
\label{app:lanea}
\emph{Population.} The frozen official public NQ development/validation evaluation contains $7{,}830$ documents, $4{,}357{,}196$ candidates, and $72{,}913$ positive candidates ($1.673\%$).
Of $5{,}499$ answerable documents, $5{,}065$ ($92.11\%$) have at least one covered target, and $6{,}534$ of $7{,}695$ annotated target spans ($84.91\%$) are covered by some candidate; results are conditional on this frozen candidate family and do not measure answer-level QA accuracy.
The hidden NQ test split was not used. Before this evaluation was fixed, the first $3{,}000$ rows of each of the first three public-validation shards had been inspected in earlier work, so the public split is not an untouched evaluation set.
\emph{Equal-capacity contract.} A0, A1, and A2 are fully fine-tuned DeBERTa-v3-base encoders (all parameters trainable) with the same $384$-token question/candidate/local-context template and the same $768{+}12\to256$ (GELU) $\to1$ head with dropout $0.1$, optimizer, training data, seeds, and calibration. A0 supplies zero history slots, A1 four generic-history fields, and A2 the same four plus eight frozen containment fields.
\emph{Statistics.} The primary contrast is A1$-$A2 document-macro Brier with a paired document bootstrap ($10{,}000$ resamples). The aggregate difference decomposes into a deterministic-closure contribution of $0.016074$ and a learned unresolved contribution of $0.000478$. In the single-target answerable stratum, unresolved Brier is $0.070659$ (A1) versus $0.067066$ (A2). Table~\ref{tab:lanea-seed} reports seed-level results; all three A1$-$A2 differences are positive overall and on unresolved candidates.

\begin{table}[ht]
\centering
\footnotesize
\setlength{\tabcolsep}{3pt}
\begin{tabular}{@{}lrrrrr@{}}
\toprule
Model & Seed & Brier & Unres. & ROC & PRC \\
\midrule
A0 & 17 & 0.027973 & 0.049730 & 0.9046 & 0.2177 \\
A0 & 42 & 0.027184 & 0.048738 & 0.9073 & 0.2358 \\
A0 & 2027 & 0.028030 & 0.049673 & 0.9021 & 0.1934 \\
A1 & 17 & 0.028276 & 0.051024 & 0.9018 & 0.1589 \\
A1 & 42 & 0.028971 & 0.051292 & 0.8856 & 0.1663 \\
A1 & 2027 & 0.026837 & 0.048453 & 0.9238 & 0.2376 \\
A2 & 17 & 0.011774 & 0.050349 & 0.9736 & 0.7987 \\
A2 & 42 & 0.011390 & 0.047092 & 0.9844 & 0.8221 \\
A2 & 2027 & 0.011266 & 0.047286 & 0.9843 & 0.8267 \\
\bottomrule
\end{tabular}
\caption{Seed-level results on official public NQ DEV.}
\label{tab:lanea-seed}
\end{table}

\subsection{Additional Cross-Backbone Neural Results}
\label{app:roberta}
\emph{Protocol.} A0/A1/A2 were re-trained with RoBERTa-base in place of DeBERTa-v3-base, with the same template, head, history fields, official public NQ DEV population ($7{,}830$ documents, $4{,}357{,}196$ candidates), and seeds $17$, $42$, $2027$; the inference protocol was fixed before public predictions, and no backbone was selected after them. CIs use $10{,}000$ paired document resamples (seed $20261005$).
\emph{Results.} Document-macro Brier (mean over seeds) is $0.027924$/$0.027667$/$0.011982$ on all candidates and $0.049457$/$0.049386$/$0.048184$ on unresolved candidates for A0/A1/A2. Unresolved A1$-$A2 is $0.001202$ ($[0.000956, 0.001447]$) and A0$-$A2 $0.001273$ ($[0.001016, 0.001536]$); all-candidate A1$-$A2 is $0.015685$ ($[0.015136, 0.016236]$). Per-seed unresolved A1$-$A2 is $0.001192$, $0.001148$, $0.001265$. From the frozen DeBERTa predictions, unresolved A0$-$A2 is $0.001138$ ($[0.000923, 0.001358]$; Figure~\ref{fig:roberta}).

\begin{figure*}[ht]
\centering
\includegraphics[width=0.9\textwidth]{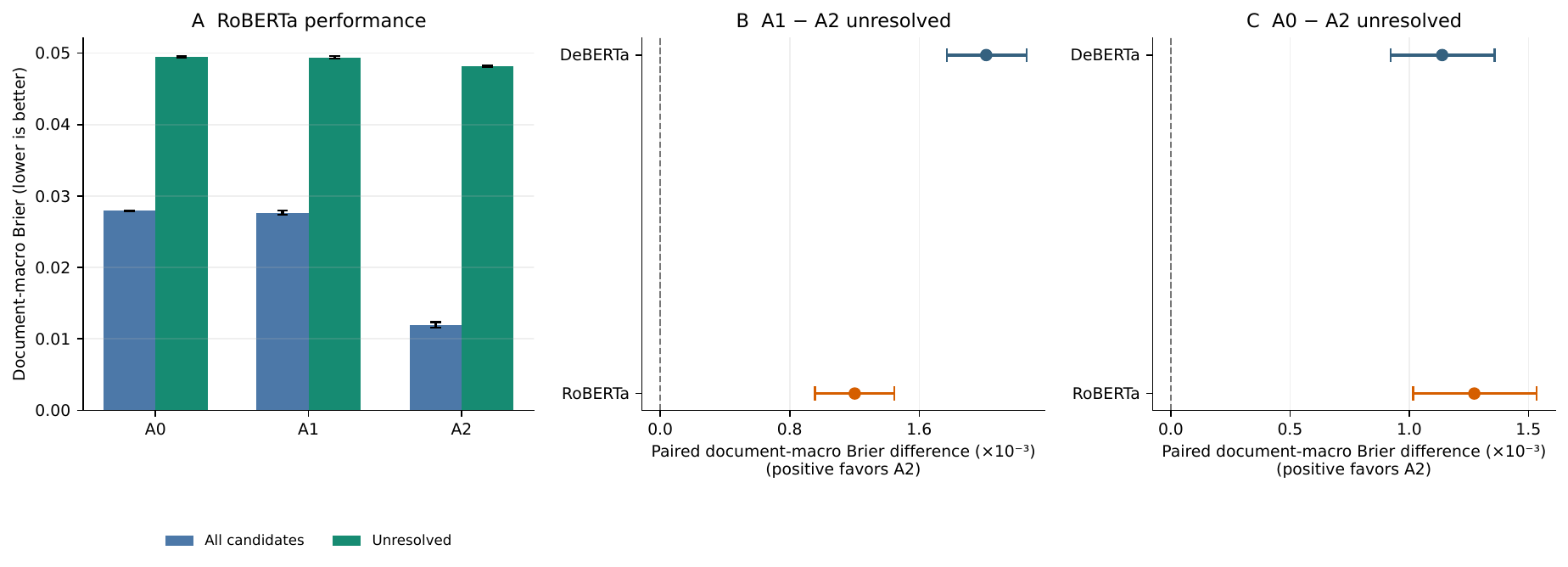}
\caption{Cross-backbone replication on official public NQ DEV. (A) RoBERTa document-macro Brier for A0/A1/A2 on all and unresolved candidates. (B, C) Unresolved A1$-$A2 and A0$-$A2 for DeBERTa and RoBERTa ($\times10^{-3}$; positive favors A2; 95\% paired bootstrap CIs).}
\label{fig:roberta}
\end{figure*}

\subsection{Practical Isotonic Instantiation}
\label{app:laneb}
\emph{Formulation.} For a document with candidate poset $(\mathcal C,\subseteq)$, B1 computes the Euclidean projection of the frozen semantic scores (B0, the N2 predictor of Appendix~\ref{app:neural}) onto $\{f\in[0,1]^{\mathcal C}: f(s)\le f(s')\text{ for every Hasse edge } s\subset s'\}$ with unit weights and no labels. B2 predicts each candidate before its label is revealed and then re-projects the squared-loss objective over revealed labels plus a $\lambda$-weighted B0 prior; $\lambda=0.01$ was selected from $\{0.01,0.1,1,10,100\}$ on $120$ training documents without evaluation labels. Both use a compiled Hildreth dual-coordinate (Dykstra) projection over order and box half-spaces on the Hasse diagram.
\emph{Validation.} On $100$ small induced candidate families the solver matched SciPy SLSQP with maximum objective difference $2.3\times10^{-11}$ (tolerance $10^{-7}$), and the compiled Hasse extractor matched a brute-force transitive reduction. On all $600$ evaluation documents, B1 and the final B2 state have zero violations over all comparable pairs, every prequential B2 state has zero violations on all Hasse edges, and the maximum B2 KKT primal residual is $4.41\times10^{-9}$.
\emph{Results and runtime.} All-candidate Brier is $0.008720$ (B0), $0.008533$ (B1), $0.003067$ (B2); on the $48{,}126$ candidates unresolved before their query it is $0.022777$, $0.023003$, $0.022549$. B2's all-candidate Brier includes predictions after prior outcomes; the unresolved comparison is the stricter view. Median batch projection takes $0.000906$\,s per document, the median online update $0.000634$\,s per verification, and a full online trajectory $0.920$\,s per document (mean $4.02$\,s; median $\meff=94$).
\emph{Compression ablation.} Endpoint-rank compression halves the coordinate payload (median $50\%$) and preserves the Hasse diagram exactly, but leaves the numbers of variables and sparse constraints unchanged; the median raw/compressed runtime ratio is $0.768$, i.e., the compressed pipeline is about $1.30\times$ slower in this implementation. Compression reduces the theoretical scale and storage, not this solver's cost.

\subsection{Direct Instantiation of Fixed-Target Learners}
\label{app:fixedinst}
\emph{Cohort.} Of the $7{,}830$ official public NQ DEV documents, $3{,}640$ admit an exact unique fixed-target realization: exactly one valid distinct target, and stored labels equal to its containment vector on every candidate. No target or label was synthesized or repaired. For these families $|\mathcal C|$ has median $381$ (max $2{,}048$) and $\meff$ median $83$ (max $870$).
\emph{Representation.} The exact number of fixed-target labelings $N_1(\mathcal C)$ has median $721$, 95th percentile $3{,}506$, and maximum $4{,}535$; $N_1(\mathcal C)\le ab+1\le\meff^2+1$ holds for every family, and the compressed representation matches brute-force enumeration on $100$ induced real families.
\emph{Learners and orders.} Exponential weights (EW), its horizon-aware pre-play choice (HA-EW), and the version-space learner are run exactly as in Appendix~\ref{app:single}, CPU-only and without tuning, under six query orders per document (left-to-right, short-to-long, long-to-short, and three seeded random orders), giving $21{,}840$ distinct-query trajectories.
\emph{Results.} Table~\ref{tab:fixedinst} summarizes the guarantees; every trajectory satisfies its bound, and every version-space path reaches a singleton (median round $153$). Medians by order are $3.93$--$4.93$ for EW regret and $1.00$--$2.11$ for version-space loss, all with $100\%$ bound satisfaction. Median exact preprocessing takes $1.52$\,ms per family, prediction/update $0.0054$\,ms (EW) and $0.0024$\,ms (version space) per query, and a learner episode $5.19$\,ms (Figure~\ref{fig:fixedinst}).

\begin{table*}[ht]
\centering
\footnotesize
\setlength{\tabcolsep}{3pt}
\begin{tabular}{@{}llrrr@{}}
\toprule
Learner & Bound & Median & p95 & Satisfied \\
\midrule
EW (regret) & $2\ln N_1(\mathcal C)$ & $4.321$ & $8.239$ & $100\%$ \\
HA-EW (regret) & $\min\{T/4,2\ln N_1(\mathcal C)\}$ & $4.410$ & $8.239$ & $100\%$ \\
Version space (loss) & $\tfrac12\ln N_1(\mathcal C)$ & $1.803$ & $2.486$ & $100\%$ \\
\bottomrule
\end{tabular}
\caption{Theorem-derived fixed-target learners on $3{,}640$ exact-realizable NQ DEV documents ($21{,}840$ trajectories). Cumulative square-loss regret or loss per trajectory.}
\label{tab:fixedinst}
\end{table*}

\begin{figure*}[ht]
\centering
\includegraphics[width=\textwidth]{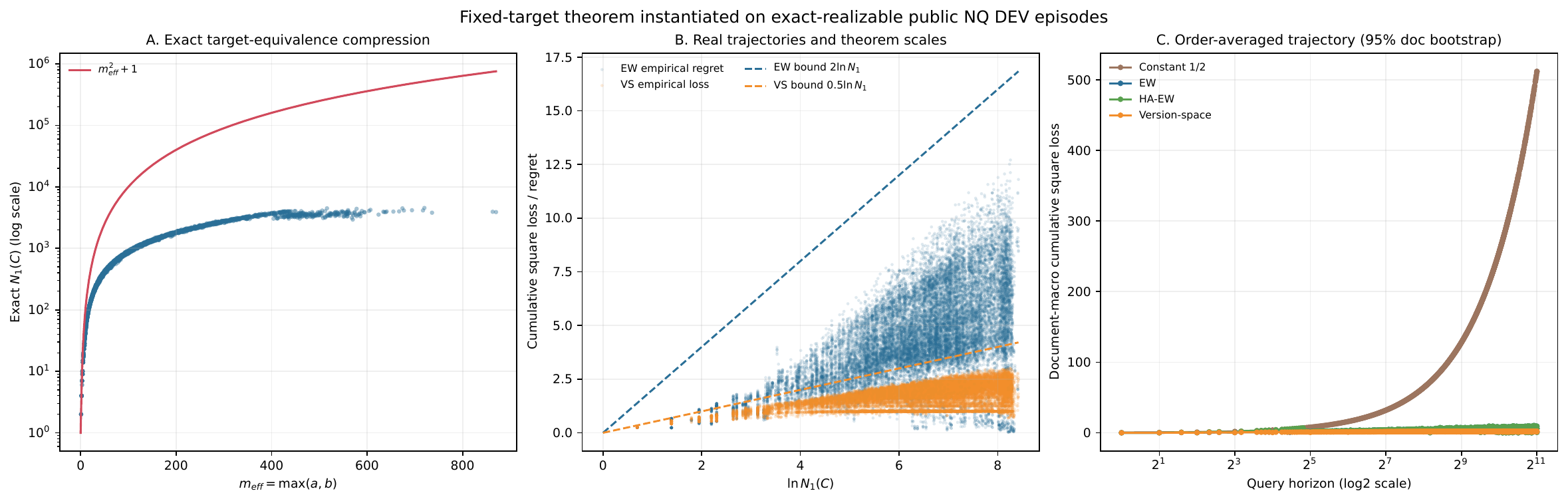}
\caption{Fixed-target theorem instantiation on exact-realizable NQ DEV families. (A) Exact $N_1(\mathcal C)$ against $\meff^2+1$. (B) Per-trajectory EW regret and version-space loss against their bounds. (C) Order-averaged cumulative loss over the query horizon, with the constant-$\tfrac12$ predictor for reference.}
\label{fig:fixedinst}
\end{figure*}

\subsection{Additional Computational Scaling}
\label{app:scaling}
\begin{table}[ht]
\centering
\small
\setlength{\tabcolsep}{4pt}
\begin{tabular}{@{}llrr@{}}
\toprule
Component & Scope & Median & p95 \\
\midrule
Batch isotonic & $1{,}689$ fam. & $0.359$\,ms & $7.267$\,ms \\
ESR decision & $64{,}415$ dec. & $0.162$\,ms & $2.108$\,ms \\
\bottomrule
\end{tabular}
\caption{Computational footprint of the structure-aware components: batch projection on $1{,}689$ real families, all satisfying the containment constraints exactly; ESR over $64{,}415$ decisions from $1{,}089$ families.}
\label{tab:computational-footprint}
\end{table}

We profiled the frozen batch isotonic projection and ESR decision code, unchanged and CPU-only, on $1{,}689$ real candidate families: the $600$ internal NQ DEV documents, $972$ ERASER validation annotation instances, and $117$ SciFact internal-dev instances.
These families span $|\mathcal C|=1$--$3{,}885$ candidates, $0$--$37{,}190$ compressed Hasse edges, and $\meff=1$--$724$; the largest geometry (six ERASER validation instances) has $|\mathcal C|=3{,}885$, $37{,}190$ edges, and $\meff=w=572$.
Per-family batch-projection time is $0.359$\,ms at the median and $7.267$\,ms at the 95th percentile ($221$\,ms on the largest family), with zero monotonicity violations in all families and repeats.
ESR timing covers the $1{,}089$ ERASER and SciFact families with a verified-positive root ($64{,}415$ decisions): $0.162$\,ms at the median, $1.508$\,ms at the 90th, and $2.108$\,ms at the 95th percentile, growing mainly with the active-set size (Figure~\ref{fig:scaling}).

\begin{figure}[ht]
\centering
\includegraphics[width=\columnwidth]{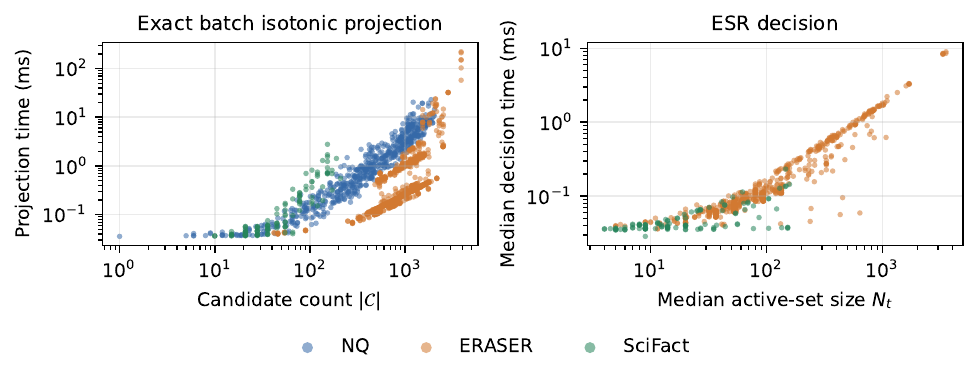}
\caption{Runtime scaling of the two structure-aware components summarized in Table~\ref{tab:computational-footprint}: exact batch isotonic projection over real candidate families (left) and ESR decision scoring over active sets (right); log--log.}
\label{fig:scaling}
\end{figure}

\subsection{Decision-Aware Verification on SciFact}
\label{app:f2}
\emph{Protocol.} A SciBERT containment model was trained on evidence-bearing instances from the official SciFact training split (internal $447/117$ train/dev split by claim) and calibrated by a single temperature ($T=2.75$) on internal dev; a separate frozen SciBERT claim classifier serves as downstream reader. The held-out evaluation is SciFact's labeled validation split ($340$ rows; $209$ evidence-bearing instances in $188$ claim clusters; $130$ NEI groups excluded), read only after the protocol, code, and models were fixed and not used for selection or calibration. Each episode starts from the verified abstract; P-GREEDY verifies the highest-probability active candidate, GEOM-BALANCE minimizes $\max(b_t(s)-1,N_t-b_t(s))$ without probabilities, and ESR minimizes $\mathcal L_t(s)$; all use exact positive/negative transitions and distinct queries. Budgets are $\{1,2,4,8,16,32\}$; AUCs integrate over $\log_2$ budget. CIs use $10{,}000$ paired claim-cluster resamples.
\emph{Results.} Table~\ref{tab:f2-budget} gives per-budget localization. ESR also improves on GEOM-BALANCE in normalized-length AUC ($0.0637$, CI $[0.0486, 0.0790]$) and active-set AUC ($0.0258$, CI $[0.0082, 0.0417]$). Among $212$ shared common-prefix states, ESR and P-GREEDY choose different actions in $207$ ($97.64\%$).
\emph{Downstream reading.} For completeness, downstream macro-F1 (AUC over budgets) is $0.5460$ (P-GREEDY), $0.5432$ (GEOM-BALANCE), and $0.5430$ (ESR), between $0.5373$ for the full abstract and $0.5451$ for the minimal gold region; the policies are statistically indistinguishable on this reader (ESR$-$P-GREEDY $-0.0030$, CI $[-0.0138, 0.0073]$).
\emph{Evaluation split.} The labeled SciFact validation split serves as the held-out evaluation, since official test evidence is not publicly released; no test data were read.

\begin{table}[ht]
\centering
\footnotesize
\setlength{\tabcolsep}{2.6pt}
\begin{tabular}{@{}lrrrrr@{}}
\toprule
Policy & $B$ & Len. & Orc. & Act. & F1 \\
\midrule
P-GREEDY & 1 & 1.0000 & 0.0048 & 1.0000 & 0.5373 \\
P-GREEDY & 2 & 0.8155 & 0.0239 & 0.6550 & 0.5500 \\
P-GREEDY & 4 & 0.5653 & 0.1770 & 0.2768 & 0.5500 \\
P-GREEDY & 8 & 0.3199 & 0.8038 & 0.0301 & 0.5437 \\
P-GREEDY & 16 & 0.2860 & 0.9856 & 0.0054 & 0.5451 \\
P-GREEDY & 32 & 0.2856 & 0.9904 & 0.0030 & 0.5451 \\
GEOM-BAL. & 1 & 1.0000 & 0.0048 & 1.0000 & 0.5373 \\
GEOM-BAL. & 2 & 0.8660 & 0.0096 & 0.4852 & 0.5348 \\
GEOM-BAL. & 4 & 0.4925 & 0.1531 & 0.0992 & 0.5500 \\
GEOM-BAL. & 8 & 0.2857 & 0.9856 & 0.0002 & 0.5451 \\
GEOM-BAL. & 16 & 0.2849 & 1.0000 & 0.0000 & 0.5451 \\
GEOM-BAL. & 32 & 0.2849 & 1.0000 & 0.0000 & 0.5451 \\
ESR & 1 & 1.0000 & 0.0048 & 1.0000 & 0.5373 \\
ESR & 2 & 0.6022 & 0.2297 & 0.3153 & 0.5412 \\
ESR & 4 & 0.4103 & 0.5598 & 0.0946 & 0.5490 \\
ESR & 8 & 0.3073 & 0.9091 & 0.0276 & 0.5386 \\
ESR & 16 & 0.2901 & 0.9809 & 0.0140 & 0.5451 \\
ESR & 32 & 0.2859 & 0.9856 & 0.0081 & 0.5451 \\
\bottomrule
\end{tabular}
\caption{SciFact held-out localization by budget $B$: mean normalized evidence length, fraction reaching the minimal oracle region, mean active-set fraction, and downstream macro-F1.}
\label{tab:f2-budget}
\end{table}

\subsection{ESR Replication on ERASER}
\label{app:eraser-esr}
\emph{Protocol.} ESR, P-GREEDY, the budget grid, and the localization metrics are unchanged from the SciFact protocol (Appendix~\ref{app:f2}); probabilities come from a frozen ERASER probability provider, with no retraining. The evaluation covers all $959$ eligible ERASER instances, i.e., those whose frozen candidate family contains a positive candidate and whose localization root is positive under the existing admissible-target semantics. The minimal oracle region is the shortest positive candidate in the frozen family, with deterministic tie-breaking. CIs use a paired cluster bootstrap with $10{,}000$ resamples (seed $20261005$).
\emph{Results.} Normalized evidence-length AUC is $0.469872$ for P-GREEDY and $0.353517$ for ESR, a paired P-GREEDY$-$ESR difference of $0.116355$ (95\% CI $[0.097527, 0.135707]$). Median verifier calls to the minimal oracle region fall from $16$ (P-GREEDY) to $11$ (ESR). Over shared common-prefix states, ESR and P-GREEDY choose different actions in $83.668\%$ of shared states.

\subsection{Retrieval-Generated Candidate Validation}
\label{app:realgen}
\emph{Generators and pools.} DPR uses the NQ-trained single-NQ question and context encoders \citep{karpukhin2020dense}; it scores every natural paragraph of a document against its question by inner product. BM25 \citep{robertson2009probabilistic} ($k_1=1.5$, $b=0.75$) scores the same document-local paragraphs. For a cutoff $k$, a pool contains the full-document root, the top-$k$ paragraphs, and their natural sentences, deduplicated; no multi-unit runs or sections are added. Pools are built without labels; $k=8$ is primary and $k\in\{4,16\}$ are sensitivity cutoffs fixed in advance.
\emph{Prediction.} The DeBERTa A0/A1/A2 predictors are applied without retraining, in generator-score order (DPR) or a fixed generator-independent random order. Unresolved document-macro Brier is computed on generated-covered documents, with $10{,}000$ paired document resamples (seed $20261005$).
\emph{Learned scheduler.} LEARNED-SCHED is an MLP ($18\to64\to32\to1$, GELU) trained per generator on states from $2{,}400$ NQ TRAIN documents ($1{,}920$ fit, $480$ early stopping), visited by P-GREEDY, GEOM-BALANCE, and random behavior policies in equal thirds. It regresses the negative normalized realized one-step active-set size from $18$ pre-verification features (probability, generator score, span geometry and unit type, branch sizes, and verification history); three seeds are averaged. No public DEV label enters fitting, scaling, or checkpoint selection.
\emph{Policies and metrics.} All policies use the same probabilities and exact transitions. Normalized evidence-length AUC is integrated over $\log_2$ budgets $\{1,\dots,32\}$ on generated-covered documents. SUCCESS@$B$ is computed over all $5{,}499$ answerable documents and equals one only if the pool contains a non-root positive and the policy reaches the minimal oracle region, or a candidate at most $1.25\times$ its length, within $B$ calls.
\emph{Results.} Table~\ref{tab:realgen-k} reports coverage, prediction, and the two ESR contrasts across generators and cutoffs; Table~\ref{tab:realgen-policy} lists all evaluated policies at $k=8$; Figures~\ref{fig:realgen-main} and~\ref{fig:realgen-cov} show the pipeline, prediction, localization curves, coverage, and SUCCESS@$B$.
\emph{Evaluation scope.} Evaluation uses official public NQ DEV, on which the A0/A1/A2 predictors were previously evaluated; the scheduler's TRAIN documents coincide with the predictor's training documents. An initial evaluation run, which had already read public-DEV labels, stopped on an indexing error in the LEARNED-SCHED interface before producing any policy output; the reported results come from the corrected run with identical policies, features, models, seeds, metrics, budgets, and tie rules. Hidden NQ test data were not used; the downstream reader is evaluated in Appendix~\ref{app:downstream}.

\begin{table*}[ht]
\centering
\footnotesize
\setlength{\tabcolsep}{2.4pt}
\begin{tabular}{@{}lrrrlll@{}}
\toprule
Gen. & $k$ & Cov. & $|\mathcal C|$ & A1$-$A2 unres. & LS$-$ESR AUC & ESR$-$GG S@8 \\
\midrule
DPR & 4 & .627 & 22 & .0093 [.0086, .0100] & $-$.0038 [$-$.0068, $-$.0007] & $-$.0071 [$-$.0116, $-$.0025] \\
DPR & 8 & .707 & 40 & .0062 [.0057, .0067] & .0087 [.0055, .0121] & .0109 [.0031, .0189] \\
DPR & 16 & .755 & 74 & .0043 [.0039, .0046] & .0212 [.0176, .0250] & $-$.0042 [$-$.0151, .0067] \\
BM25 & 4 & .531 & 23 & .0070 [.0066, .0074] & $-$.0014 [$-$.0050, .0020] & $-$.0002 [$-$.0038, .0035] \\
BM25 & 8 & .637 & 44 & .0040 [.0037, .0042] & .0154 [.0116, .0193] & .0284 [.0202, .0362] \\
BM25 & 16 & .721 & 80 & .0027 [.0025, .0028] & .0315 [.0270, .0361] & .0440 [.0329, .0551] \\
\bottomrule
\end{tabular}
\caption{Retrieval-generated hierarchies by generator and cutoff $k$: non-root positive coverage over answerable documents, median pool size, unresolved A1$-$A2 Brier (DPR in generator order, BM25 in random order), LEARNED-SCHED$-$ESR length AUC, and ESR$-$GENERATOR-GREEDY SUCCESS@8 (positive favors ESR in the last two). 95\% paired bootstrap CIs.}
\label{tab:realgen-k}
\end{table*}

\begin{table*}[ht]
\centering
\footnotesize
\setlength{\tabcolsep}{3pt}
\begin{tabular}{@{}lrrrrrrrr@{}}
\toprule
& \multicolumn{4}{c}{DPR $k=8$} & \multicolumn{4}{c}{BM25 $k=8$} \\
\cmidrule(lr){2-5}\cmidrule(lr){6-9}
Policy & AUC$\downarrow$ & S@8 & S@16 & Calls & AUC$\downarrow$ & S@8 & S@16 & Calls \\
\midrule
P-GREEDY & .237 & .644 & .659 & 2 & .236 & .592 & .604 & 2 \\
GENERATOR-GREEDY & .275 & .631 & .668 & 3 & .320 & .560 & .608 & 3 \\
GEOM-BALANCE & .371 & .598 & .636 & 4 & .387 & .550 & .588 & 5 \\
LEARNED-SCHED & .256 & .628 & .651 & 2 & .264 & .573 & .599 & 2 \\
ESR & .247 & .642 & .656 & 2 & .248 & .588 & .603 & 2 \\
\bottomrule
\end{tabular}
\caption{Full policy comparison on $k=8$ retrieval-generated pools: normalized evidence-length AUC on generated-covered documents, SUCCESS@8/16 over all answerable documents, and median calls to the minimal oracle region.}
\label{tab:realgen-policy}
\end{table*}

\begin{figure*}[ht]
\centering
\includegraphics[width=\textwidth]{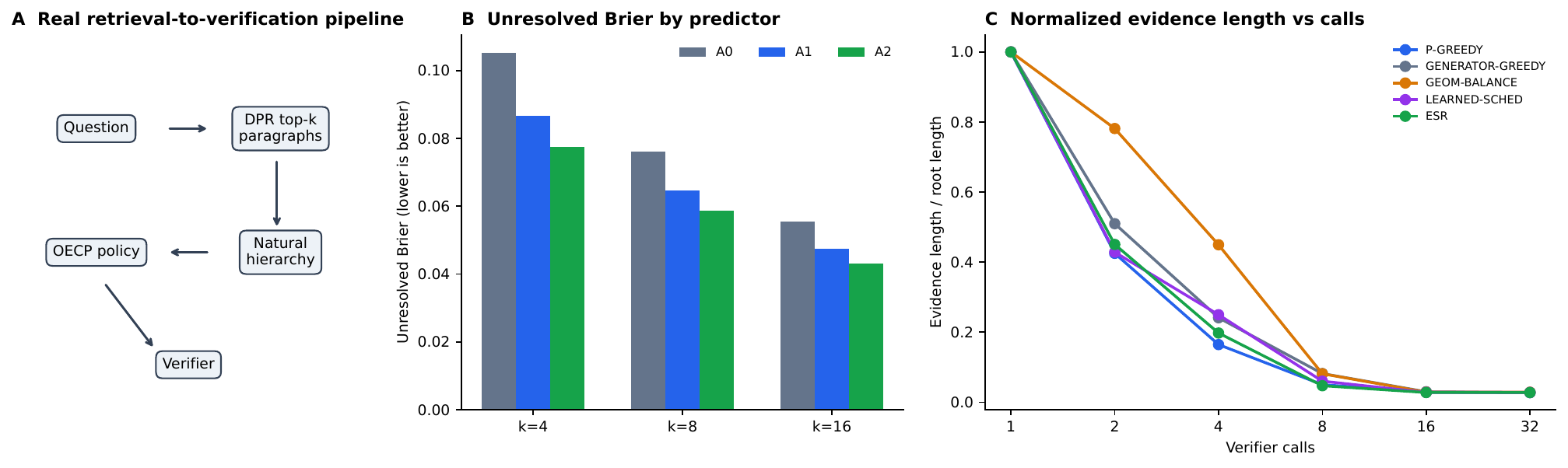}
\caption{Retrieval-generated validation with DPR. (A) Pipeline: DPR top-$k$ paragraphs induce a natural hierarchy verified under an OECP policy. (B) Unresolved Brier of A0/A1/A2 by cutoff $k$. (C) Normalized evidence length against verifier calls at $k=8$.}
\label{fig:realgen-main}
\end{figure*}

\begin{figure*}[ht]
\centering
\includegraphics[width=0.75\textwidth]{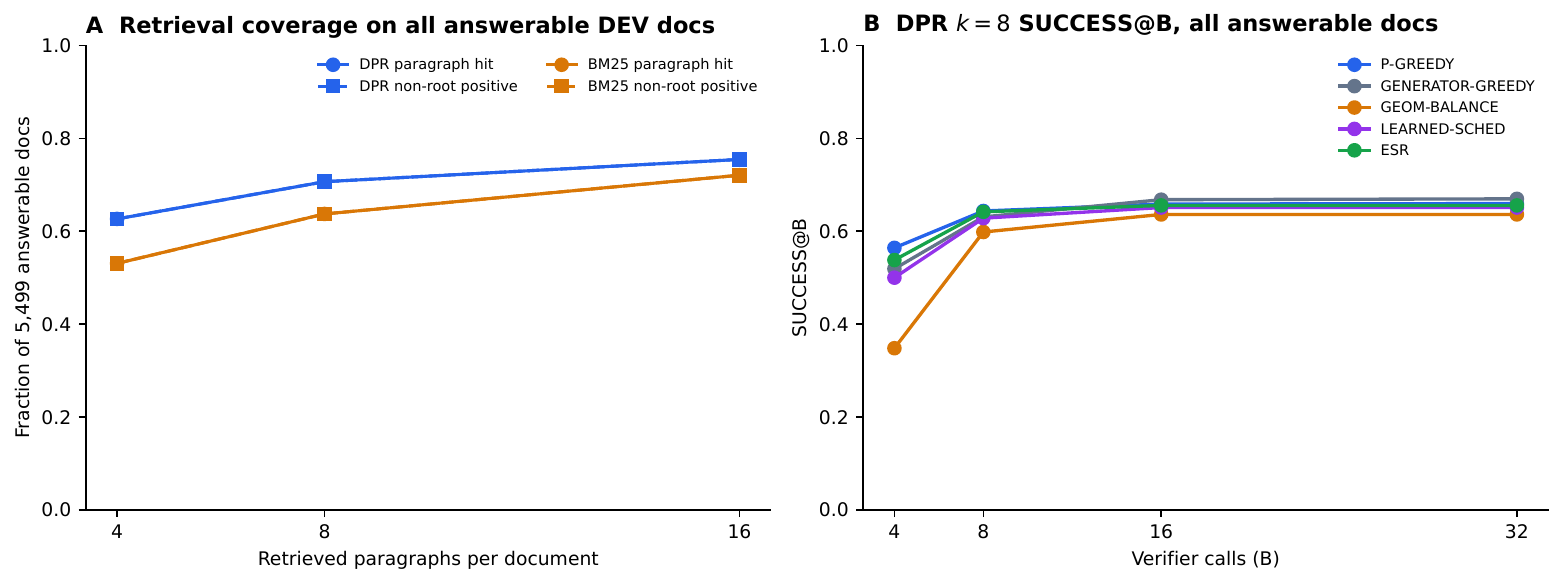}
\caption{(A) DPR and BM25 paragraph hit rate and non-root positive coverage over $5{,}499$ answerable documents. (B) DPR $k=8$ SUCCESS@$B$ over all answerable documents.}
\label{fig:realgen-cov}
\end{figure*}

\subsection{Recent External Selector Comparison}
\label{app:recent}
\emph{Setting.} All selectors act on the frozen DPR $k=8$ pools of Appendix~\ref{app:realgen} ($3{,}887$ generated-covered documents for AUC; $5{,}499$ answerable documents for SUCCESS, with retrieval misses scored zero), with the same binary containment verifier, active-set transition, verifier-call budgets $\{1,2,4,8,16,32\}$, and verification-token budgets $\{64,\dots,2048\}$ (non-root candidate tokens; the shared root initialization is excluded). No method is trained or tuned on public DEV, and hidden TEST is not used. All label-free rankings, features, and adapter configurations were fixed before public-DEV labels were read. CIs are paired document percentile bootstraps with $10{,}000$ resamples (seed $20261006$).
\emph{Methods and fidelity.} REARANK runs the official listwise reasoning reranker \citep{zhang2025rearank} (checkpoint \texttt{le723z/Rearank-7B}, revision \texttt{530a7e15}) once per frozen DPR query; verification follows the frozen ranking over active candidates. NEST-Adapt applies nested evidence survival \citep{verma2026nest} over the root/paragraph/sentence hierarchy with a fixed survival-consistent top-$20$ order. RCD-Adapt keeps the official budget-aware routing objective and coefficients \citep{qureshi2026budgetrouting} (commit \texttt{73e22b3c}) with frozen TF-IDF features and token verification cost; Cost-MMR-Adapt keeps the official MMR objective \citep{wu2026knowledgefree} (commit \texttt{dc68c0dd}) over frozen \texttt{Qwen3-Embedding-0.6B} representations, replacing source-access cost by TRAIN-normalized token cost; both adapters reproduce the official objective values on small test sets. FunnelRAG \citep{zhao2025funnelrag} is not reproducible here: its progressive stage scorer cannot be mapped faithfully onto the frozen pools, and no substitute selector is reported.
\emph{Reproducibility note.} One NEST-Adapt configuration description was regenerated after the label-free NEST rankings were fixed, and its original text is unavailable; the NEST rankings, survival scores, and candidate inputs themselves were verified unchanged before labels were read.
\emph{Probabilities and isotonic correction.} In all DPR policy runs, ESR, P-GREEDY, and LEARNED-SCHED read the same frozen, Platt-calibrated, seed-averaged A2 probabilities; no isotonic projection is applied to any policy, so isotonic correction (Section~\ref{sec:esr}) is an optional module of the OECP pipeline rather than part of the evaluated ESR. External selectors use their own label-free scores, listed in Table~\ref{tab:recent-fidelity}; no corrected-score control for the external selectors was run.
\emph{Results.} Table~\ref{tab:recent-full} lists all policies under both budget types and Table~\ref{tab:recent-contrasts} the paired contrasts; Figure~\ref{fig:recent} shows the budget curves and call-AUC differences. Under verifier-call budgets ESR has lower evidence-length AUC than every runnable external selector; NEST-Adapt is the strongest ($+.0612$ $[.0548,.0677]$). Under token budgets REARANK ($.2872$) and NEST-Adapt ($.3065$) are below ESR ($.3197$), while ESR remains below Cost-MMR-Adapt and RCD-Adapt. On SUCCESS@8, NEST-Adapt ($.6434$) is marginally above ESR ($.6419$), followed by REARANK ($.6265$).
\emph{Runtime.} Mean CPU action-selection time per decision, excluding offline ranking and feature encoding, is $0.016$\,ms (REARANK), $0.007$\,ms (NEST-Adapt), $0.011$\,ms (RCD-Adapt), and $0.260$\,ms (Cost-MMR-Adapt); REARANK ranking and NEST/Cost-MMR encoding run offline on one GPU, with unrecorded wall time.

\begin{table*}[t]
\centering
\scriptsize
\setlength{\tabcolsep}{3pt}
\begin{tabular}{@{}>{\raggedright\arraybackslash}p{1.6cm}>{\raggedright\arraybackslash}p{1.5cm}>{\raggedright\arraybackslash}p{2.2cm}>{\raggedright\arraybackslash}p{2.4cm}>{\raggedright\arraybackslash}p{2.4cm}>{\raggedright\arraybackslash}p{1.4cm}>{\raggedright\arraybackslash}p{1.3cm}>{\raggedright\arraybackslash}p{1.5cm}@{}}
\toprule
Method & Fidelity class & Source / revision & Frozen & Adapted & Score source & Tuning data & Isotonic / contain. \\
\midrule
REARANK & paper-faithful adapter & \texttt{le723z/Rearank-7B}, rev.\ \texttt{530a7e15} & official checkpoint; one listwise ranking per DPR query & ranking replayed over active candidates & LLM listwise rank & none & no / shared transition only \\
NEST-Adapt & hierarchy adapter & ACL 2026 Industry; frozen adapter & survival scores; top-$20$ order & root/paragraph/sentence hierarchy as nest levels & survival/MRR & none & no / shared transition only \\
RCD-Adapt & cost adapter & commit \texttt{73e22b3c} & official objective, coefficients, TF-IDF features & token-length verification cost & routing objective & none & no / shared transition only \\
Cost-MMR-Adapt & cost adapter & commit \texttt{dc68c0dd} & official MMR objective; Qwen3-Emb.-0.6B embeddings & source-access cost $\to$ token cost & MMR score & TRAIN token scale & no / shared transition only \\
ESR & OECP & frozen policy code & A2 probabilities & -- & calibrated A2 & NQ TRAIN (A2) & no / ESR uses order \\
\bottomrule
\end{tabular}
\caption{Configuration of recent external selectors on frozen DPR ($k=8$) pools. All methods share candidate pools, the exact containment verifier and active-set transition, call budgets $\{1,\dots,32\}$ (non-root verifications after the root) and token budgets $\{64,\dots,2048\}$ (non-root candidate tokens). At each step the action is fixed before its label is revealed.}
\label{tab:recent-fidelity}
\end{table*}

\begin{table*}[!t]
\centering
\footnotesize
\setlength{\tabcolsep}{4pt}
\begin{tabular}{@{}lrrrrrrr@{}}
\toprule
& \multicolumn{4}{c}{Verifier-call budget} & \multicolumn{3}{c}{Verification-token budget} \\
\cmidrule(lr){2-5}\cmidrule(lr){6-8}
Policy & AUC$\downarrow$ & S@8$\uparrow$ & Hit@32$\uparrow$ & Med.\ calls & AUC$\downarrow$ & S@512$\uparrow$ & Med.\ tokens \\
\midrule
P-GREEDY & \textbf{.2365} & \textbf{.6438} & .9118 & 2 & .2925 & .6037 & 144 \\
GENERATOR-GREEDY & .2752 & .6310 & .9259 & 3 & .3080 & .5921 & 151 \\
GEOM-BALANCE & .3710 & .5985 & .8778 & 4 & .5012 & .4186 & 386 \\
LEARNED-SCHED & .2560 & .6283 & .8968 & 2 & .3428 & .5363 & 160 \\
\midrule
REARANK & .3239 & .6265 & \textbf{.9264} & 3 & \textbf{.2872} & \textbf{.6179} & 149 \\
NEST-Adapt & .3085 & .6434 & .9259 & 3 & .3065 & .6021 & 162 \\
RCD-Adapt & .8639 & .0687 & .6620 & 22 & .6898 & .1704 & 789 \\
Cost-MMR-Adapt & .6758 & .2944 & .7955 & 8 & .5075 & .4679 & 274 \\
\midrule
ESR & .2473 & .6419 & .9048 & 2 & .3197 & .5672 & 151 \\
\bottomrule
\end{tabular}
\caption{Full policy comparison on frozen DPR ($k=8$) pools. AUC: normalized evidence-length AUC on generated-covered documents; S@$B$: coverage-aware SUCCESS at budget $B$ over all answerable documents; Hit@32: minimal-oracle hit rate at $32$ calls; Med.\ calls/tokens: median calls or verified tokens to the minimal oracle region among hits. Bold: column best for AUC, S@$B$, and Hit@32.}
\label{tab:recent-full}
\end{table*}

\makeatletter
\setlength{\@dblfptop}{0pt}\setlength{\@dblfpbot}{0pt plus 1fil}
\makeatother
\begin{figure*}[tp]
\centering
\begin{minipage}[b]{0.33\textwidth}\centering\includegraphics[width=\linewidth]{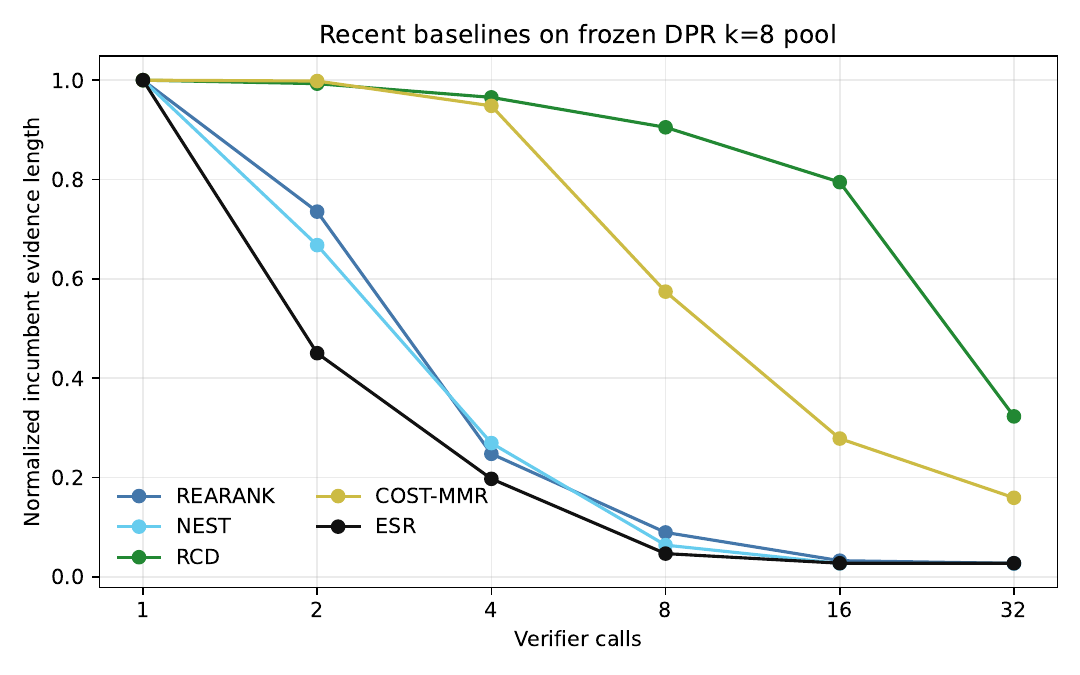}\\{\footnotesize (a)}\end{minipage}\hfill
\begin{minipage}[b]{0.33\textwidth}\centering\includegraphics[width=\linewidth]{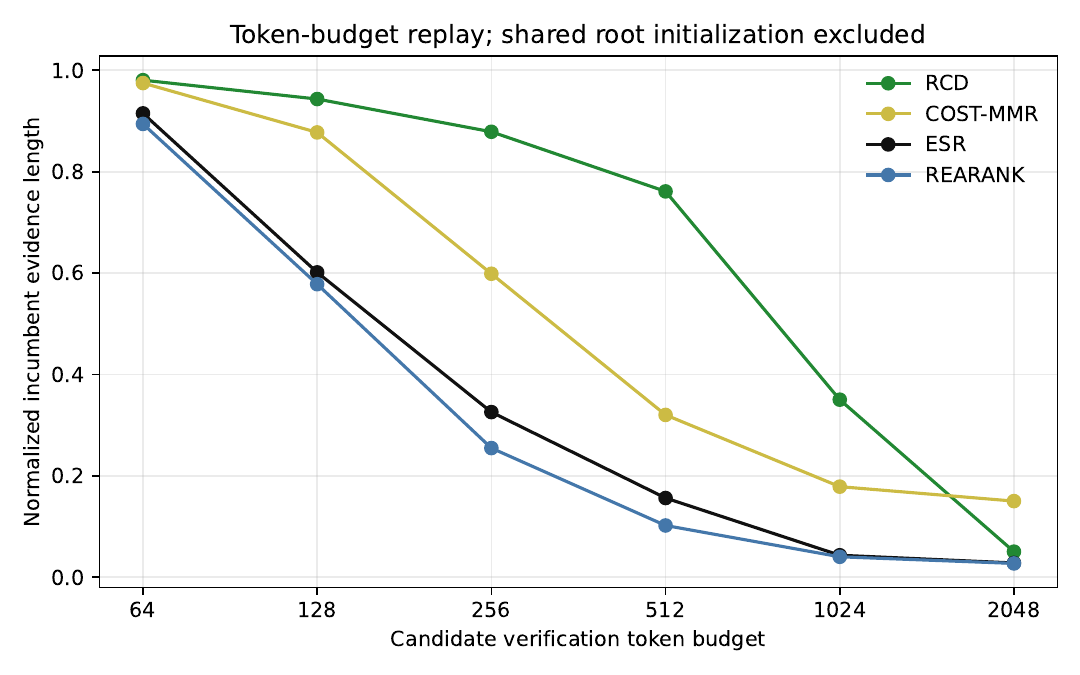}\\{\footnotesize (b)}\end{minipage}\hfill
\begin{minipage}[b]{0.33\textwidth}\centering\includegraphics[width=\linewidth]{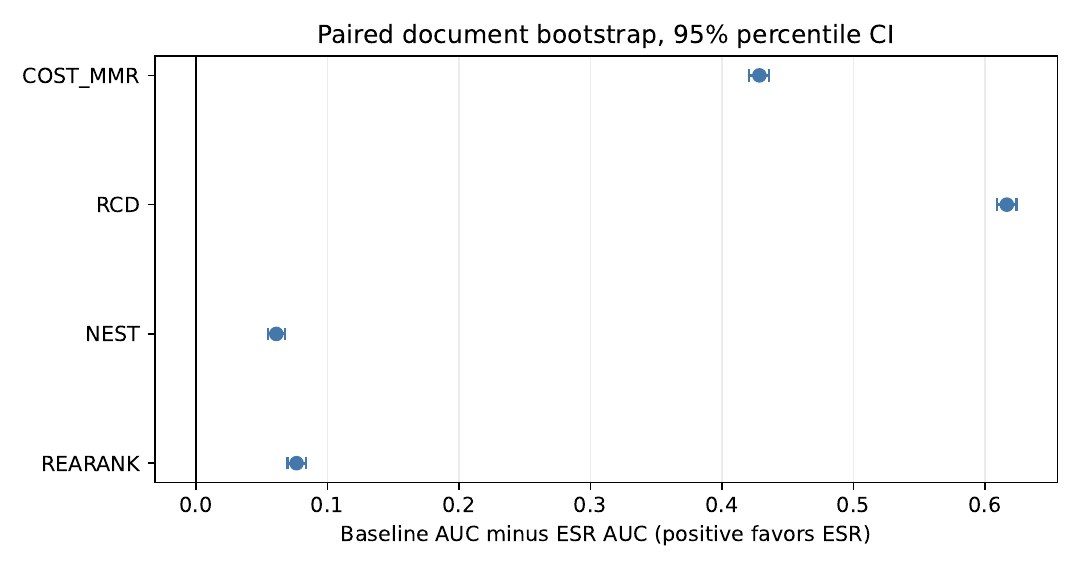}\\{\footnotesize (c)}\end{minipage}
\caption{Recent external selectors on frozen DPR ($k=8$) pools: normalized evidence length versus (a) verifier calls and (b) verification-token budget (shared root initialization excluded); (c) paired call-AUC differences, method$-$ESR, with 95\% CIs (positive favors ESR).}
\label{fig:recent}

\vspace{14pt}
\begin{minipage}[t]{0.48\textwidth}
\centering
\footnotesize
\setlength{\tabcolsep}{2.5pt}
\begin{tabular}{@{}llrr@{}}
\toprule
Track & Contrast & Diff. & 95\% CI \\
\midrule
Call & NEST$-$ESR & .0612 & \mbox{[.0548,.0677]} \\
Call & REARANK$-$ESR & .0766 & \mbox{[.0697,.0835]} \\
Call & C-MMR$-$ESR & .4285 & \mbox{[.4208,.4360]} \\
Call & RCD$-$ESR & .6166 & \mbox{[.6093,.6240]} \\
Token & C-MMR$-$ESR & .1878 & \mbox{[.1785,.1972]} \\
Token & RCD$-$ESR & .3701 & \mbox{[.3619,.3783]} \\
\bottomrule
\end{tabular}
\captionof{table}{Registered paired contrasts of evidence-length AUC, method$-$ESR (positive favors ESR), on $3{,}887$ generated-covered documents. C-MMR: Cost-MMR-Adapt. Token AUCs of REARANK and NEST-Adapt are reported descriptively in Table~\ref{tab:recent-full}.}
\label{tab:recent-contrasts}
\end{minipage}\hfill
\begin{minipage}[t]{0.48\textwidth}
\centering
\footnotesize
\setlength{\tabcolsep}{3.0pt}
\begin{tabular}{@{}lrrrrrrr@{}}
\toprule
Calls & 1 & 2 & 4 & 8 & 16 & 32 & 64 \\
\midrule
Macro-F1 & .614 & .595 & .642 & .718 & .805 & .865 & \textbf{.879} \\
\midrule
\multicolumn{8}{@{}l}{Oracle .887; Gold .912} \\
\multicolumn{8}{@{}l}{16 calls: length .241; ${\le}1.25\times$ oracle .540} \\
\bottomrule
\end{tabular}
\captionof{table}{Evidence localization$^\dagger$.}
\label{tab:loc}

\vspace{10pt}
\centering
\footnotesize
\setlength{\tabcolsep}{3.2pt}
\begin{tabular}{@{}lrrrrrr@{}}
\toprule
 & Brier$\downarrow$ & Top1+ & Sh.+ & Top1$=$ & Rk$=$ & $\tau/\rho$ \\
\midrule
BASE & .2185 & 34.9\% & 14.7\% & -- & -- & -- \\
COMB.$^*$ & \textbf{.2118} & 34.9\% & 14.7\% & 100\% & 100\% & 1/1 \\
\bottomrule
\end{tabular}
\captionof{table}{Frontier ranking audit$^\dagger$.}
\label{tab:rank}
\end{minipage}

\vspace*{\fill}\end{figure*}

\subsection{Downstream Answer Utility on NQ}
\label{app:downstream}
\emph{Cohort.} We evaluate downstream answer quality on the official NQ public-DEV short-answer/yes-no subset, with cohort membership determined solely by official annotation type and the frozen DPR answer normalization. Of the $5{,}499$ answerable documents of Appendix~\ref{app:realgen}, $4{,}287$ have short-answer text and $173$ are yes/no-only; $1{,}037$ are long-answer-only and $2$ have no scoreable short target (both normalize to the empty string), so $D_{\rm SA}$ has $N=4{,}460$ documents with full gold coverage. Within $D_{\rm SA}$, $3{,}272$ documents are generated-covered and $1{,}188$ are retrieval misses. The $1{,}039$ annotation-type exclusions are outside the evaluated population.
\emph{Reader and protocol.} The reader is \texttt{facebook/dpr-reader-single-nq-base} (revision \texttt{38f47a49}), used frozen in FP32 without fine-tuning or DEV tuning (maximum input $512$, maximum answer length $30$). For each policy and budget, the reader receives the latest verified-positive non-root incumbent of the frozen trajectories of Appendix~\ref{app:realgen}; if none has been acquired, the system abstains. Coverage-aware F1 and EM average over all of $D_{\rm SA}$, scoring retrieval misses and abstentions as zero; answers are normalized as in the DPR validator (lowercase, punctuation and articles removed), and each official short-answer span is an alias. Call budgets are $\{1,\dots,32\}$ verifier calls and token budgets $\{64,\dots,2048\}$ verified tokens; AUCs integrate over $\log_2$ budget. Paired document bootstraps use $10{,}000$ resamples (seed $20261006$). As references, the reader on all available DPR top-8 paragraphs (DPR8) reaches F1 $.4505$ and EM $.3726$ on $D_{\rm SA}$; on the minimal oracle evidence it reaches F1 $.6288$ and EM $.5238$ on the $3{,}272$ generated-covered documents (a conditional diagnostic).
\emph{Results.} Table~\ref{tab:ds-full} reports all five policies, Table~\ref{tab:ds-contrasts} the registered paired contrasts, and Figure~\ref{fig:downstream} the budget curves. Under call budgets ESR improves F1-AUC and F1@8 over both LEARNED-SCHED and GENERATOR-GREEDY; under token budgets it improves over LEARNED-SCHED and is below GENERATOR-GREEDY. All five policies reach $90\%$ and $95\%$ of DPR8 F1 at $8$ calls; under token budgets P-GREEDY and GENERATOR-GREEDY reach $90\%$ at $512$ tokens and all policies reach $95\%$ at $1{,}024$. Table~\ref{tab:ds-fail} decomposes outcomes at $8$ and $16$ calls.

\begin{figure*}[p]
\centering
\includegraphics[width=\textwidth]{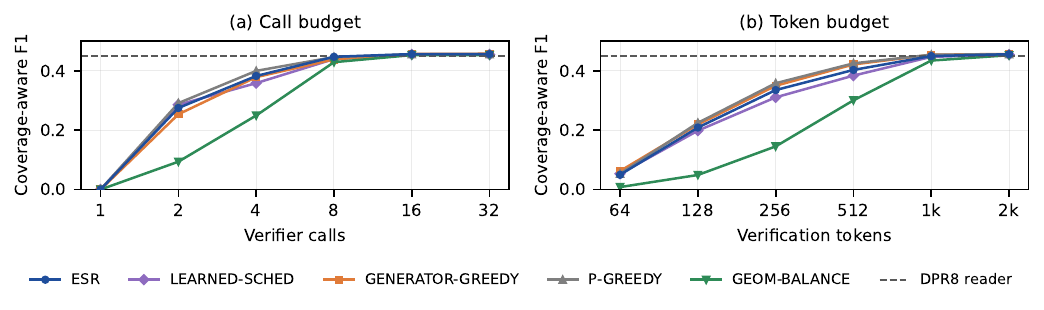}
\caption{Coverage-aware answer F1 of the frozen DPR reader on the official NQ public-DEV short-answer/yes-no subset ($N=4{,}460$) as a function of (a) verifier calls and (b) verified tokens; the dashed line is the DPR8 reader reference.}
\label{fig:downstream}

\vspace{10pt}
\footnotesize
\setlength{\tabcolsep}{5pt}
\begin{tabular}{@{}lrrrrrrrr@{}}
\toprule
& \multicolumn{4}{c}{Verifier-call budget} & \multicolumn{4}{c}{Verification-token budget} \\
\cmidrule(lr){2-5}\cmidrule(lr){6-9}
Policy & F1-AUC & EM-AUC & F1@8 & F1@16 & F1-AUC & EM-AUC & F1@512 & F1@1024 \\
\midrule
P-GREEDY & \textbf{.3638} & \textbf{.3037} & .4458 & .4552 & \textbf{.3430} & \textbf{.2870} & \textbf{.4246} & \textbf{.4538} \\
GENERATOR-GREEDY & .3509 & .2933 & .4381 & \textbf{.4563} & .3404 & .2848 & .4210 & .4537 \\
GEOM-BALANCE & .2900 & .2412 & .4286 & .4524 & .2317 & .1927 & .3003 & .4341 \\
LEARNED-SCHED & .3528 & .2954 & .4392 & .4543 & .3183 & .2673 & .3834 & .4466 \\
ESR & .3574 & .2991 & \textbf{.4467} & .4559 & .3295 & .2765 & .4025 & .4492 \\
\bottomrule
\end{tabular}
\captionof{table}{Coverage-aware answer quality of the frozen DPR reader on $D_{\rm SA}$ ($N=4{,}460$) for all five frozen policies (higher is better; best value per column in bold). DPR8 reference: F1 $.4505$, EM $.3726$.}
\label{tab:ds-full}

\vspace{10pt}
\footnotesize
\setlength{\tabcolsep}{4pt}
\begin{tabular}{@{}llrlrl@{}}
\toprule
Track & Endpoint & ESR$-$LS & 95\% CI & ESR$-$GG & 95\% CI \\
\midrule
Call & F1-AUC & .0046 & \mbox{[.0022,.0070]} & .0065 & \mbox{[.0017,.0113]} \\
Call & EM-AUC & .0037 & \mbox{[.0013,.0060]} & .0058 & \mbox{[.0009,.0107]} \\
Call & F1@8 & .0075 & \mbox{[.0043,.0108]} & .0086 & \mbox{[.0030,.0142]} \\
Call & EM@8 & .0058 & \mbox{[.0027,.0092]} & .0078 & \mbox{[.0022,.0135]} \\
Call & F1@16 & .0016 & \mbox{[.0001,.0031]} & $-$.0004 & \mbox{[$-$.0038,.0029]} \\
Call & EM@16 & .0011 & \mbox{[$-$.0004,.0029]} & .0011 & \mbox{[$-$.0029,.0049]} \\
Token & F1-AUC & .0112 & \mbox{[.0086,.0138]} & $-$.0108 & \mbox{[$-$.0164,$-$.0054]} \\
Token & EM-AUC & .0093 & \mbox{[.0068,.0118]} & $-$.0083 & \mbox{[$-$.0139,$-$.0028]} \\
Token & F1@512 & .0191 & \mbox{[.0146,.0239]} & $-$.0185 & \mbox{[$-$.0264,$-$.0106]} \\
Token & EM@512 & .0161 & \mbox{[.0117,.0206]} & $-$.0146 & \mbox{[$-$.0222,$-$.0070]} \\
Token & F1@1024 & .0026 & \mbox{[.0007,.0047]} & $-$.0045 & \mbox{[$-$.0087,$-$.0004]} \\
Token & EM@1024 & .0022 & \mbox{[.0002,.0043]} & $-$.0022 & \mbox{[$-$.0067,.0022]} \\
\bottomrule
\end{tabular}
\captionof{table}{Registered paired contrasts on $D_{\rm SA}$ (ESR minus comparator; positive favors ESR). LS: LEARNED-SCHED; GG: GENERATOR-GREEDY.}
\label{tab:ds-contrasts}

\vspace{10pt}
\footnotesize
\setlength{\tabcolsep}{4pt}
\begin{tabular}{@{}lrrrrrrrrrr@{}}
\toprule
& \multicolumn{5}{c}{8 calls} & \multicolumn{5}{c}{16 calls} \\
\cmidrule(lr){2-6}\cmidrule(lr){7-11}
Policy & Miss & No pos. & Zero & Partial & Exact & Miss & No pos. & Zero & Partial & Exact \\
\midrule
P-GREEDY & 1188 & 70 & 793 & 752 & 1657 & 1188 & 5 & 810 & 765 & 1692 \\
GENERATOR-GREEDY & 1188 & 160 & 750 & 733 & 1629 & 1188 & 6 & 799 & 776 & 1691 \\
GEOM-BALANCE & 1188 & 165 & 792 & 725 & 1590 & 1188 & 1 & 824 & 771 & 1676 \\
LEARNED-SCHED & 1188 & 105 & 801 & 728 & 1638 & 1188 & 1 & 821 & 759 & 1691 \\
ESR & 1188 & 60 & 803 & 745 & 1664 & 1188 & 0 & 813 & 763 & 1696 \\
\bottomrule
\end{tabular}
\captionof{table}{Outcome decomposition on $D_{\rm SA}$ ($N=4{,}460$): retrieval miss; generated-covered but no positive acquired; positive acquired but reader F1 $=0$; partial F1; exact answer.}
\label{tab:ds-fail}
\end{figure*}

\subsection{Decision Diagnostics on ERASER}
\label{app:f1diag}
First, when a downstream Evidence Inference classifier reads the region localized after a budget of verification calls, macro-F1 rises from $.614$ with one call (the full document) to $.879$ with $64$, close to the minimal oracle region ($.887$) and gold evidence ($.912$; Table~\ref{tab:loc}); this curve reflects localization, not COMBINED$^*$. Second, on $17{,}986$ realized decision frontiers ($93.1\%$ with several choices, $96.4\%$ mixed-label; median size $30$), COMBINED$^*$ lowers Brier from $.2185$ to $.2118$ yet selects the same top-1 candidate and the same full ranking as BASE in every case (Table~\ref{tab:rank}), so probability-greedy acquisition cannot use the improved predictor.

\subsection{Computational Resources and Budget}
\label{app:compute}
\emph{Hardware.} Experiments ran on a single compute node with $4\times$ NVIDIA A100 40GB GPUs and an Intel Xeon (Ice Lake) CPU with $96$ logical cores.
\emph{GPU time.} Runs with recorded timing used $8.735$ GPU-hours for training and $24.363$ for prediction and evaluation ($33.098$ in total). This covers only timed runs; some retrieval, feature-caching, external-baseline, and repeated runs were not timed, so total compute is higher. CPU-side costs of the structure-aware components are reported in Appendix~\ref{app:scaling}.
\emph{Model sizes.} Table~\ref{tab:params} lists parameter counts of the pretrained backbones and whether their parameters were updated.

\begin{table}[h]
\centering
\footnotesize
\setlength{\tabcolsep}{3pt}
\begin{tabular}{@{}lrl@{}}
\toprule
Model & Params & Use \\
\midrule
DeBERTa-v3-base & 184.03M & fine-tuned \\
RoBERTa-base & 124.85M & fine-tuned \\
BERT + MLP & ${\approx}$109.68M & head only \\
SciBERT & 109.92M & App.~\ref{app:f2} \\
DPR encoders (2) & 217.78M & frozen \\
REARANK-7B & 7.616B & frozen \\
\bottomrule
\end{tabular}
\caption{Parameter counts (total, including task heads). DeBERTa and RoBERTa: A0/A1/A2, all parameters trained; DPR: question and context encoders combined. For BERT + MLP, only the MLP head is trained; for the SciFact setting, the containment model is trained and the separate claim classifier used as reader is frozen (Appendix~\ref{app:f2}).}
\label{tab:params}
\end{table}

\subsection{Software and Implementation Details}
\label{app:software}
\emph{Software.} The reference environment uses Python 3.10.19, PyTorch 2.7.1 (CUDA 12.6 build), Transformers 5.14.1, NumPy 2.2.6, scikit-learn 1.7.2, and SciPy 1.15.3. Some earlier experiments ran in prior environments whose exact versions were not recorded.
\emph{Pretrained models and freezing.} All backbones are public pretrained checkpoints. The DeBERTa and RoBERTa A0/A1/A2 predictors are fully fine-tuned (Appendices~\ref{app:lanea} and~\ref{app:roberta}); the BERT N2/FULL pair trains only an MLP on frozen encoder embeddings (Appendix~\ref{app:neural}); the DPR retriever and reader, REARANK, and the Cost-MMR embedding model are used frozen, without fine-tuning or tuning on public DEV (Appendices~\ref{app:realgen}, \ref{app:recent}, and~\ref{app:downstream}).
\emph{Optimization and implementation.} Optimizers, learning rates, epochs, calibration, and seeds are given in the appendix of each experiment (Appendices~\ref{app:model}, \ref{app:neural}, \ref{app:lanea}, \ref{app:f2}, and~\ref{app:realgen}) and are not repeated here. Batch isotonic projection and ESR are implemented in Python and run on CPU (Appendix~\ref{app:scaling}). Paired bootstrap CIs use document- or claim-cluster resampling with the resample counts and seeds stated in each appendix.

\end{document}